\documentclass[twoside]{article} 

\usepackage[accepted]{aistats2017}

\usepackage{graphicx} 
\usepackage{algorithm}
\usepackage{algorithmic}

\usepackage{amsmath,mathtools}

\usepackage{xspace}

\usepackage[utf8]{inputenc} \usepackage[T1]{fontenc}    \usepackage{hyperref}       \usepackage{url}            \usepackage{booktabs}       \usepackage{amsfonts}       \usepackage{nicefrac}       \usepackage{microtype}      \usepackage{enumitem}
\usepackage{xspace}

\usepackage{amsmath}
\usepackage{amsthm}
\usepackage{nccmath}
\usepackage{amsfonts}
\usepackage{amssymb}
\usepackage{tikz}
\usepackage{bbold}
\usepackage{bm}
\usepackage{natbib}

\usetikzlibrary{graphs}
\usetikzlibrary{arrows}
\usepackage[pdf]{pstricks}
\usepackage{auto-pst-pdf}
\usepackage{wrapfig}
\usepackage{subcaption}

\usepackage[colorinlistoftodos, textwidth=20mm, disable]{todonotes}
\definecolor{citrine}{rgb}{0.89, 0.82, 0.04}
\definecolor{blued}{RGB}{70,197,221}
\newcommand{\todoa}[1]{\todo[color=green, inline]{\small #1}}

\newlength{\minipagewidth}
\newlength{\minipagewidthx}
\newcommand{\bookboxx}[1]{\small
\par\medskip\noindent
\framebox[0.49\textwidth]{
\begin{minipage}{0.47\dimexpr\textwidth-\parindent\relax} {#1} \end{minipage} } \par\medskip }

\newtheorem{assumption}{Assumption}
\newtheorem{proposition}{Proposition}
\newtheorem{definition}{Definition}
\newtheorem{theorem}{Theorem}
\newtheorem{lemma}{Lemma}
\newtheorem{corollary}{Corollary}

\newcommand{\M}{\mathcal M}
\newcommand{\A}{\mathcal A}
\renewcommand{\O}{\mathcal O}
\newcommand{\calS}{\mathcal S}

\DeclareMathOperator*{\Argmax}{Arg\,max}

\DeclareMathOperator*{\subExp}{subExp}

\newcommand{\Meq}{\text{eq}}

\newcommand{\ucrl}{{\small\textsc{UCRL}}\xspace}
\newcommand{\ucrlsmdp}{{\small\textsc{UCRL-SMDP}}\xspace}

\newcommand{\wt}[1]{\widetilde{#1}}
\newcommand{\wh}[1]{\widehat{#1}}

\newcommand{\wb}[1]{\overline{#1}}
 
\begin{document}

\twocolumn[

\aistatstitle{Exploration--Exploitation in MDPs with Options}

\aistatsauthor{ Ronan Fruit \And Alessandro Lazaric}

\aistatsaddress{ Inria Lille - SequeL Team \And Inria Lille - SequeL Team} ]

\begin{abstract}
While a large body of empirical results show that temporally-extended actions and options may significantly affect the learning performance of an agent, the theoretical understanding of how and when options can be beneficial in online reinforcement learning is relatively limited. In this paper, we derive an upper and lower bound on the regret of a variant of UCRL using options. While we first analyze the algorithm in the general case of semi-Markov decision processes (SMDPs), we show how these results can be translated to the specific case of MDPs with options and we illustrate simple scenarios in which the regret of learning with options can be \textit{provably} much smaller than the regret suffered when learning with primitive actions.
\end{abstract}

\vspace{-0.05in}
\section{Introduction}
\vspace{-0.05in}

The option framework~\citep{sutton1999between} is a simple yet powerful model to introduce temporally-extended actions and hierarchical structures in reinforcement learning (RL)~\citep{sutton1998introduction}. An important feature of this framework is that Markov decision process (MDP) planning and learning algorithms can be easily extended to accommodate options, thus obtaining algorithms such as option value iteration and $Q$-learning~\citep{sutton1999between}, LSTD~\citep{sorg2010linear}, and actor-critic~\citep{bacon2015the-option-critic}. Temporally extended actions are particularly useful in high dimensional problems that naturally decompose into a hierarchy of subtasks. For instance, \citet{DBLP:journals/corr/TesslerGZMM16} recently obtained promising results by combining options and deep learning for lifelong learning in the challenging domain of Minecraft. A large body of the literature has then focused on how to automatically construct options that are beneficial to the learning process within a single task or across similar tasks (see e.g.,~\citep{mcgovern2001automatic,menache2002q-cut,simsek2004using,Castro:2011:ACT:2341664.2341685,conf/ewrl/LevyS11}). An alternative approach is to design an initial set of options and optimize it during the learning process itself (see e.g., interrupting options \citealt{DBLP:conf/icml/MannMM14} and options with exceptions \citealt{Sairamesh:2011:OE:2341664.2341687}). Despite the empirical evidence of the effectiveness of most of these methods, it is well known that options may as well worsen the performance w.r.t.\ learning with primitive actions~\citep{jong2008the-utility}. Moreover, most of the proposed methods are heuristic in nature and the theoretical understanding of the actual impact of options on the learning performance is still fairly limited. Notable exceptions are the recent results of~\citet{mann2014scaling} and~\citet{brunskill2014pac-inspired}. Nonetheless, \citet{mann2014scaling} rather focus on a batch setting and they derive a sample complexity analysis of approximate value iteration with options. Furthermore, the PAC-SMDP analysis of~\citet{brunskill2014pac-inspired} describes the performance in SMDPs but it cannot be immediately translated into a sample complexity of learning with options in MDPs. 

In this paper, we consider the case where a fixed set of options is provided and we study their impact on the learning performance w.r.t. learning without options. In particular, we derive the first regret analysis of learning with options. Relying on the fact that using options in an MDP induces a semi-Markov decision process (SMDP), we first introduce a variant of the \ucrl algorithm~\citep{Jaksch10} for SMDPs and we upper- and lower-bound its regret (sections~\ref{sec:smdp.ucrl} and~\ref{sec:analysis}). While this result is of independent interest for learning in SMDPs, its most interesting aspect is that it can be translated into a regret bound for learning with options in MDPs and it provides a first understanding on the conditions sufficient for a set of options to reduce the regret w.r.t.\ learning with primitive actions (Sect.~\ref{sec:smdp2mdp}). Finally, we provide an illustrative example where the empirical results support the theoretical findings (Sect.~\ref{sec:exp}).
 
\vspace{-0.05in}
\section{Preliminaries}
\vspace{-0.05in}

\textbf{MDPs and options.} A finite MDP is a tuple $M = \big\{ \mathcal{S}, \mathcal{A}, p, r \big\}$ where $\mathcal{S}$ is a finite set of states, $\mathcal{A}$ is a finite set of actions, $p(s'|s,a)$ is the probability of transitioning from state $s$ to state $s'$ when action $a$ is taken, $r(s,a,s')$ is a distribution over rewards obtained when action $a$ is taken in state $s$ and the next state is $s'$. A stationary deterministic policy $\pi : \calS \rightarrow \A$ maps states to actions.
A (Markov) option is a tuple $o = \big\lbrace \mathcal{I}_o, \beta_o , \pi_o \big\rbrace$ where $\mathcal{I}_o \subset \mathcal{S} $ is the set of states where the option can be initiated, $\beta_o: \mathcal{S} \rightarrow [0,1] $ is the probability distribution that the option ends in a given state, and $\pi_o : \mathcal{S} \rightarrow \mathcal{A}$ is the policy followed until the option ends. Whenever the set of primitive actions $\A$ is replaced by a set of options $\O$, the resulting decision process is no longer an MDP but it belongs to the family of semi-Markov decision processes (SMDP).

\begin{proposition}\label{prop:smdp}[\citealt{sutton1999between}]
For any MDP $M$ and a set of options $\O$, the resulting decision process is an SMDP $M_{\O} = \big\{\mathcal{S}_{\O}, \mathcal{O}, p_{\O}, r_{\O}, \tau_{\O} \big\}$, where $\mathcal{S}_{\O} \subseteq \mathcal{S}$ is the set of states where options can start and end,
\begin{align*}
\calS_\O = \Big(\cup_{o\in\O} \mathcal{I}_o\Big) \bigcup \Big(\cup_{o\in\O} \{s: \beta_o(s) > 0\}\Big),
\end{align*}
$\O$ is the set of available actions, $p_{\O}(s,o,s')$ is the probability of transition from $s$ to $s'$ by taking the policy $\pi_o$ associated to option $o$, i.e., 
\begin{align*}
p_{\O}(s,o,s') = \sum_{k=1}^\infty p(s_k = s' | s,\pi_o) \beta_o(s'),
\end{align*}
where $p(s_k = s' | s,\pi_o)$ is the probability of reaching state $s'$ after exactly $k$ steps following policy $\pi_o$, $r_{\O}(s,o,s')$ is the distribution of the cumulative reward obtained by executing option $o$ from state $s$ until interruption at $s'$, and $\tau_{\O}(s,o,s')$ is the distribution of the holding time (i.e., number of primitive steps executed to go from $s$ to $s'$ by following $\pi_o$).
\end{proposition}

Throughout the rest of the paper, we only consider an ``admissible'' set of options $\O$ such that all options terminate in finite time with probability 1 and in all possible terminal states there exists at least one option that can start, i.e., $\cup_{o\in\O} \{s: \beta_o(s) > 0\} \subseteq \cup_{o\in\O} \mathcal{I}_o$. We further assume that the resulting SMDP $M_\O$ is communicating i.e. has a finite diameter (see Def. \ref{def:Diameter}). Finally, we notice that a stationary deterministic policy constructed on a set of options $\O$ may result into a non-stationary policy on the set of actions $\A$.

\textbf{Learning in SMDPs.} Relying on this mapping, we first study the exploration-exploitation trade-off in a generic SMDP. A thorough discussion on the implications of the analysis of learning in SMDPs for the case of learning with options in MDPs is reported in Sect.~\ref{sec:smdp2mdp}.
For any SMDP $M = \big\{\mathcal{S}, \mathcal{A}, p, r, \tau \big\}$, we denote by $\overline{\tau}(s,a,s')$ (resp. $\overline{r}(s,a,s')$) the expectation of $\tau(s,a,s')$ (resp. $r(s,a,s')$) and by $\overline{\tau}(s,a) = \sum_{s' \in \mathcal{S}} \overline{\tau}(s,a,s')p(s'|s,a)$ (resp. $\overline{r}(s,a) = \sum_{s' \in \mathcal{S}} \overline{r}(s,a,s')p(s'|s,a)$) the expected holding time (resp. cumulative reward) of action $a$ from state $s$. In the next proposition we define the average-reward performance criterion and we recall the properties of the optimal policy in SMDPs.

\begin{proposition}\label{prop:optimal.policy}
Denote $N(t) = \sup \big\{n \in \mathbb{N}, \ \sum_{i=1}^{n}{\tau_i} \leq t \big\}$ the number of decision steps that occurred before time $t$. For any policy $\pi$ and $s\in \mathcal{S}$:
\begin{align}\label{eq:rho}
\begin{split}
\overline{\rho}^\pi(s) \overset{def}{=} \limsup_{t \rightarrow + \infty} \mathbb{E}^\pi \bigg[ \frac{\sum_{i=1}^{N(t)}{r_i}}{t} \bigg| s_0 =s \bigg] \\
\underline{\rho}^\pi(s) \overset{def}{=} \liminf_{t \rightarrow + \infty} \mathbb{E}^\pi \bigg[ \frac{\sum_{i=1}^{N(t)}{r_i}}{t} \bigg| s_0 =s \bigg].
\end{split}
\end{align}
If $M$ is communicating and the expected holding times and reward are finite, there exists a stationary deterministic optimal policy $\pi^*$ such that for all states $s$ and policies $\pi$, $\underline{\rho}^{\pi^*}(s) \geq \overline{\rho}^\pi(s)$ and $\overline{\rho}^{\pi^*}(s) = \underline{\rho}^{\pi^*}(s) = \rho^*$.
\end{proposition}

Finally, we recall the average reward optimality equation for a communicating SMDP
\begin{align}\label{eq:optimality.eq}
u^*(s) = \max_{a \in \mathcal{A}} \Big\lbrace &\overline{r}(s,a) - \rho^* \overline{\tau}(s,a) \\
&+ \sum_{s' \in \mathcal{S}}{p(s'|s, a)}u^*(s') \Big\rbrace, \nonumber
\end{align}
where $u^*$ and $\rho^*$ are the bias (up to a constant) and the gain of the optimal policy $\pi^*$.

We are now ready to consider the learning problem. For any $i \in \mathbb{N}^*$, $a_i$ denotes the action taken by the agent at the $i$-th decision step\footnote{Notice that decision steps are discrete points in time in which an action is started, while the (possibly continuous) holding time is determined by the distribution $\tau$.} and $s_i$ denotes the state reached after $a_i$ is taken, with $s_0$ being the initial state. We denote by $(r_i(s,a,s'))_{i \in \mathbb{N}^*}$ (resp. $(\tau_i(s,a,s'))_{i \in \mathbb{N}^*}$) a sequence of i.i.d.\ realizations from distribution $r(s,a,s')$ (resp. $\tau(s,a,s')$). When the learner explores the SMDP, it observes the sequence $(s_0, \ldots, s_i, a_{i+1}, r_{i+1}(s_{i},a_{i+1},s_{i+1}), \tau_{i+1}(s_{i},a_{i+1},s_{i+1}),\\ \ldots)$.
The performance of a learning algorithm is measured in terms of its cumulative \textit{regret}.

\begin{definition}\label{def:regret}
For any SMDP $M$, any starting state $s \in \mathcal{S}$, and any number of decision steps $n \geq 1$, let $\{\tau_i\}_{i=1}^n$ be the random holding times observed along the trajectory generated by a learning algorithm $\mathfrak{A}$. Then the total regret of $\mathfrak{A}$ is defined as
\begin{equation}\label{eq:regret}
\Delta(M,\mathfrak{A},s,n) = \bigg( \sum_{i=1}^{n}{\tau_i} \bigg) \rho^*(M) - \sum_{i=1}^{n}{r_i}.
\end{equation}
\end{definition}

We first notice that this definition reduces to the standard regret in MDPs for $\tau_i = 1$ (i.e., primitive actions always terminate in one step). The regret measures the difference in cumulative reward obtained by the optimal policy and the learning algorithm. While the performance of the optimal policy is measured by its asymptotic average reward $\rho^*$, the total duration after $n$ decision steps may vary depending on the policy. As a result, when comparing the performance of $\pi^*$ after $n$ decision steps, we multiply it by the length of the trajectory executed by the algorithm $\mathfrak{A}$. More formally, from the definition of $\rho^*$ (Eq.~\ref{eq:rho}) and Prop.~\ref{prop:optimal.policy} we have\footnote{In this expectation, $N(t)$ is a r.v. depending on $\pi^*$.}
\begin{equation*}
\mathbb{E}^{\pi^*} \bigg[ \sum_{i=1}^{N(t)}{r_i} \Big| s_0 =s \bigg] \underset{t \rightarrow +  \infty}{\sim} \rho^*t + o(t).
\end{equation*}
By introducing the total duration $N(t)$ of $\mathfrak{A}$ we have
\begin{equation*}
\rho^*t + o(t) = \rho^*\bigg( \sum_{i=1}^{N(t)}{\tau_i} \bigg) + \rho^*\bigg(t-\sum_{i=1}^{N(t)}{\tau_i}\bigg) + o(t).
\end{equation*}
We observe that $\big(t-\sum_{i=1}^{N(t)}{\tau_i}\big) = o(t)$ almost surely since $\big(t-\sum_{i=1}^{N(t)}{\tau_i}\big) \leq \tau_{N(t)+1}$ and $\tau_{N(t)+1}$ is bounded by an almost surely finite (a.s.) random variable since the expected holding time for all state-action pairs is bounded by assumption. So $\tau_{N(t)+1}/t \underset{t \rightarrow +  \infty}{\rightarrow} 0$ a.s. and
\begin{equation*}
\mathbb{E}^{\pi^*} \bigg[ \sum_{i=1}^{N(t)}{r_i} \Big| s_0 =s \bigg] \underset{t \rightarrow +  \infty}{\sim} \rho^*\bigg( \sum_{i=1}^{N(t)}{\tau_i} \bigg) + o(t),
\end{equation*}
which justifies the definition of the regret.
 
\vspace{-0.05in}
\section{SMDP-UCRL}\label{sec:smdp.ucrl}
\vspace{-0.05in}

\begin{figure}[t!]
\bookboxx{
\textbf{Input:} Confidence $\delta \in ]0,1[$, $\mathcal{S}$, $\mathcal{A}, b_r, \sigma_r$, $b_\tau, \sigma_\tau, R_{\max}, \tau_{\max}$ and $\tau_{\min}$.

\noindent \textbf{Initialization:} Set $i= 1$, and observe initial state $s_0$.\\

\noindent \textbf{For} episodes $k=1, 2, ...$ \textbf{do}

\ \ \ \textit{\textbf{Initialize episode}} $k$:
\vspace{-0.1in}

\begin{enumerate}[leftmargin=4mm,itemsep=-1mm]
\item Set the start step of episode $k$, $i_k := i$
\item For all $(s,a)$ initialize the counter for episode $k$, $\nu_k (s,a) := 0$ and set counter prior to episode $k$, 
\begin{equation*}
N_k (s,a) = \#\{\iota < i_k : s_{\iota} = s, a_{\iota} = a \}
\end{equation*}

\item For $s,s',a$ set the accumulated rewards, duration and transition counts prior to episode $k$,
\begin{align*}
\begin{split}
R_k (s,a) \!=\!\! \sum_{\iota=1}^{i_k-1} r_{\iota} \mathbb{1}_{s_{\iota}=s, a_{\iota}=a}, T_k (s,a) \!=\!\! \sum_{\iota=1}^{i_k-1} \tau_{\iota} \mathbb{1}_{s_{\iota}=s, a_{\iota}=a}
\\
P_k(s,a,s') = \#\{\iota< i_k : s_{\iota} = s, a_{\iota} = a, s_{\iota + 1} = s' \}
\end{split}
\end{align*}

Compute estimates $\hat{p}_k(s' \mid s,a) := \frac{P_k (s,a,s')}{\max \{ 1, N_k (s,a)\}}$ and $\hat{\tau}_k(s,a) := \frac{T_k (s,a)}{N_k (s,a)}$ and $\hat{r}_k(s,a) := \frac{R_k (s,a)}{N_k (s,a)}$

\textit{\textbf{Compute policy}} $\widetilde{\pi}_k$:

\item Let $\mathcal{M}_k$ be the set of all SMDPs with states and actions as in $M$, and with transition probabilities $\widetilde{p}$, rewards $\widetilde{r}$, and holding time $\widetilde{\tau}$ such that for any $(s,a)$
\begin{align*}
\begin{split}
&\arrowvert \widetilde{r} - \hat{r}_k \arrowvert \leq \beta_k^r  \ \textrm{  and  } \ R_{\max}\tau_{\max} \geq \widetilde{r}(s,a) \geq 0
\\
&\arrowvert \widetilde{\tau} - \hat{\tau}_k \arrowvert \leq \beta_k^{\tau} \ \textrm{  and  }  \ \tau_{\max} \geq \widetilde{\tau}(s,a)  \geq\widetilde{r}(s,a)/R_{\max} , \tau_{\min} 
\\
&\Arrowvert  \widetilde{p}(\cdot) - \hat{p}_k(\cdot) \Arrowvert_1 \leq \beta_k^p \ \textrm{  and  } \ \sum_{s' \in \mathcal{S}}{\widetilde{p}(s' \mid s,a)} = 1
\end{split}
\end{align*}

\item Use \textit{extended value iteration} (EVI) to find a policy $\widetilde{\pi}_k$ and an optimistic SMDP $\widetilde{M}_k \in \mathcal{M}_k$ such that:
\begin{equation*} 
\widetilde{\rho}_k := \min_{s} \rho (\widetilde{M}_k, \widetilde{\pi}_k, s) \geq \max_{M' \in \mathcal{M}_k , \pi, s} \rho (M', \pi, s) - \frac{R_{\max}}{\sqrt{i_k}}
\end{equation*}

\vspace{-0.1in}
\textit{\textbf{Execute policy}} $\widetilde{\pi}_k$:

\item \textbf{While} $\nu_k(s_{i}, \widetilde{\pi}_k(s_{i})) < \max \{1, N_k(s_{i},\widetilde{\pi}_k(s_{i})\}$ \textbf{do}
\begin{enumerate}
\item Choose action $a_{i} = \widetilde{\pi}_k(s_{i})$, obtain reward $r_{i}$, and observe next state $s_{i + 1}$
\item Update $\nu_k (s_i,a_i) := \nu_k(s_i,a_i)+1$ and set $i = i+1$
\end{enumerate}

\end{enumerate}
}
\vspace{-0.1in}
\caption{\ucrlsmdp}
\label{fig:ucrl.smdp}
\vspace{-0.2in}
\end{figure}

In this section we introduce \ucrlsmdp (Fig.~\ref{fig:ucrl.smdp}), a variant of \ucrl~\citep{Jaksch10}. At each episode $k$, the set of plausible SMDPs $\mathcal{M}_k$ is defined by the current estimates of the SMDP parameters and a set of constraints on the rewards, the holding times and the transition probabilities derived from the confidence intervals. Given $\M_k$, extended value iteration (EVI) finds an SMDP $\widetilde{M}_k \in \mathcal{M}_k$ that maximizes $\rho^{*}(\widetilde{M}_k)$ and the corresponding optimal policy $\widetilde{\pi}_k^*$ is computed. To solve this problem, we note that it can be equivalently formulated as finding the optimal policy of an extended\footnote{In the MDP literature, the term Bounded Parameter MDPs (BPMDPs) \citep{Tewari2007} is often used for "extended" MDPs built using confidence intervals on rewards and transition probabilities.} SMDP $\widetilde{M}_k^+$ obtained by combining all SMDPs in $\mathcal{M}_k$: $\widetilde{M}_k^+$ has the same state space and an extended continuous action space $\widetilde{\mathcal{A}}_k^+$. Choosing an action $a^+ \in\widetilde{\mathcal{A}}_k^+$ amounts to choosing an action $a \in \mathcal{A}$,  a reward $\widetilde{r}_k(s,a)$, a holding time $\widetilde{\tau}_k(s,a)$ and a transition probability $\widetilde{p}_k(\cdotp \mid s,a)$ in the confidence intervals. When $a^{+}$ is executed in $\widetilde{M}_k^+$, the probability, the expected reward and the expected holding time of the transition are respectively $\widetilde{p}_k(\cdotp \mid s,a)$, $\widetilde{r}_k(s,a)$ and $\widetilde{\tau}_k(s,a)$. Finally, $\widetilde{\pi}_k^*$ is executed until the number of samples for a state-action pair is doubled. Since the structure is similar to \ucrl's, we focus on the elements that need to be rederived for the specific SMDP case: the confidence intervals construction and the extended value iteration algorithm.

\textbf{Confidence intervals.}
Unlike in MDPs, we consider a slightly more general scenario where cumulative rewards and holding times are not bounded but are sub-exponential r.v. (see Lem.~\ref{lem:MDPwithOptions}). As a result, the confidence intervals used at step 4 are defined as follows. For any state action pair $(s,a)$ and for rewards, transition probabilities, and holding times we define

\vspace{-0.1in}
\begin{small}
\begin{align*}
\beta_k^r(s,a) &\!=\! 
\begin{cases} 
\sigma_r \sqrt{\frac{14 \log \left( 2SA i_k/\delta \right)}{\max\{1, N_k(s,a)\}}}, \ \!\mbox{if } N_k(s,a) \!\geq\! \frac{2 b_r^2}{\sigma_r^2}\log \big( \frac{240SAi_k^7}{\delta}\big)\\
14 b_r \frac{\log \left( 2SA i_k/\delta\right)}{\max\{1, N_k(s,a)\}}, \ \mbox{ otherwise }
\end{cases}\\ 
\beta_k^p(s,a) &= \sqrt{\frac{14 S \log(2A i_k / \delta)}{\max\{1, N_k(s,a)\}}},\\
\beta_k^{\tau}(s,a)& \!=\! 
\begin{cases} 
\sigma_\tau \sqrt{\frac{14 \log \left( 2SA i_k/\delta \right)}{\max\{1, N_k(s,a)\}}}, \ \!\mbox{if } N_k(s,a) \!\geq\! \frac{2 b_\tau^2}{\sigma_\tau^2}\log \big( \frac{240SAi_k^7}{\delta}\big)\\
14 b_\tau \frac{ \log \left( 2SA i_k/\delta\right)}{\max\{1, N_k(s,a)\}}, \ \mbox{ otherwise }
\end{cases}
\end{align*}
\end{small}
\vspace{-0.1in}

\noindent where $\sigma_r$, $b_r$, $\sigma_\tau$, $b_r$ are suitable constants characterizing the sub-exponential distributions of rewards and holding times. As a result, the empirical estimates $\hat r_k$, $\hat\tau_k$, and $\hat p_k$ are $\pm \beta_k^r(s,a), \beta_k^\tau(s,a), \beta_k^p(s,a)$ away from the true values.

\textbf{Extended value iteration (EVI).}
We rely on a data-transformation (also called ``uniformization'') that turns an SMDP $M$ into an ``equivalent'' MDP $M_{\Meq} = \big\{ \mathcal{S}, \mathcal{A}, p_{\Meq}, r_{\Meq} \big\}$ with same state and action spaces and such that $\forall s,s' \in \mathcal{S}, \forall a \in \mathcal{A}$:
\begin{align}\label{eq:eq.MDP}
\begin{split}
\overline{r}_{\Meq}(s,a) &= \frac{\overline{r}(s,a)}{\overline{\tau}(s,a)}\\
p_{\Meq}(s' | s,a) &= \frac{\tau}{\overline{\tau}(s,a)} \big( p(s' | s,a) - \delta_{s,s'} \big) + \delta_{s,s'}
\end{split}
\end{align}
where $\delta_{s,s'} = 0$ if $s \neq s'$ and $\delta_{s,s'} = 1$ otherwise, and $\tau$ is an arbitrary non-negative real such that $\tau < \tau_{\min}$. $M_{\Meq}$ enjoys the following equivalence property.

\begin{proposition}[\citep{Federgruen83}, Lemma~2]\label{lem:equivalence.uniformization}
If $(v^*, g^*)$ is an optimal pair of bias and gain in $M_{\Meq}$ then $(\tau^{-1}v^*, g^*)$ is a solution to Eq.~\ref{eq:optimality.eq}, i.e., it is an optimal pair of bias/gain for the original SMDP $M$.
\end{proposition}
As a consequence of the equivalence stated in Prop. \ref{lem:equivalence.uniformization}, computing the optimal policy of an SMDP amounts to computing the optimal policy of the MDP obtained after data transformation (see App. \ref{GainSMDPs} for more details). Thus, EVI is obtained by applying a value iteration scheme to an MDP $\widetilde{M}_{k,eq}^+$ equivalent to the extended SMDP $\widetilde{M}_k^+$. We denote the state values of the $j$-th iteration by $u_j(s)$. We also use the vector notation $u_j = (u_j(s))_{s \in \mathcal{S}}$. Similarly, we denote by $\widetilde{p}(\cdot \mid s,a) = (\widetilde{p}(s' \mid s,a))_{s' \in \mathcal{S}}$ the transition probability vector of state-action pair $(s,a)$. The optimistic reward at episode $k$ is fixed through the EVI iterations and it is obtained as $\widetilde{r}_{j+1}(s,a) = \min \big\lbrace \hat{r}_{k}(s,a) + \beta_k^r(s,a); R_{\max}\tau_{\max} \big\rbrace$, i.e., by taking the largest possible value compatible with the confidence intervals. At iteration $j$, the optimistic transition model is obtained as $\widetilde{p}_{j+1}(\cdotp \mid s,a) \in \Argmax_{p(\cdotp) \in \mathcal{P}_k(s,a)} \left\lbrace p^\intercal u_j \right\rbrace$ and $\mathcal{P}_k(s,a)$ is the set of probability distributions included in the confidence interval defined by $\beta_k^p(s,a)$.
This optimization problem can be solved in $O(S)$ operations using the same algorithm as in \ucrl. Finally, the optimistic holding time depends on $u_j$ and the optimistic transition model $\widetilde{p}_{j+1}$ as
\begin{align*}
\begin{split}
&\widetilde{\tau}_{j+1}(s,a) = \min \Big\lbrace \vphantom{2^1} \tau_{\max}; \max \big\lbrace  \tau_{\min}; \hat{\tau}_k(s,a) \\
&  \!\!-\!  \text{sgn} \big[ \widetilde{r}_{j+1}(s,a) \!+\!\tau \big(\widetilde{p}_{j+1}(\cdot|s,a)^\intercal u_j  \!-u_j(s) \big) \big] \beta_k^{\tau}(s,a) \big\rbrace  \Big\rbrace,
\end{split}
\end{align*}
The $\min$ and $\max$ insure that $\widetilde{\tau}_{j+1}$ ranges between $\tau_{\min}$ and $\tau_{\max}$. When the term $\widetilde{r}_{j+1}(s,a) + \big(\widetilde{p}_{j+1}(\cdot|s,a)^\intercal u_j  \!-u_j(s) \big)$ is positive (resp. negative), $\widetilde{\tau}_{j+1}(s,a)$ is set to the minimum (resp. largest) possible  value compatible with its confidence intervals so as to maximize the right-hand side of Eq.~\ref{eq:evi} below. As a result, for any $\tau \in \left]0,\tau_{\min}\right[$, EVI is applied to an MDP equivalent to the extended SMDP $\widetilde{M}_k^+$ generated over iterations as
\begin{align} \label{eq:evi}
u_{j+1}&(s) = \max_{a \in \mathcal{A}} \bigg\lbrace \frac{\widetilde{r}_{j+1}(s,a)}{\widetilde{\tau}_{j+1}(s,a)} \\
&+ \frac{\tau}{\widetilde{\tau}_{j+1}(s,a)} \Big( \vphantom{\frac{1}{2} } \widetilde{p}_{j+1}(\cdotp \mid s,a)^\intercal u_j \vphantom{\frac{1}{2} }  - u_j(s) \Big) \bigg\rbrace + u_j(s) \nonumber
\end{align}
with arbitrary $u_0$. Finally, the stopping condition is
\begin{align}\label{eq:stopping}
\max_{s \in \mathcal{S}} \lbrace u_{i+1}(s) \!-\! u_i(s) \rbrace \!-\! \min_{s \in \mathcal{S}} \lbrace u_{i+1}(s) \!-\! u_i(s) \rbrace < \epsilon.
\end{align}
We prove the following.

\begin{lemma}\label{lem:validity.evi}
If the stopping condition holds at iteration $i$ of EVI, then the greedy policy w.r.t. $u_i$ is $\epsilon$-optimal w.r.t.\ extended SMDP $\widetilde{M}_k^+$. The stopping condition is always reached in a finite number of steps.
\end{lemma}

As a result, we can conclude that running EVI at each episode $k$ with an accuracy parameter $\epsilon=R_{\max}/\sqrt{i_k}$ guarantees that $\widetilde{\pi}_k$ is $R_{\max}/\sqrt{i_k}$-optimal w.r.t. $\max_{M' \in \mathcal{M}_k}\rho^{*}(M')$.

\vspace{-0.05in}
\section{Regret Analysis}\label{sec:analysis}
\vspace{-0.05in}

In this section we report upper and lower bounds on the regret of \ucrlsmdp. We first extend the notion of diameter to the case of SMDP as follows.

\begin{definition}\label{def:Diameter}
For any SMDP $M$, we define the diameter $D(M)$ by
\begin{equation}\label{eq3}
D(M) = \max_{s,s' \in \mathcal{S}} \bigg\{ \min_{\pi} \Big\{ \mathbb{E}^{\pi} \big[ T(s')|s_0 = s \big]  \Big\} \bigg\}
\end{equation}
where $T(s')$ is the first time in which $s'$ is encountered, i.e., $T(s') = \inf \big\{ \sum_{i=1}^{n}{\tau_i} : n \in \mathbb{N}, \ s_n = s'\big\}$.
\end{definition}

Note that the diameter of an SMDP corresponds to an average \textit{actual} duration and not an average number of decision steps. However, if the SMDP is an MDP the two definitions of diameter coincides. Before reporting the main theoretical results about \ucrlsmdp, we introduce a set of technical assumptions.

\begin{assumption} \label{asm:BoundedExpectedDurationsRewards}
For all $s \in \mathcal{S}$ and $a \in \mathcal{A}$, we assume that $\tau_{\max} \geq \overline{\tau}(s,a) \geq \tau_{\min} > 0$ and $\max_{s \in \mathcal{S},a \in \mathcal{A}} \left\lbrace \frac{\overline{r}(s,a)}{\overline{\tau}(s,a)} \right\rbrace \leq R_{\max}$ with $\tau_{\min}$, $\tau_{\max}$, and $R_{\max}$ known to the learning algorithm.
Furthermore, we assume that the random variables $(r(s,a,s'))_{s,a,s'}$ and $(\tau(s,a,s'))_{s,a,s'}$ are either \textbf{1)} \textit{sub-Exponential} with constants $(\sigma_r, b_r)$ and $(\sigma_\tau , b_\tau )$, or \textbf{2)} \textit{bounded} in $[0, R_{\max}T_{\max}]$ and $[T_{\min},T_{\max}]$, with $T_{\min}>0$. We also assume that the constants characterizing the distributions are known to the learning agent.
\end{assumption}

We are now ready to introduce our main result.

\begin{theorem}\label{UpperBound}
With probability of at least $1-\delta$, it holds that for any initial state $s \in \mathcal{S}$ and any $n>1$, the regret of \ucrlsmdp $\Delta(M,\mathfrak{A},s,n)$ is bounded by:
\begin{align}
O \left( \left( D \sqrt{S}+ \mathcal{C}(M,n,\delta) \right) R_{\max} \sqrt{S A n \log \left( \frac{n}{\delta}\right)}\right),
\end{align}
where $\mathcal{C}(M,n,\delta)$ depends on which case of Asm.~\ref{asm:BoundedExpectedDurationsRewards} is considered~\footnote{We denote $\max\{a,b\} = a \vee b$.}
\begin{align*}
\begin{split}
&\hspace{-0.1in}\text{\textbf{1)} sub-Exponential}\\
& \mathcal{C}(M,n,\delta) = \tau_{\max} + \Big(\frac{\sigma_r \vee b_r}{R_{\max}} + \sigma_\tau \vee b_\tau\Big)\sqrt{\log \left( \frac{n}{\delta}\right)},
\end{split}\\
\begin{split}
&\hspace{-0.1in}\text{\textbf{2)} bounded}\\
&\mathcal{C}(M,n,\delta) = T_{\max} + (T_{\max}-T_{\min}).
\end{split}
\end{align*}
\end{theorem}

\textit{Proof.} The proof (App.~\ref{app:proof.smdp.ucrl}) follows similar steps as in~\citep{Jaksch10}. Apart from adapting the concentration inequalities to sub-exponential r.v. and deriving the guarantees about EVI applied to the equivalent MDP $M_{\Meq}$ (Lem.~\ref{lem:validity.evi}), one of the key aspects of the proof is to show that the learning complexity is actually determined by the diameter $D(M)$ in Eq.~\ref{def:Diameter}. As for the analysis of EVI, we rely on the data-transformation and we show that the span of $u_j$ (Eq.~\ref{eq:evi}) can be bounded by the diameter of $M_{\Meq}$, which is related to the diameter of the original SMDP as $D(M_{\Meq}) = D(M)/\tau$ (Lem.~\ref{SpanVI} in App.~\ref{app:proof.smdp.ucrl}).

\textit{The bound.} The upper bound is a direct generalization of the result derived by~\citet{Jaksch10} for \ucrl in MDPs. In fact, whenever the SMDP reduces to an MDP (i.e., each action takes exactly one step to execute), then $n=T$ and the regret, the diameter, and the bounds are the same as for \ucrl. If we consider $R_{\max}=1$ and bounded holding times, the regret scales as $\wt{O}(DS\sqrt{An} + T_{\max}\sqrt{SAn})$. The most interesting aspect of this bound is that the extra cost of having actions with random duration is only partially additive rather than multiplicative (as it happens e.g., with the per-step reward $R_{\max}$). This shows that errors in estimating the holding times do not get amplified by the diameter $D$ and number of states $S$ as much as it happens for errors in reward and dynamics. This is confirmed in the following lower bound.

\begin{theorem}\label{LowerBound}
For any algorithm $\mathfrak{A}$, any integers $S, A \geq 10$, any reals ${T_{\max} \geq 3T_{\min} \geq 3 }$, ${R_{\max}>0}$, $D > \max\lbrace{20 T_{\min}{\log}_A(S), 12 T_{\min} \rbrace}$, and for ${n\geq \max\lbrace D, T_{\max} \rbrace SA}$, there is an SMDP $M$ with at most $S$ states, $A$ actions, and diameter $D$, with holding times in $\left[T_{\min},T_{\max}\right]$ and rewards in $\left[0,\frac{1}{2}R_{\max}T_{\max}\right]$ satisfying $\forall s \in \mathcal{S}$, $\forall a \in \mathcal{A}_s$, $\overline{r}(s,a) \leq R_{\max}\overline{\tau}(s,a)$, such that for any initial state $s \in \mathcal{S}$ the expected regret of $\mathfrak{A}$ after $n$ decision steps is lower-bounded by:
\begin{align*}
\mathbb{E}\left[\Delta(M,\mathfrak{A},s,n)\right] = \Omega \left( \big( \sqrt{D} + \sqrt{T_{\max}} \big) R_{\max}\sqrt{S A n }\right)
\end{align*}
\end{theorem}

\textit{Proof.} Similar to the upper bound, the proof (App.~\ref{app:proof.lower.bound}) is based on~\citep{Jaksch10} but it requires to perturb transition probabilities and rewards at the same time to create a family of SMDPs with different optimal policies that are difficult to discriminate. The contributions of the two perturbations can be made independent. More precisely, the lower bound is obtained by designing SMDPs where learning to distinguish between ``good'' and ``bad'' transition probabilities and learning to distinguish between ``good'' and ``bad'' rewards are two independent problems, leading to two additive terms $\sqrt{D}$ and $\sqrt{T_{\max}}$ in the lower bound.

\textit{The bound.} Similar to \ucrl, this lower bound reveals a gap of $\sqrt{DS}$ on the first term and $\sqrt{T_{\max}}$. While closing this gap remains a challenging open question, it is a problem beyond the scope of this paper. 

In the next section, we discuss how these results can be used to bound the regret of options in MDPs and what are the conditions that make the regret smaller than using \ucrl on primitive actions.

\vspace{-0.05in}
\section{Regret in MDPs with Options}\label{sec:smdp2mdp}
\vspace{-0.05in}

Let $M$ be an MDP and $\mathcal{O}$ a set of options and let $M_{\O}$ be the corresponding SMDP obtained from Prop.~\ref{prop:smdp}. We index time steps (i.e., time at primitive action level) by $t$ and decision steps (i.e., time at option level) by $i$. We denote by $N(t)$ the total number of decision steps that occurred before time $t$. Given $n$ decision steps, we denote by $T_n = \sum_{i=1}^{n} \tau_i$ the number of time steps elapsed after the execution of the $n$ first options so that $N(T_n) = n$.
Any SMDP-learning algorithm $\mathfrak{A}_{\O}$ applied to $M_{\O}$ can be interpreted as a learning algorithm $\mathfrak{A}$ on $M$ so that at each time step $t$, $\mathfrak{A}$ selects an action of $M$ based on the policy associated to the option started at decision step $N(t)$. We can thus compare the performance of \ucrl and \ucrlsmdp when learning in $M$. We first need to relate the notion of average reward and regret used in the analysis of \ucrlsmdp to the original counterparts in MDPs.

\begin{lemma}\label{thm:smdp.mdp.equivalence}
Let $M$ be an MDP, $\mathcal{O}$ a set of options and $M_{\O}$ the corresponding SMDP. Let $\pi_{\O}$ be any stationary policy on $M_{\O}$ and $\pi$ the equivalent policy on $M$ (not necessarily stationary). For any state $s \in \mathcal{S_{\O}}$, any learning algorithm $\mathfrak{A}$, and any number of decision steps $n$ we have $\rho^{\pi_{\O}}(M_{\O},s) = \rho^{\pi}(M,s)$ and 
\begin{align*}\Delta(M,\mathfrak{A},T_n) = \Delta(M_{\O},&\mathfrak{A},n) + T_n \left(\rho^*(M) -\rho^*(M_{\O})\right).
\end{align*}
\end{lemma}

The linear regret term is due to the fact that the introduction of options amounts to constraining the space of policies that can be expressed in $M$. As a result, in general we have $\rho^*(M) \geq \rho^*(M_{\O}) = \max_{\pi_{\O}} \rho^{\pi_{\O}}(M_{\O})$, where $\pi_{\O}$ is a stationary deterministic policy on $M_{\O}$. Thm.~\ref{thm:smdp.mdp.equivalence} also guarantees that the optimal policy computed in the SMDP $M_{\O}$ (i.e., the policy maximizing $\rho^{\pi_{\O}}(M_{\O},s)$) is indeed the best in the subset of policies that can be expressed in $M$ by using the set of options $\O$. 
In order to use the regret analysis of Thm.~\ref{UpperBound}, we still need to show that Asm.~\ref{asm:BoundedExpectedDurationsRewards} is verified. 

\begin{lemma}\label{lem:MDPwithOptions}
An MDP provided with a set of options is an SMDP where the holding times and rewards $\tau(s,o,s')$ and $r(s,o,s')$ are distributed as sub-exponential random variables. Moreover, the holding time of an option is sub-Gaussian if and only if it is almost surely bounded.
\end{lemma}

This result is based on the fact that once an option is executed, we obtain a Markov chain with absorbing states corresponding to the states with non-zero termination probability $\beta_o(s)$ and the holding time is the number of visited states before reaching a terminal state. While in general this corresponds to a sub-exponential distribution, whenever the option has a zero probability to reach the same state twice before terminating (i.e., there is no cycle), then the holding times become bounded. Finally, we notice that no intermediate case between sub-exponential and bounded distributions is admissible (e.g., sub-Gaussian). Since these are the two cases considered in Thm.~\ref{UpperBound}, we can directly apply it and obtain the following corollary.

\begin{corollary}\label{cor:mdp.options.regret}
For any MDP $M=\{\calS,\A,p,r\}$ with $r(s,a,s') \in [0,R_{\max}]$ and a set of options $\O$, consider the resulting SMDP $M_{\O}=\{\calS_{\O},\A_{\O},p_{\O},r_{\O},\tau_{\O}\}$. Then with probability of at least $1-\delta$, it holds that for any initial state $s \in \mathcal{S}$ and any $n>1$, the regret of \ucrlsmdp in the original MDP is bounded as
\begin{align*}
O \Big( \big( D_{\O} \sqrt{S_{\O}}+ \mathcal{C}(M_{\O},n,\delta) \big) &R_{\max}^{_{\O}} \sqrt{S_{\O} O n \log \left( \frac{n}{\delta}\right)}\Big) \\
&+ T_n \left(\rho^*(M)-\rho^*(M_{\O})\right),
\end{align*}
where $O$ is the number of options.
\end{corollary}

We can also show that the lower bound holds for MDPs with options as well. More precisely, it is possible to create an MDP and a set of options such that the lower bound is slightly smaller than that of Thm.~\ref{LowerBound}.

\begin{corollary}\label{cor:mdp.options.regret.lb}
Under the same assumptions as in Theorem \ref{LowerBound}, there exists an MDP with options such that the regret of any algorithm is lower-bounded as
\begin{align*} 
\Omega \left( \left( \sqrt{D_{\O}} + \sqrt{T_{\max} - T_{\min}} \right.\right.&\left.\left. \vphantom{\sqrt{D_{\O}}} \right) R_{\max}^{\O}\sqrt{S_{\O} O n }\right)\\ &+ T_n \left(\rho^*(M)-\rho^*(M_{\O})\right).
\end{align*}
\end{corollary}
\vspace{-0.1in}

This shows that MDPs with options are slightly easier to learn than SMDPs. This is due to the fact that in SMDPs resulting from MDPs with options rewards and holding times are strictly correlated (i.e., $r(s,o,s') \leq R_{\max}\tau(s,o,s')$ a.s.\ and not just in expectation for all $(s,o,s')$).

We are now ready to proceed with the comparison of the bounds on the regret of learning with options and primitive actions. We recall that for \ucrl $\Delta(M,\ucrl,s,T_n) = \wt{O}(DSR_{\max}\sqrt{AT_n})$. We first notice that\footnote{The largest per-step reward in the SMDP is defined as $R_{\max}^{\O} \geq \max_{s \in \mathcal{S},a \in \mathcal{A}} \big\lbrace \frac{\overline{r}(s,a)}{\overline{\tau}(s,a)} \big\rbrace$.} $R_{\max}^{\O} \leq R_{\max}$ and since $\mathcal{S}_{\O} \subseteq \mathcal{S}$ we have that $S_{\O} \leq S$. Furthermore, we introduce the following simplifying conditions: \textbf{1)} $\rho^*(M)=\rho^*(M_{\O})$ (i.e., the options do not prevent from learning the optimal policy), \textbf{2)} $O \leq A$ (i.e., the number of options is not larger than the number of primitive actions), \textbf{3)} options have bounded holding time (case 2 in Asm.~\ref{asm:BoundedExpectedDurationsRewards}). While in general comparing upper bounds is potentially loose, we notice that both upper-bounds are derived using similar techniques and thus they would be ``similarly'' loose and they both have almost matching worst-case lower bounds. 
Let $\mathcal{R}(M,n,\delta)$ be the ratio between the regret upper bounds of \ucrlsmdp using options $\O$ and \ucrl, then we have (up to numerical constants)
\begin{align*}
\mathcal{R}(M,n) &\leq  \frac{\big(D_{\O}\sqrt{S_{\O}} + T_{\max}\big) \sqrt{S_{\O} On \log(n/\delta)}}{ DS\sqrt{AT_n \log(T_n/\delta)}} \\
&\leq \frac{D_{\O}\sqrt{S} + T_{\max}}{D\sqrt{S}} \sqrt{\frac{n}{T_n}},
\end{align*}
where we used $n\leq T_n$ to simplify the logarithmic terms. Since $\liminf\limits_{n \rightarrow +\infty} \frac{T_n}{n} \geq \tau_{\min}$, then the previous expression gives an (asymptotic) sufficient condition for reducing the regret when using options, that is
\begin{align}\label{eq:ratio.condition}
\frac{D_{\O}\sqrt{S} + T_{\max}}{D\sqrt{S\tau_{\min}}} \leq 1.
\end{align}
In order to have a better grasp on the cases covered by this condition, let $D_{\O} = \alpha D$, with $\alpha \geq 1$. This corresponds to the case when navigating through some states becomes more difficult with options than with primitive actions, thus causing an increase in the diameter. If options are such that $T_{\max} \leq D\sqrt{S}$ and $\tau_{\min} > (1+\alpha)^2$, then it is easy to see that the condition in Eq.~\ref{eq:ratio.condition} is satisfied. This shows that even when the introduction of options partially disrupt the structure of the original MDP (i.e., $D_{\O} \geq D$), it is enough to choose options which are long enough (but not too much) to guarantee an improvement in the regret. Notice that while conditions \textbf{1)} and \textbf{2)} are indeed in favor of \ucrlsmdp, $S_{\O}$, $O$, and $T_{\max}$ are in general much smaller than $S$, $A$, $D\sqrt{S}$ ($S$ and $D$ are large in most of interesting applications). Furthermore, $\tau_{\min}$ is a very loose upper-bound on $\liminf_{n \rightarrow +\infty} \frac{T_n}{n}$ and in practice the ratio $\frac{T_n}{n}$ can take much larger values if $\tau_{\max}$ is large and many options have a high expected holding time. As a result, the set of MDPs and options on which the regret comparison is in favor of \ucrlsmdp is much wider than the one defined in Eq.~\ref{eq:ratio.condition}. Nonetheless, as illustrated in Lem.~\ref{lem:MDPwithOptions}, the case of options with bounded holding times is quite restrictive since it requires the absence of self-loops during the execution of an option. If we reproduce the same comparison in the general case of sub-exponential holding times, then the ratio between the regret upper bounds becomes
\begin{align*}
\mathcal{R}(M,n) &\leq \frac{D_{\O}\sqrt{S} + \mathcal{C}(M,n,\delta)}{D\sqrt{S}} \sqrt{\frac{n}{T_n}},\end{align*}
where $\mathcal{C}(M,n,\delta) = O(\sqrt{\log(n/\delta)})$. As a result, as $n$ increases, the ratio is always greater than 1, thus showing that in this case the regret of \ucrlsmdp is asymptotically worse than \ucrl. Whether this is an artefact of the proof or it is an intrinsic weakness of options, it remains an open question. 
 
\vspace{-0.05in}
\section{Illustrative Experiment}\label{sec:exp}
\vspace{-0.05in}

\begin{figure}[t]
\centerline{
\includegraphics[scale=0.5]{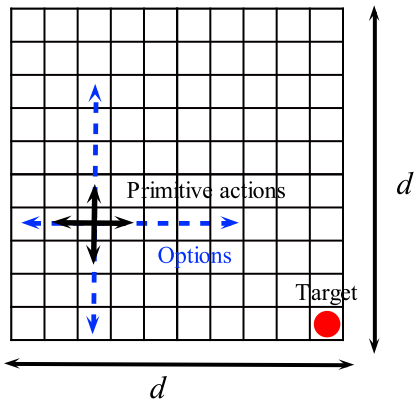}
}
\vspace{-0.1in}
\caption{\small Navigation problem.}\label{fig:navigation}
\vspace{-0.2in}
\end{figure}

\begin{figure*}
\begin{subfigure}{.33\textwidth}
  \centering
  \includegraphics[width=1.1\textwidth]{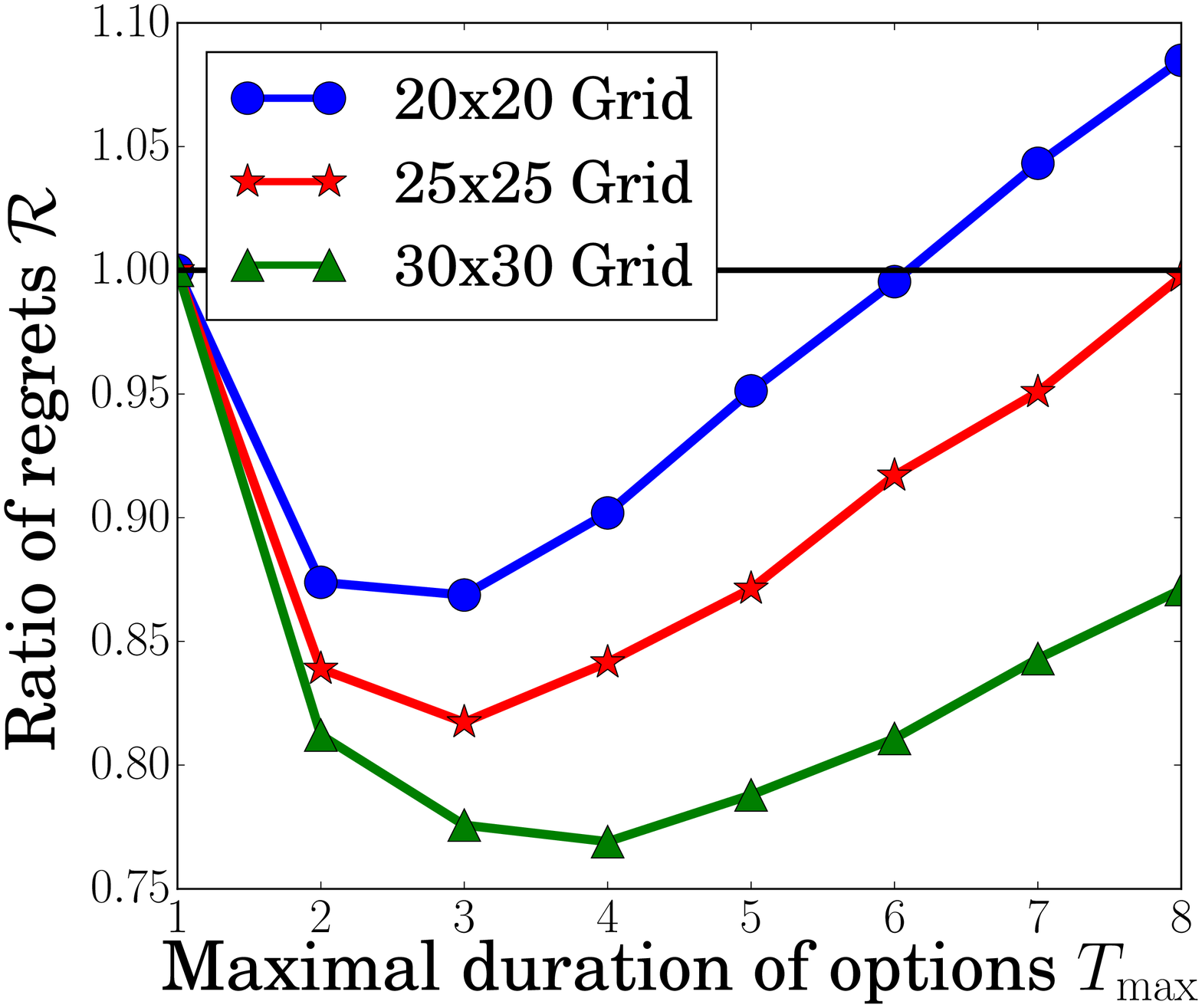}
  \caption{}
  \label{fig:results.ratio}
\end{subfigure}
\begin{subfigure}{.33\textwidth}
  \centering
  \includegraphics[width=1.1\textwidth]{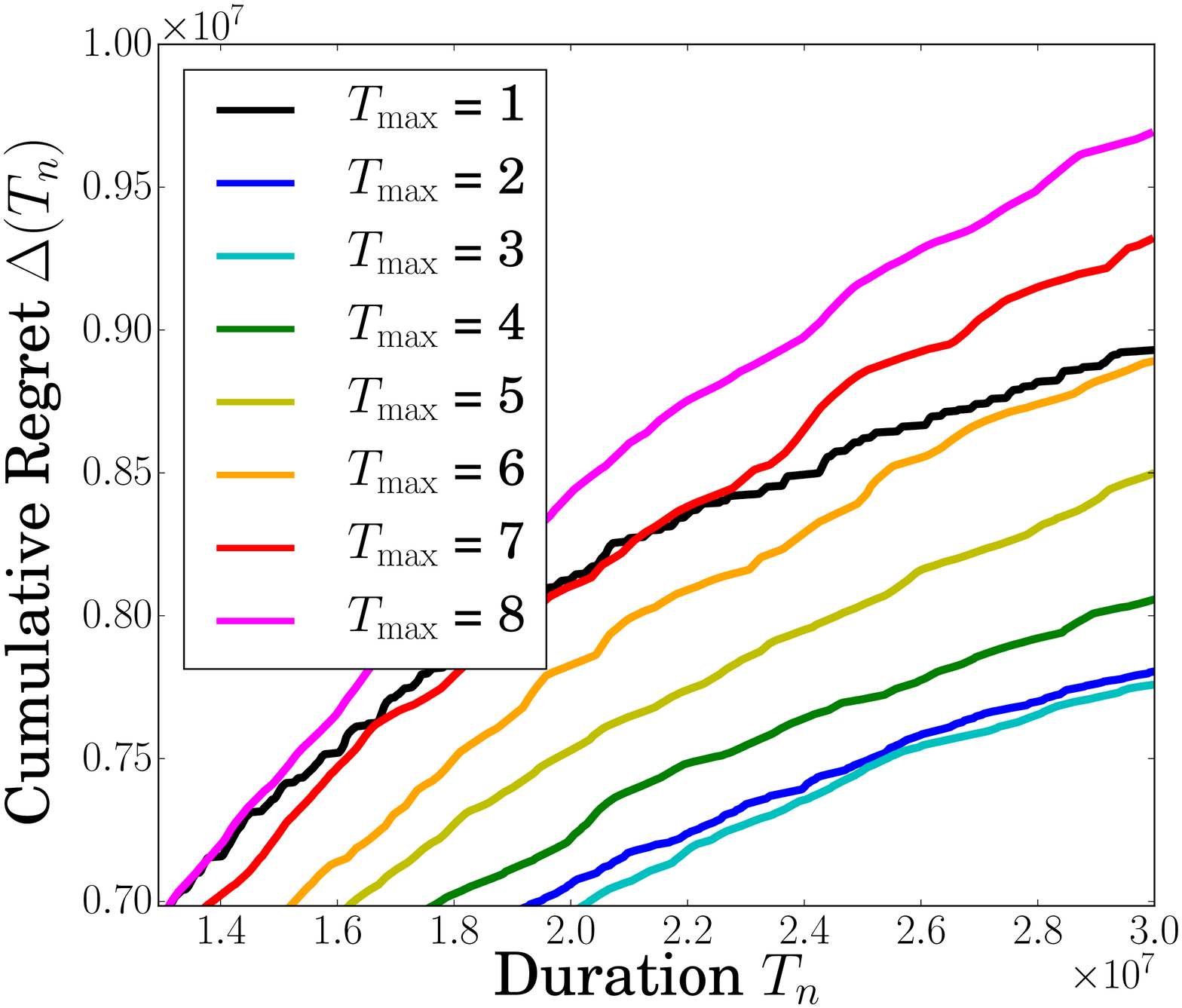}
  \caption{}
  \label{fig:results.regret}
\end{subfigure}
\begin{subfigure}{.33\textwidth}
  \centering
  \includegraphics[width=1.1\textwidth]{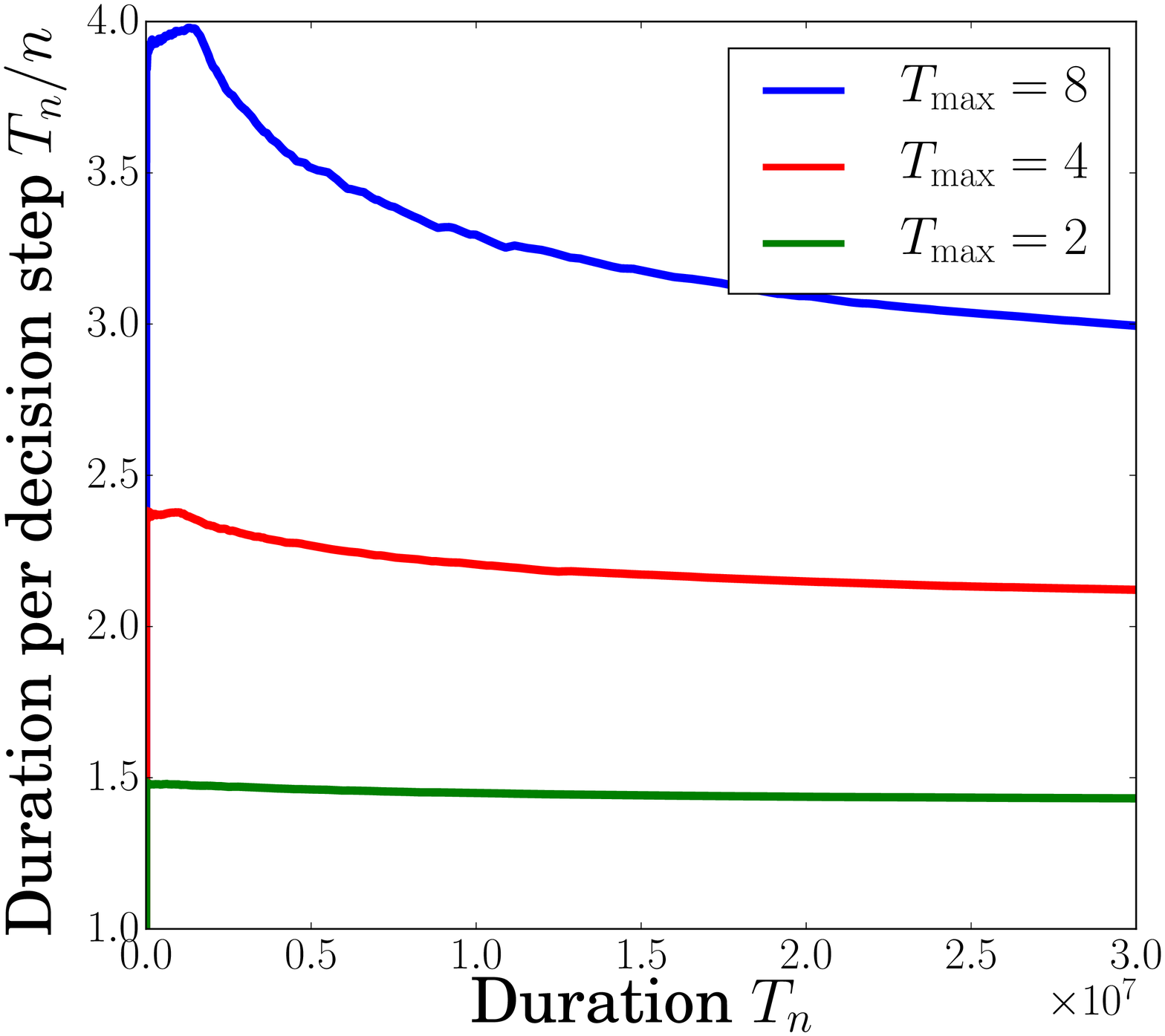}
  \caption{}
  \label{fig:results.duration}
\end{subfigure}
\vspace{-0.1in}
\caption{\textit{(a)} Ratio of the regrets with and without options for different values of $T_{\max}$; \textit{(b)} Regret as a function of $T_{n}$ for a 20x20 grid; \textit{(c)} Evolution of $T_{n}/n$ for a 20x20 grid.}
\label{fig:results}
\vspace{-0.15in}
\end{figure*}

We consider the navigation problem in Fig.~\ref{fig:navigation}.
In any of the $d^2$ states of the grid except the target, the four cardinal actions are available, each of them being successful with probability $1$. If the agent hits a wall then it stays in its current position with probability $1$. When the target state is reached, the state is reset to any other state with uniform probability. The reward of any transition is 0 except when the agent leaves the target in which case it equals $R_{\max}$. The optimal policy simply takes the shortest path from any state to the target state. The diameter of the MDP is the longest shortest path in the grid, that is $D = 2d-2$. Let $m$ be any non-negative integer smaller than $d$ and in every state but the target we define four macro-actions: \textit{LEFT}, \textit{RIGHT}, \textit{UP} and \textit{DOWN} (blue arrows in the figure). When \textit{LEFT} is taken, primitive action \textit{left} is applied up to $m$ times (similar for the other three options). For any state $s'$ which is $k\leq m$ steps on the left of the starting state $s$, we set $\beta_o(s')=1/(m-k+1)$ so that the probability of the option to be interrupted after any $k \leq m$ steps is $1/m$. If the starting state $s$ is $l$ steps close to the left border with $l<m$ then we set $\beta_o(s')=1/(l-k+1)$ for any state $s'$ which is $k\leq l$ steps on the left. As a result, for all options started $m$ steps far from any wall, $T_{\max} = m$ and the expected duration is $\tau := \tau(s,o) = (m+1)/2$, which reduces to $T_{\max} = l$ and $\tau =(l+1)/2$ for an option started \ $l < m$ step from the wall and moving towards it. More precisely, all options have an expected duration of $\tau(s,o) = \tau$ in all but in $m d$ states, which is small compared to the total number of $d^2$ states. The SMDP formed with this set of options preserves the number of state-action pairs $S_{\O} = S = d^2$ and $A'=A=4$ and the optimal average reward $\rho^*(M) = \rho^*(M')$, while it slightly perturbs the diameter $D_{\O} \leq D + m(m+1)$ (see App.~\ref{app:experiments} for further details). Thus, the two problems seem to be as hard to learn. However the (asymptotic) ratio between the regret upper bounds becomes
\begin{align*}
\lim_{n \rightarrow \infty} &\mathcal{R}(M,n) \leq \frac{(2d-2+ m^2 + m)d + m}{(2d-2)d} \left(\lim_{n \rightarrow \infty}{\sqrt{\frac{n}{T_n}}}\right)\\
&\leq \left(1+ 2\frac{m^2}{d}\right)\left(\lim_{n \rightarrow \infty}{\sqrt{\frac{n}{T_n}}}\right),
\end{align*}
where we assume $m, d \geq 2$. While a rigorous analysis of the ratio between the number of option decision steps $n$ and number of primitive actions $T_n$ is difficult, we notice that as $d$ increases w.r.t.\ $m$, the chance of executing options close to a wall decreases, since for any option only $m d$ out of $d^2$ states will lead to a duration smaller than $\tau$ and thus we can conclude that $n/T_n$ tends to $1/\tau = 2/(m+1)$ as $n$ and $d$ grow. As a result, the ratio would reduce to $(1+2m^2/d)\sqrt{2/(m+1)}$ that is smaller than 1 for a wide range of values for $m$ and $d$. Finally, the ratio is (asymptotically in $d$) minimized by $m \approx \sqrt{d}$, which gives $\mathcal{R}(M,n) = O(d^{-1/4})$, thus showing that as $d$ increases there is always an appropriate choice of $m$ for which learning with options becomes significantly better than learning with primitive actions. In Fig.~\ref{fig:results.ratio} we empirically validate this finding by studying the ratio between the actual regrets (and not their upper-bounds) as $d$ and $m$ (i.e., $T_{\max}$) vary, and with a fixed value of $T_n$ that is chosen big enough for every $d$. As expected, for a fixed value of $d$, the ratio $\mathcal{R}$ first decreases as $m$ increases, reaches a minimum and starts increasing to eventually exceed 1. As $d$ increases, the value of the minimum decreases, while the optimal choice of $m$ increases. This behaviour matches the theory, which suggests that the optimal choice for $m$ increases as $O(\sqrt{d})$. In Fig.~\ref{fig:results.regret} we report the cumulative regret and we observe that high values of $T_{\max}$ worsen the learning performances w.r.t. learning without options ($T_{\max} = 1$, plotted in black). Finally, Fig.~\ref{fig:results.duration} shows that, as $n$ tends to infinity, $T_n/n$ tends to converge to $(m+1)/2$ when $m \ll d$, whereas it converges to slightly smaller values when $m$ is close to $d$ because of the truncations operated by walls.

\textbf{Discussion.} Despite its simplicity, the most interesting aspect of this example is that the improvement on the regret is not obtained by trivially reducing the number of state-action pairs, but it is intrinsic in the way options change the dynamics of the exploration process. The two key elements in designing a successful set of options $\O$ is to preserve the average reward of the optimal policy and the diameter. The former is often a weaker condition than the latter. In this example, we achieved both conditions by designing a set $\O$ where the termination conditions allow any option to end after only one step. This preserves the diameter of the original MDP (up to an additive constant), since the agent can still navigate at the level of granularity of primitive actions. Consider a slightly different set of options $\O'$, where each option moves exactly by $m$ steps (no intermediate interruption). The number of steps to the target remains unchanged from any state and thus we can achieve the optimal performance. Nonetheless, having $\pi^*$ in the set of policies that can be represented with $\O'$ does not guarantee that the \ucrlsmdp would be as efficient in learning the optimal policy as \ucrl. In fact, the expected number of steps needed to go from a state $s$ to an adjacent state $s'$ may significantly increase. Despite being only one primitive action apart, there may be no sequence of options that allows to reach $s'$ from $s$ without relying on the random restart triggered by the target state. A careful analysis of this case shows that the diameter is as large as $D_{\O'} = D(1+m^2)$ and there exists no value of $m$ that satisfies Eq.~\ref{eq:ratio.condition} (see App.~\ref{app:experiments}).
 
\vspace{-0.1in}
\section{Conclusions}\label{sec:conclusions}
\vspace{-0.1in}

We derived upper and lower bounds on the regret of learning in SMDPs and we showed how these results apply to learning with options in MDPs. Comparing the regret bounds of \ucrlsmdp with \ucrl, we provided sufficient conditions on the set of options and the MDP (i.e., similar diameter and average reward) to reduce the regret w.r.t.\ learning with primitive actions. To the best of our knowledge, this is the first attempt of explaining when and how options affect the learning performance. Nonetheless, we believe that this result leaves space for improvements. In fact, Prop.~\ref{prop:smdp} implies that the class of SMDPs is a strict superset of MDPs with options. This suggests that a more effective analysis could be done by leveraging the specific structure of MDPs with options rather than moving to the more general model of SMDPs. This may actually remove the additional $\sqrt{\log(n/\delta)}$ factor appearing because of sub-exponential distributions in the \ucrlsmdp regret. An interesting direction of research is to use this theoretical result to provide a more explicit and quantitative objective function for option discovery, in the line of what is done in~\citep{brunskill2014pac-inspired}. Finally, it would be interesting to extend the current analysis to more sophisticated hierarchical approaches to RL such as MAXQ~\citep{dietterich2000hierarchical}.

{\small
\textbf{Acknowledgement} This research is supported in part by a grant from CPER Nord-Pas de Calais/FEDER DATA Advanced data science and technologies 2015-2020, CRIStAL (Centre de Recherche en Informatique et Automatique de Lille), and the French National Research Agency (ANR) under project ExTra-Learn n.ANR-14-CE24-0010-01.
}
 
\newpage
\begin{small}

\bibliographystyle{plainnat}
\end{small}

\newpage
\appendix
\onecolumn

\section{Optimal average reward in discrete and continuous SMDPs: existence and computation}\label{GainSMDPs}

In this section we prove Proposition~\ref{prop:optimal.policy} and Lemma~\ref{lem:validity.evi}. Since extended value iteration is run on SMDP $ \widetilde{M}_k^+$ with continuous actions (the choice of the transition probability), in the following we consider both the continuous and discrete case at the same time. In order to have a more rigorous treatment of SMDPs, we introduce further notations from~\citep{Puterman94}.

A decision rule is a function $d: \mathcal{H} \rightarrow \Delta(\mathcal{A})$ where $\mathcal{H}$ is the set of possible histories and $\Delta(\mathcal{A})$ is the set of probability distributions over $\mathcal{A}$. For an SMDP $M$, we will denote by $D_M^{HR}$ the set of history-dependent randomized decision rules and $D_M^{MD}$ the subset of Markovian deterministic decision rules ($D_M^{MD} \subset D_M^{HR}$). A history-dependent randomized policy is a sequence of elements of $D_M^{HR}$ indexed by the decision steps, i.e., $\pi = (d_1, d_2, ...) \in (D_M^{HR})^{\mathbb{N}}$, while a stationary deterministic policy is a constant sequence of elements of $D_M^{MD}$: $\pi = (d, d, ...) = d^\infty$. The set of history-dependent randomized policies will be denoted $\Pi_M^{HR}$ and the subset of stationary deterministic policies will be denoted $\Pi_M^{SD}$: $\Pi_M^{SD} \subset \Pi_M^{HR}$. We also consider the more general case where the set of available actions may depend on the state, i.e., there exists a set $\A_s$ for each $s\in\calS$.

\subsection{Optimality criterion}

We start by defining the optimality criterion in SMDPs. Unlike for MDPs, where the average reward of a fixed policy is uniquely defined, in SMDPs there are three different definitions that are usually encountered in the literature (see~\citealt{Schal92}, \citealt{Federgruen83}, and \citealt{Ross70}).\footnote{Notice that the definition we provide in Eq.~\ref{eq:rho} of Prop.~\ref{prop:optimal.policy} is $\rho_1$.}

\begin{definition}\label{AverageReward}
Denote $N(t) = \sup \bigg\{n : n \in \mathbb{N}, \ \sum_{i=1}^{n}{\tau_i} \leq t \bigg\}$ the number of decision steps that occured before time $t$. For any $\pi \in \Pi_M^{HR}$ and $s\in \mathcal{S}$, we define:
\begin{align}\label{eq5}
\overline{\rho_1}^\pi(s) = \limsup_{t \rightarrow + \infty} \mathbb{E}^\pi \bigg[ \frac{\sum_{i=1}^{N(t)}{r_i}}{t} \bigg| s_0 =s \bigg], \enspace
\underline{\rho_1}^\pi(s) = \liminf_{t \rightarrow + \infty} \mathbb{E}^\pi \bigg[ \frac{\sum_{i=1}^{N(t)}{r_i}}{t} \bigg| s_0 =s \bigg]
\end{align}
\begin{align}\label{eq6}
\begin{split}
\overline{\rho_2}^\pi(s) = \limsup_{n \rightarrow + \infty}  \frac{\mathbb{E}^\pi \big[ \sum_{i=1}^{n}{r_i} \big| s_0 =s \big]}{\mathbb{E}^\pi \big[ \sum_{i=1}^{n}{\tau_i} \big| s_0 =s \big]}, \enspace
\underline{\rho_2}^\pi(s) = \liminf_{n \rightarrow + \infty}  \frac{\mathbb{E}^\pi \big[\sum_{i=1}^{n}{r_i} \big| s_0 =s \big]}{\mathbb{E}^\pi \big[ \sum_{i=1}^{n}{\tau_i} \big| s_0 =s \big]}
\end{split}
\end{align}
and for any $d^\infty \in \Pi_M^{SD}$ and $s\in \mathcal{S}$ we define:
\begin{align}\label{eq7}
\rho_3^{d^\infty}(s) = \sum_{\alpha = 1}^{\nu(d)}{p^d(\alpha|s) g_{d}(\alpha)}, \quad\text{with}\quad g_{d}(\alpha) = \frac{\sum_{s \in R_{\alpha}^d}{\mu_{\alpha}^d(s)\overline{r}(s,d(s))}}{\sum_{s \in R_{\alpha}^d}{\mu_{\alpha}^d(s)\overline{\tau}(s,d(s))}}
\end{align}
where $\nu(d)$ is the number of positive recurrent classes under $d^\infty$, $p^d(\alpha|s)$ is the probability of entering positive recurrent class $\alpha$ starting from $s$ and following policy $d^\infty$, $R_{\alpha}^d$ is the set of states of class $\alpha$ and $\mu_{\alpha}^d$ is the stationary probability distribution of class $\alpha$.
\end{definition}

\subsection{Proof of Proposition~\ref{prop:optimal.policy}}

We say that $(d^*)^\infty$ is (3)-average-optimal if for all states $s \in \mathcal{S}$ and decision rules $d \in D_M^{MD}$, $\rho_3^{(d^*)^\infty}(s) \geq \rho_3^{d^\infty}(s)$. We say that $\pi^*$ is (1)-average-optimal (respectively (2)-average-optimal) if for all states $s \in \mathcal{S}$ and all $\pi \in \Pi_M^{HR}$, $\underline{\rho_1}^{\pi^*}(s) \geq \overline{\rho_1}^\pi(s)$ (respectively $ \underline{\rho_2}^{\pi^*}(s) \geq \overline{\rho_2}^\pi(s)$). We prove a slightly more general formulation than Proposition~\ref{prop:optimal.policy}.

\begin{proposition}\label{prop:EquivalenceGains}
If $M$ is communicating and the expected holding times and rewards are finite, then
\begin{itemize}
\item There exists a stationary policy $\pi^* = (d^*)^\infty$ which is (1,2,3)-average-optimal.
\item All optimal values are equal and constant and we will denote this value by $\rho^*$:
$$\forall s \in \mathcal{S}, \ \overline{\rho_1}^{(d^*)^\infty}(s) = \underline{\rho_1}^{(d^*)^\infty}(s) = \overline{\rho_2}^{(d^*)^\infty}(s) = \underline{\rho_2}^{(d^*)^\infty}(s) = \rho_3^{(d^*)^\infty}(s) = \rho^*.$$ 
\end{itemize}
\end{proposition}

\begin{proof}

\textbf{Step 1: Optimality equation of a communicating SMDP.}
We first recall the average reward optimality equations for a communicating SMDP (Eq.~\ref{eq:optimality.eq})
\begin{align}\label{eq21}
\forall s \in \mathcal{S}, \ u^*(s) = \max_{a \in \mathcal{A}_s} \left\lbrace \overline{r}(s,a) - \rho^* \overline{\tau}(s,a) + \sum_{s' \in \mathcal{S}}{p(s'|s, a)}u^*(s') \right\rbrace
\end{align}
where $u^*$ and $\rho^*$ are the bias (up to additive constant) and average reward respectively. Since we need to analyse both the case where $\mathcal{A}_s$ is finite and the case where $\mathcal{A}_s$ is continuous, we verify that it is appropriate to consider the $\max$ instead of $\sup$ in the previous expression. For the original SMDP $M$, $\mathcal{A}_s$ is finite and the maximum is well-defined. For the extended SMDPs $\widetilde{M}_k^+$ considered while computing the optimistic SMDP, $\widetilde{\mathcal{A}}_{k,s}^+$ is compact and $\overline{r}(s,a)$, $\overline{\tau}(s,a)$ and $p(\cdotp \mid s, a)$ are continuous in $\widetilde{\mathcal{A}}_{k,s}^+$ by the very definition of $\widetilde{M}_k^+$. The function $\overline{r}(s,a) - \rho^* \overline{\tau}(s,a) + \sum_{s' \in \mathcal{S}}{p(s'|s, a)}u^*(s')$ is thus continuous on $\widetilde{\mathcal{A}}_{k,s}^+$ compact and by Weierstrass theorem, we know that the maximum is reached (i.e., there exists a maximizer). As a result, Eq.~\ref{eq21} is well-defined and we can study the existence and properties of its solutions.

\textbf{Step 2: Data-transformation (uniformization) of an SMDP.}
The structure of EVI is based on a data-transformation (also called "uniformization") which turns the initial SMDP $M$ into an ``equivalent'' MDP $M_{\Meq} = \big\{ \mathcal{S}, \mathcal{A}, p_{\Meq}, r_{\Meq} \big\}$ defined as in Eq.~\ref{eq:eq.MDP}. As a result, we can just apply standard MDP theory to the equivalent MDP. The average optimality equation of $M_{\Meq}$ is~\citep{Puterman94}
\begin{align}\label{eq23}
\forall s \in \mathcal{S}, \ v^*(s) = \max_{a \in \mathcal{A}_s} \left\lbrace \frac{\overline{r}(s,a)}{\overline{\tau}(s,a)} - g^* + \frac{\tau}{\overline{\tau}(s,a)}\sum_{s' \in \mathcal{S}}{p(s'|s, a)}v^*(s') + \left( 1 - \frac{\tau}{\overline{\tau}(s,a)} \right) v^*(s) \right\rbrace 
\end{align}
Since $\tau < \tau_{\min}$, every Markov Chain induced by a stationary deterministic policy on $M_{\Meq}$ is necessarily aperiodic (for any action, the probability of any state to loop on itself is strictly positive). Moreover, since $M$ is assumed to be communicating, $M_{\Meq}$ is also communicating. The same holds for $\widetilde{M}_{k,\Meq}^+$ (i.e., the MDP obtained from the extended SMDP $\widetilde{M}_k^+$ after data transformation). Under these conditions, Eq.~\ref{eq23} has a solution $(v^*, g^*)$ where $g^*$ is the optimal average reward of $M_{\Meq}$ (respectively $\widetilde{M}_{k,\Meq}^+$) and the (stationary deterministic) greedy policy w.r.t.\ $v^*$ is average-optimal. Moreover, standard value iteration is guaranteed to converge and it can be applied with the stopping condition in Eq.~\ref{eq:stopping} to obtain an $\epsilon$-optimal policy in finitely many steps. This holds for both finite and compact $\mathcal{A}_s$ with continuous $\overline{r}_{\Meq}(s,a)$ and $p_{\Meq}(s' | s,a)$ (see for example \citealt{Puterman94} and \citealt{Leizarowitz13}). It is easy to show that EVI in Eq.~\ref{eq:evi} is exactly value iteration applied to $\wt{M}_{k,\Meq}^+$. Finally, Lemma 2 of~\citep{Federgruen83} (Prop. \ref{lem:equivalence.uniformization}) shows the ``equivalence'' between $M$ and $M_{\Meq}$ (respectively $\widetilde{M}_k^+$ and $\widetilde{M}_{k,\Meq}^+$): if $(v^*, g^*)$ is a solution to Eq.~\ref{eq23}, then $(\tau^ {-1}v^*, g^*)$ is a solution to Eq.~\ref{eq21} and conversely. As a result, there exists a solution $(u^*, \rho^*)$ to Eq.~\ref{eq21} for both $M$ and $\widetilde{M}_{k}^+$.

\textbf{Step 3: Existence of deterministic stationary optimal policy.}
We are now ready to prove the existence of a deterministic stationary policy that is (1,2,3)-optimal and that the corresponding optimal value is constant and equal in all three cases. We consider the case of finite and continuous $\A_s$ separately.

\textbf{Step 3a: For $M$ (finite $\mathcal{A}_s$).}
Since conditions (L), (F) and (R) of~\citep{Schal92} hold, we can apply their main theorem and obtain that
\begin{enumerate}
\item Any greedy policy $\left(d^* \right)^\infty$ w.r.t. $u^*$ is such that $\ \overline{\rho_1}^{(d^*)^\infty}(s) \geq \overline{\rho_1}^{\pi}(s)$ for any $\pi \in \Pi_M^{HR}$ and any $s \in \mathcal{S}$,
\item $\forall s \in \mathcal{S}, \ \overline{\rho_1}^{(d^*)^\infty}(s) = \rho^*$,
\end{enumerate}
where $(u^*, \rho^*)$ is a solution of Eq.~\ref{eq21}. Furthermore, from renewal theory (see e.g.,~\citealt{Tijms03} and~\citep{Ross70}) we have that $\forall d^\infty \in \Pi_M^{SD}, \ \ \overline{\rho_1}^{d^\infty} = \underline{\rho_1}^{d^\infty} =\rho_1^{d^\infty}$ (i.e., the limit exists for deterministic stationary policies), thus we can conclude that $\left(d^* \right)^\infty$ is (1)-optimal. Furthermore, by Lemma 2.7 of \citep{Schal92}: $\forall d^\infty \in \Pi_M^{SD}, \ \forall s \in \mathcal{S}, \ \overline{\rho_1}^{d^\infty}(s) = \rho_3^{d^\infty}(s)$ so $d^*$ is also necessarily (3)-optimal. Finally, by Theorem 7.6 of \citep{Ross70}, since $(u^*, \rho^*)$ is a solution of Eq.~\ref{eq21}:
\begin{enumerate}
\item Any greedy policy $\left(d^* \right)^\infty$ w.r.t. $u^*$ is such that $\ \overline{\rho_2}^{(d^*)^\infty}(s) \geq \overline{\rho_2}^{\pi}(s)$ for any $\pi \in \Pi_M^{HR}$ and any $s \in \mathcal{S}$,
\item $\forall s \in \mathcal{S}, \overline{\rho_2}^{(d^*)^\infty}(s) = \rho^*$.
\end{enumerate}
By Theorem 11.4.1 of \citep{Puterman94} we have $\forall d^\infty \in \Pi_M^{SD}, \ \overline{\rho_2} = \underline{\rho_2} = \rho_2^{d^\infty}$ and thus $d^*$ is also (2)-optimal. This concludes the proof for the finite case, which proves the statement of Prop.~\ref{prop:optimal.policy}.

\textbf{Step 3b: For $\widetilde{M}_k^+$ (compact $\widetilde{\mathcal{A}}_{k,s}^+$ with continuous rewards, holding times and transition probabilities).}
The proof is almost the same as with discrete action spaces. The only difference is that we can't apply the Theorem of \citep{Schal92} because conditions (R) and (C*) do not hold in general. However, we can use Propositions 5.4 and 5.5 of \citep{Schal92}  and we have the same result as in the discrete case (assumptions (L), (C), (P) and (I) hold in our case and we know that the optimality equation \ref{eq21} admits a solution $(u^*, \rho^*)$, see above). Since the state space is finite, the rest of the proof is rigorously the same (the Theorems and Lemmas still applies). This guarantees the same statement as Prop.~\ref{prop:optimal.policy} but for the optimistic SMDP $\wt{M}_k^+$.
\end{proof}

\subsection{Proof of Lemma~\ref{lem:validity.evi}}

\begin{proof}
From the proof of Prop.~\ref{prop:optimal.policy}, we already have that EVI converges towards the optimal average reward of $\widetilde{M}_{k,\Meq}^+$, which is also the optimal average reward of $\widetilde{M}_k^+$. We also know that the stopping criterion is met in a finite number of steps and that the greedy policy when the stopping criterion holds is $\epsilon$-optimal in the equivalent $\widetilde{M}_{k,\Meq}^+$. Then, in order to prove Lemma~\ref{lem:validity.evi}, we only need to prove that this policy is also $\epsilon$-optimal in the optimistic SMDP $\widetilde{M}_k^+$. \citet{Tijms03} shows that for any stationary deterministic policy $d^\infty \in \Pi_{\widetilde{M}_k^+}^{SD}$, the (1)-average reward is the same in the SMDP and the MDP obtained by uniformization, that is
\begin{align*}
\forall s \in \mathcal{S}, \rho_1^{d^\infty}(\widetilde{M}_k^+) = \rho^{d^\infty}(\widetilde{M}_{k,\Meq}^+).
\end{align*}
Then it immediately follows that the policy returned by EVI is (1)-$\epsilon$-optimal in $\widetilde{M}_k^+$ and since $\forall d^\infty \in \Pi_M^{SD}, \ \rho_1^{d^\infty} = \rho_3^{d^\infty}$, it is also (3)-$\epsilon$-optimal. Note that for any deterministic stationary policy $d \in \Pi_{\widetilde{M}_k^+}^{SD}$ defining a unichain Markov chain in $\widetilde{M}_{k,\Meq}^+$ (or equivalently in $\widetilde{M}_k^+$), we have: $\forall s \in \mathcal{S}, \rho_1^{d^\infty}(s) = \rho_2^{d^\infty}(s)$ and this value is constant across states (see for example chapter 11 of \citep{Puterman94}, Theorem 7.5 of \citep{Ross70} or \citep{JianyongX04}). However, in the general case, this equality does not hold (see Example 2.1 of \citep{JianyongX04}). Nevertheless, by Theorem 3.1 of \citep{JianyongX04} we have
\begin{align}\label{eq24}
\begin{split}
\forall d^\infty \in \Pi_{\widetilde{M}_k^+}^{SD},\ \forall s \in \mathcal{S},\ \big|\rho_2^{d^\infty}(\widetilde{M}_k^+, s) - \rho^{d^\infty}(\widetilde{M}_{k,\Meq}^+, s) \big| \leq \rho_{\max}^{d^\infty}(\widetilde{M}_{k,\Meq}^+) - \rho_{\min}^{d^\infty}(\widetilde{M}_{k,\Meq}^+)\\
\mbox{where } \ \rho_{\max}^{d^\infty}(\widetilde{M}_{k,\Meq}^+) = \max_{s \in \mathcal{S}}\rho^{d^\infty}(\widetilde{M}_{k,\Meq}^+, s) \ \mbox{ and } \ \rho_{\min}^{d^\infty}(\widetilde{M}_{k,\Meq}^+) = \min_{s \in \mathcal{S}}\rho^{d^\infty}(\widetilde{M}_{k,\Meq}^+, s).
\end{split}
\end{align}
If we denote by $d$ the policy returned by EVI and $\rho^*$ the optimal gain of $\widetilde{M}_{k,\Meq}^+$ and $\widetilde{M}_k^+$ we obtain
\begin{align*}
\forall s \in \mathcal{S},\ \ \rho^* - \rho_2^{d^\infty}(\widetilde{M}_k^+, s)  &= \rho^{d^\infty}(\widetilde{M}_{k,\Meq}^+, s)  -\rho_2^{d^\infty}(\widetilde{M}_k^+, s) + \rho^ * - \rho^{d^\infty}(\widetilde{M}_{k,\Meq}^+, s) \\
&\leq \rho_{\max}^{d^\infty}(\widetilde{M}_{k,\Meq}^+) - \rho_{\min}^{d^\infty}(\widetilde{M}_{k,\Meq}^+) + \epsilon \ \\
&= \rho_{\max}^{d^\infty}(\widetilde{M}_{k,\Meq}^+) - \rho^* + \rho^* - \rho_{\min}^{d^\infty}(\widetilde{M}_{k,\Meq}^+) + \epsilon \leq 2\epsilon.
\end{align*}
For the first inequality we used Eq.~\ref{eq24} and the fact that $d$ is $\epsilon$-optimal in $\widetilde{M}_{k,\Meq}^+$. For the second inequality, we used again that $d$ is $\epsilon$-optimal in $\widetilde{M}_{k,\Meq}^+$ and we also used the fact that $\rho_{\max}^{d^\infty}(\widetilde{M}_{k,\Meq}^+) \leq \rho^*$. In conclusion, the policy returned by EVI is (2)-$2\epsilon$-optimal. The remaining part of Theorem~\ref{lem:validity.evi} is thus proved for all optimality criteria.
\end{proof}

By Theorem 8.3.2 of \citep{Puterman94}, we know that there exists an optimal policy $\widetilde{d}^*$ of MDP $\widetilde{M}_{k,\Meq}^+$ that yields a unichain Markov Chain (i.e., a Markov Chain with a single positive recurrent class). The Markov Chain induced by $\widetilde{d}^*$ in $\widetilde{M}_k^+$ is thus also unichain and moreover: $\rho_1^{(\widetilde{d}^*)^\infty}(\widetilde{M}_k^+) = \rho^{(\widetilde{d}^*)^\infty}(\widetilde{M}_{k,\Meq}^+) = \rho^*(\widetilde{M}_{k,\Meq}^+) = \rho^*(\widetilde{M}_k^+)$. We have seen that for any policy $d \in \Pi_{\widetilde{M}_k^+}^{SD}$ yielding a unichain Markov Chain $\rho_1^{d^\infty}(\widetilde{M}_k^+) = \rho_2^{d^\infty}(\widetilde{M}_k^+)$ and so in particular, it is true for $\widetilde{d}^*$. Therefore, there exists a policy of $\widetilde{M}_k^+$ which yields a unichain Markov Chain and which is (1)-optimal, (2)-optimal and (3)-optimal. This explains why the optimal gain is the same for criteria (1) and (2) but EVI must be run with a different accuracy to insure $\epsilon$-accuracy (the Markov Chain induced by the policy returned by EVI is not necessarily unichain).

\section{Analysis of SMDP-UCRL (proof of Theorem \ref{UpperBound})}\label{app:proof.smdp.ucrl}

The proof follows the same steps as in~\citep{Jaksch10}. Therefore, in the following we only emphasize the differences between SMDPs and MDPs and we refer to~\citep{Jaksch10} for the parts of the proof which are similar.

\subsection{Splitting into Episodes}

We first recall the definition of sub-exponential random variables.

\begin{definition}[\cite{Wainwright15}]\label{DefSubExp}
A random variable $X$ with mean $\mu < +\infty$ is said to be sub-exponential, if one of the following equivalent conditions is satisfied:
\begin{enumerate}
\item (Laplace transform condition) There exists $(\sigma, b) \in \mathbb{R}^{+} \times \mathbb{R}^{+*}$ such that:
\begin{equation}\label{eq1}
\mathbb{E} [ e^{\lambda (X-\mu)} ] \leq e^{\frac{\sigma^2 \lambda^2}{2}} \ \textrm{  for all  } \ |\lambda| < \frac{1}{b}.
\end{equation}
In this case, we say that $X$ is sub-exponential of parameters $\sigma, b$ and we denote it by $X \in \subExp(\sigma,b)$.
\item There exists $c_0 >0$ such that $\mathbb{E} [ e^{\lambda (X-\mu)} ] < + \infty \ $  for all  $|\lambda| \leq c_0$.
\end{enumerate}
\end{definition}

In order to define the confidence intervals, we use the Bernstein concentration inequality for sub-exponential random variables.

\begin{proposition}[Bernstein inequality, \citep{Wainwright15}] \label{BernsteinSubExponential}
Let $(X_i)_{1 \leq i \leq n}$ be a collection of independent sub-Exponential random variables s.t. $\forall i \in \lbrace 1, ..., n \rbrace, \ X_i \in \subExp(\sigma_i,b_i) $ and $\mathbb{E}[X_i] = \mu_i$. We have the following concentration inequalities:
\begin{align}\label{eq15}
\begin{split}
\forall t \geq 0, \
\mathbb{P} \left( \sum_{i=1}^{n}{X_i} - \sum_{i=1}^{n}{\mu_i} \geq t \right) &\leq 
\begin{cases} e^{- \frac{t^2}{2 n \sigma^2}}, & \mbox{if } \ 0 \leq t \leq \frac{\sigma^2}{b} \\ 
e^{- \frac{t}{2b}}, & \mbox{if } \ t > \frac{\sigma^2}{b} 
\end{cases}\\
\mathbb{P} \left( \sum_{i=1}^{n}{X_i} + \sum_{i=1}^{n}{\mu_i} \leq t \right) &\leq 
\begin{cases} e^{- \frac{t^2}{2 n \sigma^2}}, & \mbox{if } \ 0 \leq t \leq \frac{\sigma^2}{b} \\ 
e^{- \frac{t}{2b}}, & \mbox{if } \ t > \frac{\sigma^2}{b} 
\end{cases}
\end{split}
\end{align}
where $\sigma = \sqrt{\frac{\sum_{i=1}^{n}{\sigma_i^2}}{n}}$ and $b = \max_{1 \leq i \leq n} \lbrace b_i \rbrace$.
\end{proposition}

Denoting by $N(s,a)$ the state-action counts we have
\begin{align*}
\sum_{i=1}^{n}{r_i(s_{i-1},a_{i-1})} = \sum_{s \in \mathcal{S}} \sum_{a \in \mathcal{A}_s} \sum_{j=1}^{N(s,a)}{r_{k_j}(s,a)}.
\end{align*}
Conditionally on knowing $(N(s,a))_{s,a}$, the previous sum is equal (in distribution) to a sum of independent random variables with mean $\sum_{s \in \mathcal{S}} \sum_{a \in \mathcal{A}_s} N(s,a) \overline{r}(s,a)$ and from Prop.~\ref{BernsteinSubExponential} we have
\begin{align*}
\mathbb{P} \left( \sum_{i=1}^{n}{r_i} \leq \sum_{s \in \mathcal{S}} \sum_{a \in \mathcal{A}_s} N(s,a) \overline{r}(s,a) - \sigma_r \sqrt{\frac{5}{2}n\log \left( \frac{13n}{\delta}\right)} \Bigg| \left(N(s,a)\right)_{s,a} \right) &\leq \left(\frac{\delta}{13n} \right)^{5/4} \leq \frac{\delta}{24n^{5/4}}, \\
& \mbox{if } n \geq \frac{5b_r^2}{2\sigma_r^2} \log \left( \frac{13n}{\delta} \right)\\
\mathbb{P} \left( \sum_{i=1}^{n}{r_i} \leq \sum_{s \in \mathcal{S}} \sum_{a \in \mathcal{A}_s} N(s,a) \overline{r}(s,a) - \frac{5}{2}b_r\log \left( \frac{13n}{\delta}\right) \Bigg| \left(N(s,a)\right)_{s,a} \right) &\leq \left(\frac{\delta}{13n} \right)^{5/4} \leq \frac{\delta}{24n^{5/4}}, \\
&\mbox{if } n \leq \frac{5b_r^2}{2\sigma_r^2} \log \left( \frac{13n}{\delta} \right)\\
\end{align*}
Similarly, the total holding time satisfies
\begin{align*}
\mathbb{P} \left( \sum_{i=1}^{n}{\tau_i} \geq \sum_{s \in \mathcal{S}} \sum_{a \in \mathcal{A}_s} N(s,a) \overline{\tau}(s,a) + \sigma_\tau \sqrt{\frac{5}{2}n\log \left( \frac{13n}{\delta}\right)} \Bigg| \left(N(s,a)\right)_{s,a} \right) &\leq \left(\frac{\delta}{13n} \right)^{5/4} \leq \frac{\delta}{24n^{5/4}}, \\
& \mbox{if } n \geq \frac{5b_\tau^2}{2\sigma_\tau^2} \log \left( \frac{13n}{\delta} \right)\\
\mathbb{P} \left( \sum_{i=1}^{n}{\tau_i} \geq \sum_{s \in \mathcal{S}} \sum_{a \in \mathcal{A}_s} N(s,a) \overline{\tau}(s,a) + \frac{5}{2}b_\tau\log \left( \frac{13n}{\delta}\right) \Bigg| \left(N(s,a)\right)_{s,a} \right) &\leq \left(\frac{\delta}{13n} \right)^{5/4} \leq \frac{\delta}{24n^{5/4}}, \\
&\mbox{if } n \leq \frac{5b_\tau^2}{2\sigma_\tau^2} \log \left( \frac{13n}{\delta} \right)\\
\end{align*}

\begin{lemma}\label{GainBound}
The optimal average reward can be bounded as follows: 
\begin{align*}
\rho^*(M) \leq \max_{s \in \mathcal{S},a \in \mathcal{A}_s} \left\lbrace \frac{\overline{r}(s,a)}{\overline{\tau}(s,a)} \right\rbrace \leq R_{\max}.
\end{align*}
\end{lemma}

\begin{proof}
In App.~\ref{GainSMDPs} we prove that $\rho^*(M) = \rho^*(M_{\Meq})$ where $\rho^*(M_{\Meq})$ is the optimal average reward of an MDP $M_{\Meq}$ with same state and action spaces as SMDP $M$ and with average rewards of the form $\frac{\overline{r}(s,a)}{\overline{\tau}(s,a)}$. All the rewards of $M_{\Meq}$ are thus bounded by $R_{\max}$ and so $\rho^*(M_{\Meq})$ is necessarily bounded by $R_{\max}$ as well and thus: $\rho^*(M) \leq R_{\max}$.
\end{proof}

We are now ready to split the regret over episodes. We define the per-episode regret as
\begin{align*}
\Delta_k = \sum_{s \in \mathcal{S}} \sum_{a \in \mathcal{A}_s}{\nu_k(s,a) \left( \overline{\tau}(s,a)\rho^*- \overline{r}(s,a) \right)}.
\end{align*}
Setting $\gamma_r = \max \left\lbrace \frac{5}{2}b_r, \sqrt{\frac{5}{2}}\sigma_r \right\rbrace$ and $\gamma_{\tau} = \max \left\lbrace \frac{5}{2}b_{\tau}, \sqrt{\frac{5}{2}}\sigma_{\tau} \right\rbrace$, and using a union bound on the previous inequalities we have that with probability at least $1-\frac{\delta}{12n^{5/4}}$
\begin{align*}
\Delta(M,\mathfrak{A},s,n) \leq \sum_{k=1}^{n}{\Delta_k} + \left( \gamma_r + \gamma_{\tau} R_{\max} \right) \log \left( \frac{13n}{\delta}\right) \sqrt{n} 
\end{align*}

\subsection{Dealing with Failing Confidence Regions}

\begin{lemma}\label{ConfidenceInterval}
For any episode $ k \geq 1$, the probability that the true SMDP $M$ is not contained in the set of plausible MDPs $\mathcal{M}_k$ at step $i$ is at most $\frac{\delta}{15i_k^6}$, that is: 
\begin{align}\label{eq16}
\forall k \geq 1, \ \mathbb{P} \left( M \not\in \mathcal{M}_k \right) < \frac{\delta}{15i_k^6}
\end{align}
\end{lemma}

\begin{proof}
This lemma is the SMDP-analogue of Lemma 17 in \citep{Jaksch10} and the proof is similar. Using an $\ell_1$-concentration inequality for discrete probability distributions we obtain
\begin{align*}
\mathbb{P} \Big( \big\| p(\cdot|s,a) - \hat{p}_k(\cdot|s,a) \big\|_1 \geq \beta_k^p(s,a) \Big) &= \mathbb{P} \left( \big\| p(\cdot|s,a) - \hat{p}_k(\cdot|s,a) \big\|_1 \geq \sqrt{\frac{14 S}{n} \log \left( \frac{2A i_k}{\delta} \right)} \right)\\
&\leq \mathbb{P} \left( \big\| p(\cdot|s,a) - \hat{p}_k(\cdot|s,a) \big\|_1 \geq \sqrt{\frac{2}{n} \log \left( \frac{2^S 20 S A i_k^7}{\delta} \right)} \right)\\
&\leq 2^S \exp \left( -\frac{n}{2} \times \frac{2}{n} \log \left( \frac{2^S 20 S A i_k^7}{\delta} \right) \right)\\
&= \frac{\delta}{20i_k^7SA}
\end{align*}
In the inequalities above, it is implicitly assumed that the value $N_k(s,a) = n$ is fixed. To be more rigorous, we are bounding the probability of the intersection of event $ \lbrace \big\| \widetilde{p}(\cdot|s,a) - \hat{p}_k(\cdot|s,a) \big\|_1 \geq \beta_k^p(s,a) \rbrace $ with event $ \lbrace N_k(s,a) = n \rbrace $  but we omitted the latter to simplify notations, and we will also omit it in the next inequalities.
Using Bernstein inequality (Prop.~\ref{BernsteinSubExponential}) and noting that $240\leq 2^7\left(\frac{SA}{\delta}\right)^6$ for $S, A \geq 2$ and $\delta \leq 1$, we have:
\begin{itemize}
\item If $n \geq \frac{2 b_r^2}{\sigma_r^2}\log \left( \frac{240SAi_k^7}{\delta}\right)$:
\begin{align*}
\mathbb{P} \Big( \left| \overline{r}(s,a) - \hat{r}_k(s,a) \right| \geq \beta_k^r(s,a) \Big) &= \mathbb{P} \left( \left| \overline{r}(s,a) - \hat{r}_k(s,a) \right| \geq \sigma_r \sqrt{\frac{14}{n} \log \left( \frac{2SA i_k}{\delta} \right)} \right)\\
&\leq \mathbb{P} \left( \left| \overline{r}(s,a) - \hat{r}_k(s,a) \right| \geq \sigma_r \sqrt{\frac{2}{n} \log \left( \frac{240 S A i_k^7}{\delta} \right)} \right)\\
&\leq 2 \exp \left( -\frac{n}{2 \sigma_r^2} \times \frac{2}{n} \sigma_r^2 \log \left( \frac{240 S A i_k^7}{\delta} \right) \right)\\
&= \frac{\delta}{120i_k^7SA}
\end{align*}

\item If $n < \frac{2 b_r^2}{\sigma_r^2}\log \left( \frac{240SAi_k^7}{\delta}\right)$:
\begin{align*}
\mathbb{P} \Big( \left| \overline{r}(s,a) - \hat{r}_k(s,a) \right| \geq \beta_k^r(s,a) \Big) &= \mathbb{P} \left( \left| \overline{r}(s,a) - \hat{r}_k(s,a) \right| \geq \frac{14b_r}{n} \log \left( \frac{2SA i_k}{\delta} \right) \right)\\
&\leq \mathbb{P} \left( \left| \overline{r}(s,a) - \hat{r}_k(s,a) \right| \geq  \frac{2b_r}{n} \log \left( \frac{240SA i_k^7}{\delta} \right)  \right)\\
&\leq 2 \exp \left( -\frac{n}{2 b_r} \times \frac{2}{n} b_r \log \left( \frac{240SA i_k^7}{\delta} \right)  \right)\\
&= \frac{\delta}{120i_k^7SA}
\end{align*}
\end{itemize}

Similarly for holding times we have:
\begin{align*}
\mathbb{P} \Big( \left| \overline{\tau}(s,a) - \hat{\tau}_k(s,a) \right| \geq \beta_k^\tau(s,a) \Big) \leq \frac{\delta}{120i_k^7SA}
\end{align*}

Note that when there hasn't been any observation, the confidence intervals trivially hold with probability $1$. Moreover, $N_k(s,a) < i_k$ by the stopping condition of an episode. Taking a union bound over all possible values of $N_k(s,a)$ yields:
\begin{align*}
\mathbb{P} \Big( \left| \overline{\tau}(s,a) - \hat{\tau}_k(s,a) \right| \geq \beta_k^\tau(s,a) \Big) \leq \frac{\delta}{120i_k^6SA}\\
\mathbb{P} \Big( \left| \overline{r}(s,a) - \hat{r}_k(s,a) \right| \geq \beta_k^r(s,a) \Big) \leq \frac{\delta}{120i_k^6SA}\\
\mathbb{P} \Big( \big\| p(\cdot|s,a) - \hat{p}_k(\cdot|s,a) \big\|_1 \geq \beta_k^p(s,a) \Big) \leq \frac{\delta}{20i_k^6SA}\\
\end{align*}
Summing over all state-action pairs: $\mathbb{P} \left( M \not\in \mathcal{M}_k \right) < \frac{\delta}{15i_k^6}$.
\end{proof}

We now consider the regret of episodes in which the set of plausible SMDPs $\mathcal{M}_k$ does not contain the true SMDP $M$: $\sum_{k=1}^{m}{\Delta_k \mathbb{1}_{M \not\in \mathcal{M}_k}}$. By the stopping criterion for episode $k$ (except for episodes where $\nu_k(s,a) = 1$ and $N_k(s,a) = 0$ for which $\sum_{s \in \mathcal{S}} \sum_{a \in \mathcal{A}_s} \nu_k(s,a) = 1 \leq i_k$):\\
\begin{align}\label{eq17}
\sum_{s \in \mathcal{S}} \sum_{a \in \mathcal{A}_s} \nu_k(s,a) \leq \sum_{s \in \mathcal{S}} \sum_{a \in \mathcal{A}_s} N_k(s,a) = i_k - 1
\end{align}
We can thus bound this part of the regret:
\begin{align*}
\sum_{k=1}^{m}{\Delta_k \mathbb{1}_{M \not\in \mathcal{M}_k}} &\leq \sum_{k=1}^{m} \sum_{s \in \mathcal{S}} \sum_{a \in \mathcal{A}_s} \nu_k(s,a) \overline{\tau}(s,a) \rho^* \mathbb{1}_{M \not\in \mathcal{M}_k}\\ 
&\leq \tau_{\max} \rho^*\sum_{k=1}^{m} i_k \mathbb{1}_{M \not\in \mathcal{M}_k} = \tau_{\max} \rho^* \sum_{i=1}^{n} i \sum_{k=1}^{m} \mathbb{1}_{i = i_k, M \not\in \mathcal{M}_k}\\
&\leq \tau_{\max} \rho^* \left( \sum_{i=1}^{\lfloor n^{1/4} \rfloor} i  +  \sum_{i=\lfloor n^{1/4} \rfloor +1}^{n} i \sum_{k=1}^{m} \mathbb{1}_{i = i_k, M \not\in \mathcal{M}_k} \right) \\
&\leq \tau_{\max} \rho^* \left( \sqrt{n} + \sum_{i=\lfloor n^{1/4} \rfloor +1}^{n} i \sum_{k=1}^{m} \mathbb{1}_{i = i_k, M \not\in \mathcal{M}_k} \right)
\end{align*}
where we defined: $\tau_{\max} = \max_{s,a} \overline{\tau}(s,a) < + \infty$.\\
By Lemma~\ref{ConfidenceInterval}, the probability that the second term in the right hand side of the above inequality is non-zero is bounded by
\begin{align*}
\sum_{i=\lfloor n^{1/4} \rfloor}^{n} \frac{\delta}{15i^6} \leq \frac{\delta}{15n^{6/4}} + \int_{n^{1/4}}^{+\infty} \frac{\delta}{15x^6} \, \mathrm{d}x \leq \frac{\delta}{12n^{5/4}}.
\end{align*}
In other words, with probability at least $1-\frac{\delta}{12n^{5/4}}$:
\begin{align*}
\sum_{k=1}^{m}{\Delta_k \mathbb{1}_{M \not\in \mathcal{M}_k}} \leq \tau_{\max} R_{\max}\sqrt{n}.
\end{align*}

\subsection{Episodes with $M \in \mathcal{M}_k$}

Now we assume that $M \in \mathcal{M}_k$ and we start by analysing the regret of a single episode $k$. By construction, $R_{\max} \geq \widetilde{\rho}_k \geq \rho^* - \frac{R_{\max}}{\sqrt{i_k}}$ hence:
\begin{align*}
\Delta_k = \sum_{s \in \mathcal{S}} \sum_{a \in \mathcal{A}_s}{\nu_k(s,a) \left( \overline{\tau}(s,a)\rho^*- \overline{r}(s,a) \right)} \leq \sum_{s \in \mathcal{S}} \sum_{a \in \mathcal{A}_s}{\nu_k(s,a) \left( \overline{\tau}(s,a)\widetilde{\rho}_k- \overline{r}(s,a) \right)} + R_{\max} \sum_{s \in \mathcal{S}} \sum_{a \in \mathcal{A}_s} {\frac{\nu_k(s,a)}{\sqrt{i_k}}\overline{\tau}(s,a)}
\end{align*}
\begin{align*}
\implies \Delta_k  \leq \sum_{s \in \mathcal{S}} \sum_{a \in \mathcal{A}_s}{\nu_k(s,a) \left( \widetilde{\tau}_k(s,a)\widetilde{\rho}_k- \overline{r}(s,a) \right)}
&+ R_{\max} \sum_{s \in \mathcal{S}} \sum_{a \in \mathcal{A}_s}{\nu_k(s,a) \left(\overline{\tau}(s,a)  - \widetilde{\tau}_k(s,a) \right)} \\
&+ R_{\max} \tau_{\max} \sum_{s \in \mathcal{S}} \sum_{a \in \mathcal{A}_s} {\frac{\nu_k(s,a)}{\sqrt{i_k}}}
\end{align*}

We now need two results about the extended value iteration algorithm.

\begin{lemma} \label{SpanVI}
At any iteration $i \geq 0$ of EVI (extended value iteration), the range of the state values is bounded as follows,
\begin{align}\label{eq18}
\forall i \geq 0, \ \ \max_{s\in \mathcal{S}} u_i(s) - \min_{s\in \mathcal{S}} u_i(s) \leq \frac{ R_{\max} D(M)}{\tau},
\end{align}
where $R_{\max}$ is an upper-bound on the per-step reward $\wb{r}(s,a)/\wb{\tau}(s,a)$, $\tau$ is the parameter used in the uniformization of the SMDP $M$ and $D(M)$ is its diameter (Def.~\ref{def:Diameter}).
\end{lemma}

\begin{proof}
In Appendix~\ref{GainSMDPs} we show that EVI is value iteration applied to the equivalent MDP $\widetilde{M}_{k,\Meq}^+$ obtained by ``uniformizing'' the extended SMDP $\widetilde{M}_k^+$. Thus, we focus on any SMDP $M$ and its equivalent MDP $M_{\Meq}$. Using the same argument as in section 4.3.1 of \citep{Jaksch10}, we have that: $ \forall i \geq 0, \ \max_{s\in \mathcal{S}} u_i(s) - \min_{s\in \mathcal{S}} u_i(s) \leq  R_{\max} D(M_{\Meq})$ since all rewards of $M_{\Meq}$ are bounded by $R_{\max}$ whenever the average reward in $M$ is bounded by $R_{\max}$. Thus we need to find a relationship between $D(M)$ and $D(M_{\Meq})$. Let $T(s')$ denote the first time at which state $s'$ is reached in $M$ or $M_{\Meq}$, that is
\begin{align*}
&\mbox{In SMDP } M: \ T(s') = \inf \bigg\{ \sum_{i=1}^{n}{\tau_i} : n \in \mathbb{N}, \ s_n = s'\bigg\} \\
&\mbox{In MDP } M_{\Meq}: \ T(s') = \inf \bigg\{ n : n \in \mathbb{N}, \ s_n = s'\bigg\}.
\end{align*}
We prove that $\forall s, s' \in \mathcal{S}, \ \forall \pi \in \Pi_M^{SD} = \Pi_{M'}^{SD},\ \mathbb{E}_M^\pi \big[T(s')|s_0 = s \big] = \tau \mathbb{E}_{M'}^\pi \big[T(s')|s_0 = s \big]$. We consider two cases:
\begin{enumerate}
\item If $\ \mathbb{P}_M^\pi \big(T(s') = + \infty|s_0 = s \big) > 0 \ $ then necessarily $\mathbb{E}_M^\pi \big[T(s')|s_0 = s \big] = + \infty$.\\
Moreover: $\ \mathbb{P}_M^\pi \big(T(s') = + \infty|s_0 = s \big) > 0 \implies \mathbb{P}_{M_{\Meq}}^\pi \big(T(s') = + \infty|s_0 = s \big) > 0 \ $ and so $\mathbb{E}_{M_{\Meq}}^\pi \big[T(s')|s_0 = s \big] = + \infty =\frac{1}{\tau} \mathbb{E}_M^\pi \big[T(s')|s_0 = s \big]$.
\item Conversely: $\ \mathbb{P}_M^\pi \big(T(s') = + \infty|s_0 = s \big) =0 \implies \mathbb{P}_{M_{\Meq}}^\pi \big(T(s') = + \infty|s_0 = s \big) =0 \ $ in which case both expectations are finite. To prove they are equal up to factor $\tau$, we see the holding time as a ``reward'' (the true rewards are ignored here). Note that any policy $\pi$ induces Markov chains with different dynamics on $M$ and $M_{\Meq}$ (different transition probabilities). We call these Markov chains $MC$ and $MC_{\Meq}$ respectively. Suppose we modify $MC$ as follows: all states that are not reachable from $s$ are ignored, all other states are unchanged except $s'$ that is assumed to be absorbing (i.e., $\pi(s')$ is an action that loops on $s'$ with probability 1). Furthermore, we build a Markov reward process $MR$ with the same dynamics as $MC$ and such that all transitions $(\overline{s},\pi(\overline{s}))$ have an expected reward equal to $\overline{\tau}(\overline{s},\pi(\overline{s}))$ except $(s',\pi(s'))$ which has a reward of zero. The total expected reward of this Markov reward process (MRP denoted $MR$) starting from $s$ trivially equals $\mathbb{E}_M^\pi \big[T(s')|s_0 = s \big]$. Since we assumed that $\mathbb{E}_M^\pi \big[T(s')|s_0 = s \big]$ is finite, and because all states of $MR$ are reachable from $s$ (the other states were ignored), $s'$ is reached with probability 1 no matter which starting state $\overline{s}$ of $MR$ is chosen (or in other words, even though we ignored some states, the transition matrix of $MR$ is stochastic $-$and not sub-stochastic$-$ and has a single recurrent class consisting of the absorbing state $s'$). By \citep{Puterman94}, the vector $\big( \overline{T}(\overline{s}) \big)_{\overline{s} \in \mathcal{S}} = \big( \mathbb{E}_M^\pi \big[T(s')|s_0 = \overline{s} \big] \big)_{\overline{s} \in \mathcal{S}}$ is the unique solution to the system of equations
\begin{align*}
\forall \overline{s}, \ \overline{T}(\overline{s}) = \overline{\tau}(\overline{s},d(\overline{s})) + \sum_{\widetilde{s}}p(\widetilde{s}|\overline{s},d(\overline{s})) \overline{T}(\widetilde{s}).
\end{align*}
Applying the same transformation to $MC_{\Meq}$ and assigning a reward of 1 to all transitions but $(s',\pi(s'))$ (which has reward 0) in order to build $MR_{\Meq}$, we deduce that the vector $\big( \overline{T}_{\Meq}(\overline{s}) \big)_{\overline{s} \in \mathcal{S}} = \big( \mathbb{E}_{M_{\Meq}}^\pi \big[T(s')|s_0 = \overline{s} \big] \big)_{\overline{s} \in \mathcal{S}}$ is the unique solution to the system of equations
\begin{align*}
&\forall \overline{s}, \ \overline{T}_{\Meq}(\overline{s}) = 1 + \frac{\tau}{\overline{\tau}(\overline{s},d(\overline{s}))}\sum_{\widetilde{s}}p(\widetilde{s}|\overline{s},d(\overline{s})) \overline{T}_{\Meq}(\widetilde{s}) + \left(1 - \frac{\tau}{\overline{\tau}(\overline{s},d(\overline{s}))} \right)\overline{T}_{\Meq}(\overline{s})\\
\iff &\forall \overline{s}, \ \left(\tau\overline{T}_{\Meq}(\overline{s})\right)  = \overline{\tau}(\overline{s},d(\overline{s})) + \sum_{\widetilde{s}}p(\widetilde{s}|\overline{s},d(\overline{s})) \left(\tau\overline{T}_{\Meq}(\widetilde{s})\right).
\end{align*}
By uniqueness of the solution: $\tau\overline{T}_{\Meq} = \overline{T} \implies \tau D(M_{\Meq}) = D(M)$.
\end{enumerate}
\end{proof}

\begin{lemma} \label{ConvergenceEVI}
If the convergence criterion of EVI hold at iteration $i$, then:
\begin{align}\label{eq19}
\forall s \in \mathcal{S}, \ \ \big| u_{i+1}(s) - u_i(s) - \widetilde{\rho}_k \big| \leq \frac{1}{\sqrt{i_k}}
\end{align}
\end{lemma}

\begin{proof}
We introduce the following quantities
\begin{align*}
M_i = \max_{s \in \mathcal{S}} \lbrace u_{i+1}(s) - u_i(s) \rbrace, \ \ m_i = \min_{s \in \mathcal{S}} \lbrace u_{i+1}(s) - u_i(s) \rbrace, \ \ \epsilon = \frac{1}{\sqrt{i_k}}.
\end{align*}
Since EVI is just value iteration applied to MDP $M_k'$, Theorem 8.5.6 of \citep{Puterman94} hold and we have:
\begin{align*}
\frac{1}{2}(M_i + m_i) \geq \widetilde{\rho}_k - \frac{\epsilon}{2} \iff m_i \geq \widetilde{\rho}_k - \frac{\epsilon}{2} - \frac{1}{2}(M_i - m_i) \implies m_i \geq \widetilde{\rho}_k - \epsilon \\ 
\frac{1}{2}(M_i + m_i) - \widetilde{\rho}_k \leq \frac{\epsilon}{2} \iff M_i \leq \widetilde{\rho}_k + \frac{\epsilon}{2} + \frac{1}{2}(M_i - m_i) \implies M_i \leq \widetilde{\rho}_k + \epsilon .
\end{align*}
In conclusion:
\begin{align*}
\forall s \in \mathcal{S}, \ \ \frac{-1}{\sqrt{i_k}} \leq u_{i+1}(s) - u_i(s) - \widetilde{\rho}_k \leq \frac{1}{\sqrt{i_k}}.
\end{align*}
\end{proof}

Based on Lemma~\ref{ConvergenceEVI}, Eq.~\ref{eq19}, and optimiality equation Eq.~\ref{eq23}, we have:
\begin{align}\label{eq20}
\forall s \in \mathcal{S}, \ \left| \left( \widetilde{\rho}_k - \frac{\widetilde{r}_k(s,\widetilde{\pi}_k(s))}{\widetilde{\tau}_k(s,\widetilde{\pi}_k(s))} \right) - \left( \sum_{s' \in \mathcal{S}} \widetilde{p}_k(s' | s, \widetilde{\pi}_k(s)) u_i(s') - u_i(s) \right) \frac{\tau}{\widetilde{\tau}_k(s,\widetilde{\pi}_k(s))} \right| \leq \frac{1}{\sqrt{i_k}}
\end{align}

Setting $r_k = \big( \widetilde{r}_k(s,\widetilde{\pi}_k(s)) \big)_{s \in \mathcal{S}}$ to be the column vector of rewards under policy $\widetilde{\pi}_k$, $\widetilde{P}_k = \big(\widetilde{p}_k(s' | s, \widetilde{\pi}_k(s)) \big)_{s, s' \in \mathcal{S}}$ the transition matrix and $ v_k = \big( \nu_k (s,\widetilde{\pi}_k(s)) \big)_{s \in \mathcal{S}}$ the row vector of visit counts for each state and the corresponding action chosen by $\widetilde{\pi}_k$. We will use the fact that $a \neq \widetilde{\pi}_k(s) \implies \nu_k(s,a) = 0$.

\begin{align*}
\Delta_k  &\leq \sum_{s,a}{\nu_k(s,a) \left( \widetilde{\tau}_k(s,a)\widetilde{\rho}_k- \overline{r}(s,a) \right)} + R_{\max} \sum_{s,a}{\nu_k(s,a) \left(\overline{\tau}(s,a)  - \widetilde{\tau}_k(s,a) \right)} + R_{\max}\tau_{\max}\sum_{s,a}{\frac{\nu_k(s,a)}{\sqrt{i_k}}}\\
&= \sum_{s,a}{\nu_k(s,a) \left( \widetilde{\tau}_k(s,a)\widetilde{\rho}_k- \widetilde{r}_k(s,a) \right)} + \sum_{s,a}{\nu_k(s,a) \left( \widetilde{r}_k(s,a)- \overline{r}(s,a) \right)} + R_{\max}\tau_{\max} \sum_{s,a}{\frac{\nu_k(s,a)}{\sqrt{i_k}}}\\
&\qquad\qquad\qquad\qquad\qquad\qquad\qquad\qquad\qquad\qquad\qquad\qquad + R_{\max} \sum_{s,a}{ \nu_k(s,a)\left(\overline{\tau}(s,a)  - \widetilde{\tau}_k(s,a) \right)}
\end{align*}
We will now upper-bound the four terms of the right-hand side of the above inequality. Setting $c_r = \max \lbrace 14 b_r, \sqrt{14} \sigma_r \rbrace$ and $c_\tau = \max \lbrace 14 b_\tau, \sqrt{14} \sigma_\tau \rbrace$ we have:
\begin{align*} 
\widetilde{r}_k(s,a) - \overline{r}(s,a) \leq \big| \widetilde{r}_k(s,a) - \hat{r}_k(s,a) \big| + \big| \hat{r}_k(s,a) - \overline{r}(s,a) \big| \leq 2 \beta_k^r(s,a) \leq 2 c_r \frac{\log(2S A i_k / \delta)}{\sqrt{\max \lbrace 1, N_k(s,a) \rbrace}}
\end{align*}
\begin{align*} 
\overline{\tau}(s,a) - \widetilde{\tau}_k(s,a) \leq \big| \widetilde{\tau}_k(s,a) - \hat{\tau}_k(s,a) \big| + \big| \hat{\tau}_k(s,a) - \overline{\tau}(s,a) \big| \leq 2 \beta_k^\tau(s,a) \leq 2 c_\tau \frac{\log(2S A i_k / \delta)}{\sqrt{\max \lbrace 1, N_k(s,a) \rbrace}}
\end{align*}
Finally, using \ref{eq20} and noting that $\widetilde{\tau}_k(s,a) \leq \tau_{\max}$ (by construction) we obtain:
\begin{align*} 
\widetilde{\tau}_k(s,a)\widetilde{\rho}_k- \widetilde{r}_k(s,a) \leq \frac{R_{\max} \tau_{\max}}{\sqrt{i_k}} + \tau \left( \sum_{s' \in \mathcal{S}} \widetilde{p}_k(s' | s, \widetilde{\pi}_k(s)) u_i(s') - u_i(s) \right), \mbox{ if } a = \widetilde{\pi}_k(s)
\end{align*}
\begin{align*}
\implies \sum_{s,a}{ \nu_k(s,a) \left(\widetilde{\tau}_k(s,a)\widetilde{\rho}_k- \widetilde{r}_k(s,a) \right)} \leq  R_{\max}\tau_{\max} \sum_{s,a}{ \frac{\nu_k(s,a)}{\sqrt{i_k}}} + \tau\Big(v_k \big( \widetilde{P}_k - I \big) u_i \Big)
\end{align*}
where $i$ is the iteration at which the stopping condition of EVI holds. Defining the column vector $w_k$ by:
\begin{align*} 
w_k(s) = u_i(s) - \frac{\min_{s \in \mathcal{S}} u_i(s) + \max_{s \in \mathcal{S}} u_i(s)}{2}
\end{align*}
and since the rows of $\widetilde{P}_k$ sum to one, we have: $v_k \big( \widetilde{P}_k - I \big) u_i = v_k \big( \widetilde{P}_k - I \big) w_k$. Moreover, by Lemma \ref{SpanVI}: $\| w_k \|_\infty \leq \frac{ R_{\max} D}{2 \tau}$. Noting that $\max \lbrace 1, N_k(s,a) \rbrace \leq i_k \leq n$ we get:
\begin{align*} 
\Delta_k \leq \tau\Big(v_k \big( \widetilde{P}_k - I \big) w_k \Big) + 2 \left( R_{\max}\tau_{\max} + (c_r + R_{\max}c_\tau) \log \left(\frac{2S A n}{\delta}\right) \right) \sum_{s,a}{\frac{\nu_k(s,a)}{\sqrt{\max \lbrace 1, N_k(s,a) \rbrace}}} 
\end{align*}

Using exactly the same arguments as in  \citet{Jaksch10}, it is trivial to prove that with probability at least $1-\frac{\delta}{12n^{5/4}}$:
\begin{align*}
\sum_{k=1}^{m} v_k \big( \widetilde{P}_k - I \big) w_k \mathbb{1}_{M \in \mathcal{M}_k} \leq \frac{R_{\max} D}{\tau} \Bigg[ \sqrt{14S \log \left( \frac{2An}{\delta} \right)} \sum_{k=1}^{m} \sum_{s,a}{\frac{\nu_k(s,a)}{\sqrt{\max \lbrace 1, N_k(s,a) \rbrace}}} + \sqrt{\frac{5}{2}n \log \left( \frac{8n}{\delta} \right)} \\+  S A \log_2 \left( \frac{8n}{SA} \right) \Bigg]
\end{align*}

\begin{lemma}\label{ineq}
Consider a sequence of positive reals $(z_k)_k$ and define: $\forall k, \ Z_k = \max \lbrace 1, \sum_{i=1}^{k}{z_i} \rbrace$. Assuming that $0 \leq z_k \leq Z_{k-1}$ we have:
\begin{align*}
\forall n \geq 1, \ \sum_{k=1}^{n}{\frac{z_k}{\sqrt{Z_{k-1}}}} \leq \big( \sqrt{2} + 1 \big) \sqrt{Z_n}
\end{align*}
\end{lemma}
\begin{proof}
See Appendix C.3 of \citep{Jaksch10}.
\end{proof}

Using Lemma \ref{ineq} we get:
\begin{align*}
\sum_{k=1}^{m}\sum_{s,a}{\frac{\nu_k(s,a)}{\sqrt{\max \lbrace 1, N_k(s,a) \rbrace}}} \leq \big( \sqrt{2} + 1 \big) \sum_{s,a}{\sqrt{N(s,a)}}
\end{align*}
By Jensen's inequality we thus have:
\begin{align*}
\sum_{k=1}^{m}\sum_{s,a}{\frac{\nu_k(s,a)}{\sqrt{\max \lbrace 1, N_k(s,a) \rbrace}}} \leq \big( \sqrt{2} + 1 \big) \sqrt{SAn}
\end{align*}

In conclusion, when $M \in \mathcal{M}_k$, with probability at least $1-\frac{\delta}{12n^{5/4}}$:
\begin{align*}
\sum_{k=1}^{m}{\Delta_k \mathbb{1}_{M \in \mathcal{M}_k}} \leq R_{\max} D \sqrt{\frac{5}{2}n \log \left( \frac{8n}{\delta} \right)} + R_{\max} D S A \log_2 \left( \frac{8n}{SA} \right) + (\sqrt{2}+1)\Bigg[ 2 R_{\max}\tau_{\max}\\
 + 2 (c_r + R_{\max}c_\tau) \log \left( \frac{2SAn}{\delta} \right) + R_{\max} D \sqrt{14S \log\left( \frac{2An}{\delta} \right)} \Bigg] \sqrt{SAn}
\end{align*}

\subsection{Computing the final bound}
Gathering all previous inequalities, we have that with probability at least $1 - \frac{3\delta}{12n^{5/4}} = 1 - \frac{\delta}{4n^{5/4}}$:
\begin{align*}
\Delta(M,\mathfrak{A},s,n) &\leq \left( \gamma_r + \gamma_{\tau} R_{\max} \right) \log \left( \frac{13n}{\delta}\right) \sqrt{n} + \tau_{\max} R_{\max}\sqrt{n} + R_{\max} D \sqrt{\frac{5}{2}n \log \left( \frac{8n}{\delta} \right)} \\
&+ (\sqrt{2}+1)\Bigg[ 2 \tau_{\max} + 2 (c_r + R_{\max}c_\tau) \log \left( \frac{2SAn}{\delta} \right) + R_{\max} D \sqrt{14S \log\left( \frac{2An}{\delta} \right)} \Bigg] \sqrt{SAn}\\
&+ R_{\max} D S A \log_2 \left( \frac{8n}{SA} \right) 
\end{align*}

In \citep{Jaksch10} (see Appendix C.4), it is shown that when $n > 34 A \log \left( \frac{n}{\delta}\right)$:
\begin{align*}
&DSA\log_2 \left( \frac{8n}{SA}\right) < \frac{2}{34}DS\sqrt{An\log \left( \frac{n}{\delta}\right)}, \ \mbox{and} \ \log \left( \frac{2An}{\delta}\right) \leq 2 \log \left( \frac{n}{\delta}\right)
\end{align*}
and moreover if $n > S \log \left( \frac{n}{\delta} \right)$ and $A \geq 2$ (if $A = 1$ the regret is zero):
\begin{align*}
&\frac{n^2}{\delta^2} \geq \frac{n S \log \left( \frac{n}{\delta} \right)}{\delta}  \geq \frac{n S}{\delta}  \implies \frac{n^2A^2}{\delta^2} \geq \frac{2SAn}{\delta}  \implies 4 \log \left( \frac{n}{\delta} \right) \geq 2 \log \left( \frac{An}{\delta} \right) \geq \log \left( \frac{2SAn}{\delta} \right)\\
&\implies \Delta(M,\mathfrak{A},s,n) = O \left( \left(D \sqrt{S} + \tau_{\max} + \left(\frac{C_r}{R_{\max}} + C_{\tau} \right) \sqrt{\log \left( \frac{n}{\delta}\right)} \right) R_{\max} \sqrt{S A n \log \left( \frac{n}{\delta}\right)}\right)
\end{align*}
where $C_r = \max \lbrace b_r , \sigma_r \rbrace$ and $C_{\tau} = \max \lbrace b_\tau , \sigma_\tau \rbrace$.\\ \\
Note that if $n \leq 34 A \log \left( \frac{n}{\delta}\right)$ then we trivially have: 
\begin{align*}
\sum_{k=1}^{m}{\Delta_k} \leq \tau_{\max}R_{\max}n = \tau_{\max}R_{\max}(\sqrt{n})^2 \leq 34\tau_{\max}R_{\max}\sqrt{An\log \left( \frac{n}{\delta}\right)}
\end{align*}
and if $n \leq S \log \left( \frac{n}{\delta}\right)$:
\begin{align*}
\sum_{k=1}^{m}{\Delta_k} \leq \tau_{\max}R_{\max}n = \tau_{\max}R_{\max}(\sqrt{n})^2 \leq \tau_{\max}R_{\max}\sqrt{Sn\log \left( \frac{n}{\delta}\right)}
\end{align*}
and thus the previous bound on the whole regret still holds. Taking a union bound over all possible values of $n \geq 1$ we have that with probability at least $1-\delta$:
\begin{align*}
\forall n \geq 1, \ \Delta(M,\mathfrak{A},s,n) = O \left( \left( D\sqrt{S} + \tau_{\max} + \left( \frac{C_r}{R_{\max}} + C_{\tau} \right) \sqrt{\log \left( \frac{n}{\delta}\right)} \right) R_{\max} \sqrt{S A n \log \left( \frac{n}{\delta}\right)}\right).
\end{align*}

The derivation for the case of bounded holding times is exactly the same with different concentration inequalities applied to estimates $\wh\tau(s,a)$. Note that all the terms in the upper bound are very similar to those appearing in the derivation of the upper bound for MDPs, thus the constants in the big $O$ are very close. This justifies the analysis of the ratio between the two upper bounds in Sect.~\ref{sec:smdp2mdp}. 
\section{The Lower Bound (Theorem \ref{LowerBound})}\label{app:proof.lower.bound}

\todoa{I haven't checked this yet!!}
\subsection{Lower Bound for SMDPs}

We will derive the lower bound by applying the same techniques as in the proof of the lower bound for MDPs (section 6 of \citep{Jaksch10}). We first consider the two-state SMDP $M'$ depicted in Fig. \ref{fig:small.smdp}. Since by assumption $\frac{D}{12}> T_{\min}$ and $\frac{T_{\max}}{3}> T_{\min}$, let $\overline{\tau} \in \left]T_{\min},\min\left\lbrace \frac{D}{12}, \frac{T_{\max}}{3}\right\rbrace \right]$. Define $p = \frac{\overline{\tau}}{T_{\max}}$ and $\delta = \frac{4\overline{\tau}}{D}$. By definition of $\overline{\tau}$ we have: $p,\delta \leq \frac{1}{3}$. There are $A' = \lfloor \frac{A-1}{2} \rfloor$ actions available in each state of $M'$. We assume that $\forall (s,a,s')\in \mathcal{S} \times \mathcal{A}_s \times \mathcal{S}$ and $\forall i \geq 0$, $r_i(s,a,s')$ and $\tau_i(s,a,s')$ are independent. We also assume that $\forall i \geq 0$, $\tau_i(s_{i-1},a_i,s_i)$ and $r_i(s_{i-1},a_i,s_i)$ are independent of the next state $s_{i}$ and we write: $\tau_i(s_{i-1},a_i)$ and $r_i(s_{i-1},a_i)$. For each action $a$ in $\mathcal{A}_{s_0}$, $r(s_0, a)=0$ and $\tau(s,a) \sim \tau$ where $\tau$ is a r.v. defined in Table \ref{TimeVariables}. Moreover, for all actions $a$ but a specific action $a_0^*$, $p(s_1|s_0,a) = \delta$ whereas $p(s_1|s_0,a_0^*) = \delta + \epsilon$ for some $0 < \epsilon < \delta$ specified later in the proof. For all actions $a$ in $\mathcal{A}_{s_1}$, $p(s_0|s_1,a) = \delta $ and $\tau(s_1,a) \sim \tau$. Finally, $r(s_1, a) \sim r$ for all actions $a$ except $a_1^*$ for which $r(s_1, a_1^*) \sim r^*$, where $r$ and $r^*$ are r.v. defined in Table \ref{TimeVariables} where $0 < \eta < p$ will be defined later in the proof. Note that since $\eta < \frac{\overline{\tau}}{T_{\max}}$, we have: $\overline{r} \leq \overline{\tau}R_{\max}$ and $\overline{r}^* \leq \overline{\tau}R_{\max}$ which satisfies the definition of $R_{\max}$ given in assumption \ref{asm:BoundedExpectedDurationsRewards}. We denote $\mathbb{E}[\cdot|s]$ the expectation conditionally on starting in state $s$.

Let's define $T(s_1) = \inf \lbrace{t: s_t = s_1  \rbrace}$ the first time in which $s_1$ is encountered. $\forall d^\infty \in \Pi^{SD}_{M'}$ such that $d(s_0) \neq a_0^*$ we have:
\begin{align*}
\mathbb{E}^{d^\infty}[T(s_1)|s_0] &= \mathbb{E}^{d^\infty}\left[\sum_{n=1}^{+ \infty}{\left[ \left( \sum_{i=1}^{n}{\tau_i(s_{i-1}, a_i, s_i)} \right) \left( \prod_{j=0}^{n-1}{ \mathbb{1}_{s_j = s_0}} \right) \mathbb{1}_{s_n = s_1}  \right]} \Bigg| s_0 \right] \\
&= \sum_{n=1}^{+ \infty}{\mathbb{E}^{d^\infty}\left[ \left( \sum_{i=1}^{n}{\tau_i(s_{0}, d(s_0))} \right) \left( \prod_{j=0}^{n-1}{ \mathbb{1}_{s_j = s_0}} \right) \mathbb{1}_{s_n = s_1}  \Bigg| s_0 \right]}\\
&= \sum_{n=1}^{+ \infty}{\mathbb{E}^{d^\infty}\left[ \sum_{i=1}^{n}{\tau_i(s_{0}, d(s_0))} \Bigg| s_0 \right] \mathbb{E}^{d^\infty}\left[ \left( \prod_{j=0}^{n-1}{ \mathbb{1}_{s_j = s_0}} \right) \mathbb{1}_{s_n = s_1}  \Bigg| s_0 \right]}\\
&= \sum_{n=1}^{+ \infty}{n \overline{\tau} \times \delta {(1-\delta)}^{n-1}} = \frac{\overline{\tau}}{\delta}
\end{align*}
We used the fact that $\tau_i$ is independent of the next state $s_{i}$ and that $\overline{\tau} = \mathbb{E}[\tau]$. We can compute $\mathbb{E}^{d^\infty}[T(s_1)|s_0]$ with $d(s_0) = a_0^*$ and $\mathbb{E}^{d^\infty}[T(s_0)|s_1]$ (for both $d(s_0) \neq a_0^*$ and $d(s_0) = a_0^*$) similarly.
The diameter of SMDP $M'$ is thus:
\begin{align*}
D' = \max \left\{ \min \left\{ \frac{\overline{\tau}}{\delta}, \frac{\overline{\tau}}{\delta + \epsilon} \right\}, \min \left\{\frac{\overline{\tau}}{\delta}, \frac{\overline{\tau}}{\delta}  \right\} \right\} = \max \left\{ \frac{\overline{\tau}}{\delta + \epsilon} , \frac{\overline{\tau}}{\delta} \right\} = \frac{\overline{\tau}}{\delta}
\end{align*}

Any policy $d^\infty \in \Pi^{SD}_{M'}$ induces a recurrent Markov Chain on $M'$. Let's denote by $P_{d}^*$ the limiting matrix of this Markov Chain. We know (see \citep{Puterman94}) that $P_{d}^* = e\mu_d^\intercal$ where ${\mu_d = (1-p_d, p_d)^\intercal}$ is the stationary distribution of the recurrent Markov Chain. The probability $p_d$ can take only two different values: $p_d =\frac{1}{2}$ if and only if $d(s_0) \neq a_0^*$, and $p_d =\frac{\delta + \epsilon}{2\delta + \epsilon}$ if and only if $d(s_0) = a_0^*$. Using criterion 3 of Definition \ref{AverageReward}, the gain yielded by $d^\infty$ has the form:
\begin{align*}
\rho_d = \frac{p_d \overline{X}}{(1-p_d)\overline{\tau} + p_d \overline{\tau}} = \frac{p_d \overline{X}}{\overline{\tau}}
\end{align*}
where $X =r$ if and only if $d(s_1) \neq a_1^*$, and $X =r^*$ if and only if $d(s_1) = a_1^*$. Since $\overline{r}^*>\overline{r}$, the optimal decision rule $d^*$ must satisfy: $d^*(s_1) = a_1^*$. Similarly, since $\frac{\delta + \epsilon}{2\delta + \epsilon} > \frac{1}{2}$ we must have: $d^*(s_0) = a_0^*$. The optimal gain is thus:
\begin{align*}
\rho^* = \frac{R_{\max}}{2}\times\frac{\left(\delta + \epsilon\right)\left(\overline{\tau} +\eta T_{\max} \right)}{\left(2\delta + \epsilon\right)\overline{\tau}}
\end{align*}

\begin{figure}
\begin{tikzpicture}
	\useasboundingbox (-1.6,-2) rectangle (12,2);
	
	\tikzset{VertexStyle/.style = {draw, 
									shape          = circle,
	                                 text           = black,
	                                 inner sep      = 2pt,
	                                 outer sep      = 0pt,
	                                 minimum size   = 24 pt}}
	                                                               
	\tikzset{EdgeStyle/.style   = {->, >=latex}}
	\tikzset{LabelStyle/.style =   {above}}
	\tikzset{EdgeStyle/.append style = {bend left}}
	
		\node[VertexStyle](0) at (3,0) {$ s_{0} $};
	\node[VertexStyle](1) at (9,0){$ s_{1} $};
	
		\draw[->, >=latex](0) to [out=200,in=300,looseness=8]node[below]{$ 1-\delta, \tau, 0$} (0);
	\draw[EdgeStyle](1) to node[LabelStyle]{$ \delta, \tau, r $} (0);
	\draw[EdgeStyle](0) to node[below]{$ \delta , \tau, 0 $} (1);
	\draw[->, >=latex] (1) to [out=20,in=120,looseness=8] node[LabelStyle]{$ 1- \delta, \tau, r $} (1);
	
	\begin{scope}[dashed]
	\draw[EdgeStyle](0) to [out=50,in=135,looseness=1]node[LabelStyle]{$ \delta + \epsilon, \tau, 0 $} (1);
	\draw[EdgeStyle](1) [out=50,in=135,looseness=1]to node[below]{$ \delta, \tau, r^* $} (0);
	\draw[->, >=latex](0) to [out=160,in=60,looseness=8]node[LabelStyle]{$ 1-\delta-\epsilon, \tau, 0$} (0);
	\draw[->, >=latex] (1) to [out=340,in=250,looseness=8] node[below]{$ 1- \delta, \tau, r^* $} (1);
	\end{scope}
\end{tikzpicture}
	\caption{The two-state SMDP $M'$ for the lower-bound on SMDPs. The two special actions $a_0^*$ and $a_1^*$ are shown as dashed lines.} \label{fig:small.smdp}
\end{figure}
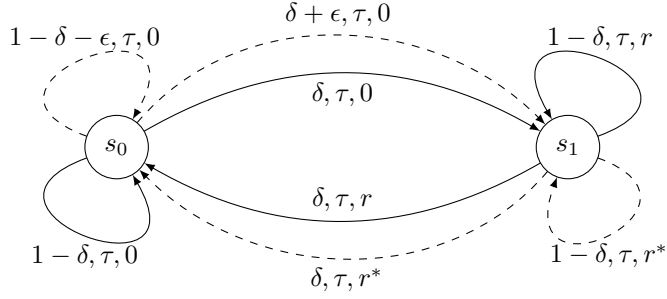

\begin{table}
\begin{align*}
\begin{array}{|c|c|c|c|c|c|}
R.v. \ \ X& X_{\min} & X_{\max} & \mathbb{P}(X = X_{\min}) & \mathbb{P}(X = X_{\max}) & \mathbb{E}[X]\\
\hline
\tau & T_{\min}& T_{\max} & >0 & >0 & \overline{\tau}\\
r & 0 & \frac{1}{2}R_{\max}T_{\max} & 1-p & p & \frac{1}{2}\overline{\tau}R_{\max} \\
r^* & 0 & \frac{1}{2}R_{\max}T_{\max} & 1-p-\eta & p+\eta & \frac{1}{2} \left(\overline{\tau}R_{\max} + \eta T_{\max} R_{\max} \right)
\end{array}
\end{align*}
\caption{Definition of random variables $\tau$, $r$ and $r^*$.} \label{TimeVariables}
\end{table}

The actual SMDP $M$ that we will use to prove the lower bound is built by considering $k = \lfloor \frac{S}{2} \rfloor$ copies of the two-state SMDP $M'$, where only one of the copies has such "good" actions $a_0^*$ and $a_1^*$ (all the other copies have the exact same number of actions as $M'$, but all actions are identical). $A'+1$ additional actions with deterministic transitions are introduced in every $s_0$-state. The reward for each of those actions is zero and the holding time is $T_{\min}$. These actions connect the $s_0$-states of the $k$ copies in a $A'$-ary tree structure on the $s_0$ states: one action goes toward the root and $A'$ actions goes toward the leaves (the same structure is described in section 6 of \citep{Jaksch10}, see Figure 6 for a representation of the $A'$-ary tree). The diameter of such an SMDP $M$ is at most $2\left(\frac{D}{4} + T_{\min}\lceil \log_{A'}{k}\rceil\right) \leq D$. All holding times of $M$ are in $[T_{\min},T_{\max}]$, and all rewards are in $[0,\frac{1}{2}R_{\max}T_{\max}]$. For all $(s,a) \in \mathcal{S} \times \mathcal{A}_s $ we have: $\overline{r}(s,a) \leq \overline{\tau}(s,a)R_{\max}$. Moreover, $M$ has at most $S$ states and $A$ actions per state.

For the analysis, we will study the simpler-to-learn SMDP $M''$ where all $s_0$-states are merged together as well as all $s_1$-states. The "merged" $s_0$-state is set to be the initial state. $M''$ is thus equivalent to the two-state SMDP $M'$ with $kA'$ available actions in both $s_0$ and $s_1$. Let's assume that the learning algorithm $\mathfrak{A}$ used is fixed. The probability distribution of the stochastic process $(s_0,a_0,\tau_0,r_0,s_1,...)$ is denoted:
\begin{itemize}[noitemsep,topsep=0pt]
\item $\mathbb{P}_{a_0,a_1}$  when $(a_0,a_1)$ are the best actions in respectively $s_0$ and $s_1$,
\item $\mathbb{P}_{\ast}$ when the pair $(a_0, a_1)$ identifying the best actions is first chosen uniformly at random from $\lbrace 1, ..., kA' \rbrace \times \lbrace 1, ..., kA' \rbrace$ before algorithm $\mathfrak{A}$ starts,
\item $\mathbb{P}_{\textrm{unif}0,a1}$ when the pair $a_1$ is the best action in $s_1$ and $\epsilon =0$ (no-optimal actions in $s_0$),
\item $\mathbb{P}_{a0,\textrm{unif}1}$ when the pair $a_0$ is the best action in $s_0$ and $\eta =0$ (no-optimal actions in $s_1$).
\end{itemize}

By construction, it is trivial to see that: ${\mathbb{E}_{\ast}[\Delta(M,\mathfrak{A},s,n)] \geq \mathbb{E}_{\ast}[\Delta(M'',\mathfrak{A},s,n)]}$ ($M''$ is easier to learn). We will show that $\mathbb{E}_{\ast}[\Delta(M'',\mathfrak{A},s_0,n)] = \Omega\left(\left(\sqrt{D'} + \sqrt{T_{\max}} \right)R_{\max}\sqrt{kA'n}\right)$ and the same result can be proved with initial state $s_1$ using similar arguments. This will necessarily imply that there exists at least one choice of pair $(a_0,a_1)$ for which, for all states $s$ we have: ${\mathbb{E}_{a_0,a_1}[\Delta(M,\mathfrak{A},s,n)] = \Omega\left(\left(\sqrt{D} + \sqrt{T_{\max}}\right)R_{\max}\sqrt{SAn}\right)}$. 

As already argued by \citep{Jaksch10}, for the analysis it is sufficient to consider algorithms with deterministic strategies for choosing actions.

We assume algorithm $\mathfrak{A}$ is run for $n$ decision steps, which means that $n+1$ states are visited in total (including the last state in which no action is taken). Let's denote by $N_0$ and $N_1$ the number of visits in states $s_0$ and $s_1$ respectively, last state excluded. For any ${(a, \overline{a}) \in \mathcal{A}_{s_0} \times \mathcal{A}_{s_1}}$, let's denote by $N_0^{a}$ and $N_1^{\overline{a}}$ the respective number of times actions $a$ and $\overline{a}$ are taken. Finally, let's denote by $N_0^{*}$ and $N_1^{*}$ the respective number of times best actions in $s_0$ and $s_1$ are taken. The SMDP $M''$ has the same transition probabilities as the MDP considered by \citep{Jaksch10} and we can use their proof to show that for any choice of best actions $(a_0,a_1)$:
\begin{align*}
&\mathbb{E}_{\textrm{unif}0,a_1}[N_1] \geq \frac{n}{2} - \frac{1}{2\delta} = \frac{n}{2} - \frac{D'}{2 \overline{\tau}}\\
&\mathbb{E}_{a_0,a_1}[N_1] \leq \frac{n}{2} + \frac{\mathbb{E}_{a_0,a_1}[N_0^*]\epsilon}{\delta} + \frac{1}{2\delta} = \frac{n}{2} + \frac{\mathbb{E}_{a_0,a_1}[N_0^{a_0}]\epsilon D'}{2\overline{\tau}} + \frac{D'}{2\overline{\tau}}
\end{align*}

The regret is defined as:
\begin{flalign*}
& \mathbb{E}_{\ast}[\Delta(M'',\mathfrak{A},s_0,n)] = \mathbb{E}_{\ast}\left[\sum_{i=1}^{n}{\tau_i(s_{i-1}, a_i, s_i)}\right] \rho^* - \mathbb{E}_{\ast}\left[\sum_{i=1}^{n}{r_i(s_{i-1}, a_i, s_i)}\right] &
\end{flalign*}
where the total duration is simply:
\begin{flalign*}
& \mathbb{E}_{\ast}\left[\sum_{i=1}^{n}{\tau_i(s_{i-1}, a_i, s_i)}\right] = \sum_{i=1}^{n}\mathbb{E}_{\ast}\left[\tau_i(s_{i-1}, a_i) \right] = n \overline{\tau} &
\end{flalign*}
and the cumulated reward is given by:
\begin{flalign*}
\mathbb{E}_{\ast}\left[\sum_{i=1}^{n}{r_i(s_{i-1}, a_i, s_i)}\right] &= \sum_{i=1}^{n}\mathbb{E}_{\ast}\left[r_i(s_{1}, a_1^*)\mathbb{1}_{s_{i-1}=s_1, a_i = a_1^*} + r_i(s_{1}, a_i \neq a_1^*)\mathbb{1}_{s_{i-1}=s_1,a_i \neq a_1^*} \right]\\
&=\mathbb{E}_{\ast}\left[\overline{r}^*\left(\sum_{i=0}^{n-1}\mathbb{1}_{s_{i}=s_1, a_{i+1} = a_1^*}\right) + \overline{r}\left(\sum_{i=0}^{n-1}\mathbb{1}_{s_{i}=s_1,a_{i+1} \neq a_1^*}\right) \right]\\
&=\mathbb{E}_{\ast}\left[\overline{r}\left(N_1 - N_1^*\right) + \overline{r}^*N_1^* \right]\\
&= \mathbb{E}_{\ast}\left[\overline{r}N_1 + (\overline{r}^*-\overline{r})N_1^* \right]\\
&= \frac{R_{\max}}{2}\overline{\tau}\mathbb{E}_{\ast}\left[N_1\right] + \eta \frac{R_{\max}}{2}T_{\max}\mathbb{E}_{\ast}[N_1^*] &
\end{flalign*}
hence the formula:
\begin{align}\label{RegretDecomposition}
\mathbb{E}_{\ast}[\Delta(M'',\mathfrak{A},s_0,n)] = \overline{\tau}\left(n\rho^* - \frac{R_{\max}}{2}\mathbb{E}_{\ast}[N_1]\right) - \eta \frac{R_{\max}}{2}T_{\max}\mathbb{E}_{\ast}[N_1^*]
\end{align}

\begin{lemma}\label{InformationCI}
Let $f : \lbrace s_0, s_1 \rbrace^{n+1} \times \lbrace T_{\min}, T_{\max} \rbrace^n \times \lbrace 0, \frac{1}{2}R_{\max}T_{\max} \rbrace^n \rightarrow [0,M]$ be any function defined on state/reward sequence $(\bm{s^{n+1}},\bm{\tau^n},\bm{r^n})\in \lbrace s_0, s_1 \rbrace^{n+1} \times \lbrace T_{\min}, T_{\max} \rbrace^n \times \lbrace 0, \frac{1}{2}R_{\max}T_{\max} \rbrace^n$ observed in the SMDP $M''$. Then for any $n \geq 1$, any $0 \leq \delta \leq \frac{1}{2}$, any $0 \leq \epsilon \leq 1-2\delta$, any $0 \leq p \leq \frac{1}{2}$, any $0 \leq \eta \leq 1-2p$, and any $(a_0,a_1) \in \lbrace 1, ..., kA' \rbrace \times \lbrace 1, ..., kA' \rbrace$:
\begin{align*}
&\bigg|\mathbb{E}_{a_0,a_1}\left[ f(\bm{s^{n+1}},\bm{\tau^n},\bm{r^n}) \right] - \mathbb{E}_{\textrm{unif}0,a_1}\left[ f(\bm{s^{n+1}},\bm{\tau^n},\bm{r^n}) \right] \bigg| \leq \frac{M}{2} \frac{\epsilon}{\sqrt{\delta}}  \sqrt{2 \mathbb{E}_{\textrm{unif}0,a_1} \left[ N_0^{a_0} \right]}\\
&\bigg|\mathbb{E}_{a_0,a_1}\left[ f(\bm{s^{n+1}},\bm{\tau^n},\bm{r^n}) \right] - \mathbb{E}_{a_0,\textrm{unif}1}\left[ f(\bm{s^{n+1}},\bm{\tau^n},\bm{r^n}) \right] \bigg| \leq \frac{M}{2} \frac{\eta}{\sqrt{p}}\sqrt{2 \mathbb{E}_{a_0,\textrm{unif}1} \left[ N_1^{a_1} \right]}
\end{align*}

\end{lemma}

\begin{proof}
We refer the reader to Appendix E of \citep{Jaksch10} where a similar Lemma is proved for the MDP-analogue of SMDP $M''$. In the following, we will only stress the main difference with the proof in \citep{Jaksch10}. We know from information theory that:
\begin{align*}
&\bigg|\mathbb{E}_{a_0,a_1}\left[ f(\bm{s^{n+1}},\bm{\tau^n},\bm{r^n}) \right] - \mathbb{E}_{\textrm{unif}0,a_1}\left[ f(\bm{s^{n+1}},\bm{\tau^n},\bm{r^n}) \right] \bigg| \leq \frac{M}{2}  \sqrt{2 \log(2) KL\left( \mathbb{P}_{\textrm{unif}0,a_1} \Vert \mathbb{P}_{a_0,a_1}\right)}\\
&\bigg|\mathbb{E}_{a_0,a_1}\left[ f(\bm{s^{n+1}},\bm{\tau^n},\bm{r^n}) \right] - \mathbb{E}_{a_0,\textrm{unif}1}\left[ f(\bm{s^{n+1}},\bm{\tau^n},\bm{r^n}) \right] \bigg| \leq \frac{M}{2}  \sqrt{2 \log(2) KL\left(\mathbb{P}_{a_0,\textrm{unif}1} \Vert \mathbb{P}_{a_0,a_1} \right)}
\end{align*}
By the chain rule of Kullback–Leibler divergences, it holds that:
\begin{align*}
&KL\left(\mathbb{P}_{a_0,\textrm{unif}1} \Vert \mathbb{P}_{a_0,a_1}\right) = \sum_{i=1}^{n} KL\left(\mathbb{P}_{a_0,\textrm{unif}1}(s_{i},\tau_{i},r_{i} | \bm{s^{i-1}},\bm{\tau^{i-1}},\bm{r^{i-1}}) \Vert \mathbb{P}_{a_0,a_1}(s_{i},\tau_{i},r_{i} | \bm{s^{i-1}},\bm{\tau^{i-1}},\bm{r^{i-1}}) \right)\\
& \mbox{where  } KL\left(\mathbb{P}(s_{i},\tau_{i},r_{i} | \bm{s^{i-1}},\bm{\tau^{i-1}},\bm{r^{i-1}}) \Vert \mathbb{Q}(s_{i},\tau_{i},r_{i} | \bm{s^{i-1}},\bm{\tau^{i-1}},\bm{r^{i-1}}) \right) =
\\
& \ \ \ \ \ \ \ \ \ \ \ \ \ \ \ \ \ \ \ \ \ \ \ \ \ \ \ \ \ \ \ \ \ \ \ \ \ \ \ \ \ \  \sum_{\bm{s^{i}} \in \mathcal{S}^{i},\bm{\tau^{i}} \in \mathcal{T}, \bm{r^{i}} \in \mathcal{R}^{i}}{\mathbb{P}(\bm{s^{i}},\bm{\tau^{i-1}},\bm{r^{i}})}\log_2{\left(\frac{\mathbb{P}(s_{i},\tau_{i},r_{i}|\bm{s^{i-1}},\bm{\tau^{i-1}},\bm{r^{i-1}})}{\mathbb{Q}(s_{i},\tau_{i},r_{i}|\bm{s^{i-1}},\bm{\tau^{i-1}},\bm{r^{i-1}})}\right)}
\end{align*}

with $\mathcal{S} = \lbrace s_0, s_1 \rbrace$, $\mathcal{T} = \lbrace T_{\min}, T_{\max} \rbrace$ and $\mathcal{R} = \lbrace 0, \frac{1}{2} R_{\max}T_{\max} \rbrace$. The same holds for $\mathbb{P}_{\textrm{unif}0,a_1}$.

Similarly to \citep{Jaksch10} and using the independence between $s_i$, $\tau_i$ and $r_i$, we obtain:
\begin{align*}
KL&\left(\mathbb{P}_{\textrm{unif}0,a_1}(s_{i},\tau_i,r_{i} | \bm{s^{i-1}},\bm{\tau^{i-1}},\bm{r^{i-1}}) \Vert \mathbb{P}_{a_0,a_1}(s_{i},\tau_i,r_{i} | \bm{s^{i-1}},\bm{\tau^{i-1}},\bm{r^{i-1}}) \right)\\
&= \mathbb{P}_{\textrm{unif}0,a_1}(s_{i-1} = s_{0},a_i =a_0) \sum_{ s'\in \mathcal{S}, \tau' \in \mathcal{T}, r' \in \mathcal{R}}{\mathbb{P}_{\textrm{unif}0,a_1}(s',\tau',r'|s_0, a_0)}\log_2{\left(\frac{\mathbb{P}_{\textrm{unif}0,a_1}(s',\tau',r'|s_0, a_0)}{\mathbb{P}_{a_0,a_1}(s',\tau',r'|s_0, a_0)}\right)}\\
&= \mathbb{P}_{\textrm{unif}0,a_1}(s_{i-1} = s_{0},a_i =a_0) \sum_{ s' \in \mathcal{S}}{\mathbb{P}_{\textrm{unif}0,a_1}(s'|s_0, a_0)}\log_2{\left(\frac{\mathbb{P}_{\textrm{unif}0,a_1}(s'|s_0, a_0)}{\mathbb{P}_{a_0,a_1}(s'|s_0, a_0)}\right)}\\
&= \mathbb{P}_{\textrm{unif}0,a_1}(s_{i-1} = s_{0},a_i =a_0) \left(\delta\log_2 \left( \frac{\delta}{\delta+\epsilon}\right) + (1-\delta)\log_2 \left( \frac{1-\delta}{1-\delta-\epsilon}\right) \right)
\end{align*}
\begin{align*}
KL&\left(\mathbb{P}_{a_0,\textrm{unif}1}(s_{i},\tau_i,r_{i} | \bm{s^{i-1}},\bm{\tau^{i-1}},\bm{r^{i-1}}) \Vert \mathbb{P}_{a_0,a_1}(s_{i},\tau_i,r_{i} | \bm{s^{i-1}},\bm{\tau^{i-1}},\bm{r^{i-1}}) \right)\\
&= \mathbb{P}_{a_0,\textrm{unif}1}(s_{i-1} = s_{1},a_i =a_1) \sum_{ s'\in \mathcal{S}, \tau' \in \mathcal{T}, r' \in \mathcal{R}}{\mathbb{P}_{a_0,\textrm{unif}1}(s',\tau',r'|s_1, a_1)}\log_2{\left(\frac{\mathbb{P}_{a_0,\textrm{unif}1}(s',\tau',r'|s_1, a_1)}{\mathbb{P}_{a_0,a_1}(s',\tau',r'|s_1, a_1)}\right)}\\
&= \mathbb{P}_{a_0,\textrm{unif}1}(s_{i-1} = s_{1},a_i =a_1) \sum_{r' \in \mathcal{R}}{\mathbb{P}_{a_0,\textrm{unif}1}(r'|s_1, a_1)}\log_2{\left(\frac{\mathbb{P}_{a_0,\textrm{unif}1}(r'|s_1, a_1)}{\mathbb{P}_{a_0,a_1}(r'|s_1, a_1)}\right)}\\
&= \mathbb{P}_{a_0,\textrm{unif}1}(s_{i-1} = s_{1},a_i =a_1) \left(p\log_2 \left( \frac{p}{p+\eta}\right) + (1-p)\log_2 \left( \frac{1-p}{1-p-\eta}\right) \right)
\end{align*}

Using Lemma 20 of \citep{Jaksch10} we have that under conditions $0 \leq \delta \leq \frac{1}{2}$, $0 \leq \epsilon \leq 1-2\delta$, $0 \leq p \leq \frac{1}{2}$, and $0 \leq \eta \leq 1-2p$  the following inequalities hold:
\begin{align*}
&\delta\log_2 \left( \frac{\delta}{\delta+\epsilon}\right) + (1-\delta)\log_2 \left( \frac{1-\delta}{1-\delta-\epsilon}\right)  \leq \frac{\epsilon^2}{\delta \log(2)} \\
\mbox{and   } \ &p\log_2 \left( \frac{p}{p+\eta}\right) + (1-p)\log_2 \left( \frac{1-p}{1-p-\eta}\right) \leq \frac{\eta^2}{p \log(2)}
\end{align*}
which concludes the proof.
\end{proof}

Note that by assumption: $\epsilon \leq \delta \leq \frac{1}{3} \leq 1 - 2\delta$ and $\eta \leq p \leq \frac{1}{3} \leq 1 - 2p$.

We can bound $\mathbb{E}_{\ast}\left[N_1\right]$ using Lemma \ref{InformationCI} as is done in \citep{Jaksch10}. This is because $N_1$ can be written as a function of $(\bm{s^{n+1}},\bm{\tau^n},\bm{r^n})$. Since the computations are rigorously the same except that $\delta = \frac{\overline{\tau}}{D'}$ instead of $\frac{1}{D'}$, we give the results without any further details:
\begin{align*}
\begin{split}
\mathbb{E}_{\ast}\left[N_1\right] \leq \frac{n}{2} + \frac{D'}{2\overline{\tau}} + \frac{\epsilon n D'}{2\overline{\tau}kA'} + \frac{\epsilon D'^ 2}{2\overline{\tau}^2 kA'}+ \frac{\epsilon^ 2 n D'}{2\overline{\tau}kA'}\sqrt{\frac{D'kA'n}{\overline{\tau}}} + \frac{\epsilon^ 2 n D'^2}{2\overline{\tau}^2 kA'}\sqrt{kA'}
\end{split}
\end{align*}
Taking into account the fact that by assumption $n \geq DSA \geq 16 D'kA'$ we get:
\begin{align}\label{FirstTerm1}
\begin{split}
\mathbb{E}_{\ast}\left[N_1\right] \leq \frac{n}{2} + \frac{D'}{2\overline{\tau}} + \frac{\epsilon n D'}{\overline{\tau}}\left(\frac{1}{2 k A'} + \frac{1}{32\overline{\tau} k^2 A'^2}\right) + \frac{\epsilon^ 2 n D'}{\overline{\tau}kA'}\sqrt{\frac{D'kA'n}{\overline{\tau}}} \left(\frac{1}{2} + \frac{1}{8\sqrt{kA'}} \right)
\end{split}
\end{align}

Given that $\delta \geq  \epsilon \geq 0$ we have the following inequality:
\begin{align}\label{FirstTerm2}
\begin{split}
\rho^* - \frac{R_{\max}}{4} &= \frac{R_{\max}}{2} \times \frac{\epsilon \overline{\tau} + (\delta + \epsilon)\eta T_{\max}}{2 (2 \delta + \epsilon) \overline{\tau} } \geq \frac{R_{\max}}{2} \times \left( \frac{\epsilon D'}{6\overline{\tau}} + \frac{\eta T_{\max}}{6 \overline{\tau} } \right)
\end{split}
\end{align}

Applying Lemma \ref{InformationCI} ($N_1^*$ is a function of $(\bm{s^{n+1}},\bm{\tau^{n}},\bm{r^n})$) and Jensen's inequality we get ${\forall a_0 \in \lbrace 1,...,kA' \rbrace}$:
\begin{align*}
\mathbb{E}_{\ast}\left[N_1^*\right] = \frac{1}{kA'} \sum_{a_1 = 1}^{kA'} \mathbb{E}_{a_0,a_1}\left[N_1^{a_1}\right] \leq \frac{\mathbb{E}_{a_0,\textrm{unif}1}\left[N_1\right]}{kA'} 
+ \frac{n}{2kA'}\eta \sqrt{\frac{2kA'T_{\max}}{\overline{\tau}}\mathbb{E}_{a_0,\textrm{unif}1}\left[N_1\right]}
\end{align*}

We will now derive an upper-bound on $\mathbb{E}_{a_0,\textrm{unif}1}\left[N_1\right]$.

\begin{lemma}\label{ArithmeticoGeometricIneq}
Let $\left(u_n\right)_{n \in \mathbb{N}} \in \mathbb{R}^\mathbb{N}$ be any real sequence satisfying the following arithmetico-geometric recurrence relation:
\begin{align*}
\forall n \in \mathbb{N}, \ \  u_{n+1} \geq q u_{n} + r \ \ \ \mbox{where} \ (q,r) \in \mathbb{R} \setminus \lbrace 1 \rbrace \times \mathbb{R}
\end{align*}
Then we have that:
\begin{align*}
\forall n \in \mathbb{N}, \ \ u_{n} \geq \left( u_0 - \frac{r}{1-q} \right)\alpha^n + \frac{r}{1-q}
\end{align*}
\end{lemma}

\begin{proof}
Defining sequence $v_n = u_n - \frac{r}{1-q}$, we have:
\begin{align*}
\forall n \in \mathbb{N}, \ \ v_{n+1} = u_{n+1} - \frac{r}{1-q} \geq q u_{n} + r - \frac{r}{1-q} = q \left(u_{n} - \frac{r}{1-q} \right) = q v_{n}
\end{align*}
By trivial induction we get: $\forall n \in \mathbb{N}, \ \ v_{n} \geq v_0 q^n$. The result follows by replacing $v_{n}$ by ${u_n - \frac{r}{1-q}}$.
\end{proof}

By the law of total probability and since $\mathbb{P}_{a_0,\textrm{unif}1}\left(s_{i} = s_0\right) + \mathbb{P}_{a_0,\textrm{unif}1}\left(s_{i} = s_1\right) =1$ we have:
\begin{align*}
\forall i \geq 0, \ \ \mathbb{P}_{a_0,\textrm{unif}1}\left(s_{i+1} = s_0\right) = &\mathbb{P}_{a_0,\textrm{unif}1}\left(s_{i+1} = s_0 | s_i = s_0\right)\mathbb{P}_{a_0,\textrm{unif}1}\left(s_{i} = s_0\right) \\
&+ \mathbb{P}_{a_0,\textrm{unif}1}\left(s_{i+1} = s_0 | s_i = s_1\right)\mathbb{P}_{a_0,\textrm{unif}1}\left(s_{i} = s_1\right)\\
= &\mathbb{P}_{a_0,\textrm{unif}1}\left(s_{i} = s_0\right)\big( \mathbb{P}_{a_0,\textrm{unif}1}\left(s_{i+1} = s_0 | s_i = s_0\right)\\
&- \mathbb{P}_{a_0,\textrm{unif}1}\left(s_{i+1} = s_0 | s_i = s_1\right)\big) + \mathbb{P}_{a_0,\textrm{unif}1}\left(s_{i+1} = s_0 | s_i = s_1\right)\\
\geq &\mathbb{P}_{a_0,\textrm{unif}1}\left(s_{i} = s_0\right)(1-2\delta-\epsilon) + \delta
\end{align*}
Since the initial state is $s_0$ and $2\delta+\epsilon \leq 3\delta$ we have by Lemma \ref{ArithmeticoGeometricIneq}:
\begin{align*}
&\forall i \geq 0, \ \ \ \mathbb{P}_{a_0,\textrm{unif}1}\left(s_{i} = s_0\right) \geq \left(1-\frac{\delta}{2\delta+\epsilon}\right)(1-2\delta-\epsilon)^i + \frac{\delta}{2\delta+\epsilon} \geq \frac{1}{3}\\
&\implies \mathbb{E}_{a_0,\textrm{unif}1}\left[N_0\right] = \sum_{i=0}^{n} \mathbb{P}_{a_0,\textrm{unif}1}\left(s_{i} = s_0\right) \geq \frac{n}{3}\\
&\implies \mathbb{E}_{a_0,\textrm{unif}1}\left[N_1\right]\leq \frac{2n}{3}\\
&\implies \mathbb{E}_{\ast}\left[N_1^*\right] \leq \frac{2n}{3kA'} + \frac{n\eta}{kA'} \sqrt{\frac{T_{\max}nkA'}{3\overline{\tau}}}
\end{align*}
Hence the bound:
\begin{align}\label{SecondTerm}
\begin{split}
\eta T_{\max}\mathbb{E}_{\ast}\left[N_1^*\right] \leq \frac{2n\eta T_{\max}}{3kA'}
\end{split} + \frac{n\eta^2 T_{\max}}{kA'} \sqrt{\frac{T_{\max}nkA'}{3\overline{\tau}}}
\end{align}

By setting $\epsilon = c \sqrt{\frac{kA'}{nD'}}$ and $\eta = \kappa \sqrt{\frac{kA'}{n T_{\max}}}$ and incorporating inequalities \ref{FirstTerm1}, \ref{FirstTerm2} and \ref{SecondTerm} into inequality \ref{RegretDecomposition} we obtain:
\begin{align*}
\mathbb{E}_{\ast}[\Delta(M'',\mathfrak{A},s_0,n)] \geq &\left[ \frac{c}{6} - \frac{c}{2kA'} - \frac{c}{32k^2A'^2\overline{\tau}} -\frac{c^2}{2\sqrt{\overline{\tau}}} -\frac{c^2}{8\sqrt{kA'\overline{\tau}}} -\frac{1}{8kA'} \right]\frac{R_{\max}}{2}\sqrt{D'kA'n}\\
&+ \left[ \frac{\kappa}{6} - \frac{2\kappa}{3kA'} - \frac{\kappa^2}{\sqrt{3\overline{\tau}}} \right]\frac{R_{\max}}{2}\sqrt{T_{\max}kA'n}\\
&\geq \left[ \frac{c}{6} - \frac{c}{2kA'} - \frac{c}{32k^2A'^2} -\frac{c^2}{2} -\frac{c^2}{8\sqrt{kA'}} -\frac{1}{8kA'} \right]\frac{R_{\max}}{2}\sqrt{D'kA'n}\\
&+ \left[ \frac{\kappa}{6} - \frac{2\kappa}{3kA'} - \frac{\kappa^2}{\sqrt{3}} \right]\frac{R_{\max}}{2}\sqrt{T_{\max}kA'n}
\end{align*}
For the second inequality, we used the fact that $\overline{\tau} > T_{\min} \geq 1$ (by assumption). If $c$ and $\kappa$ are sufficiently small, then the conditions of lemma \ref{InformationCI} are indeed satisfied and the above polynomials in $c$ and $\kappa$ are non-negative. For example, if $c = \kappa = \frac{1}{5}$:
\begin{align*}
& n \geq DSA \geq 16D'kA' \implies \epsilon = \frac{1}{5} \sqrt{\frac{kA'}{n D'}} \leq \frac{\delta}{20} < \delta\\
& n \geq T_{\max}SA \geq 4T_{\max}kA' \implies \eta = \frac{1}{5} \sqrt{\frac{kA'}{n T_{\max}}} \leq \frac{p}{10} < p\\
&\mbox{and    } \mathbb{E}_{\ast}[\Delta(M'',\mathfrak{A},s_0,n)] \geq 0.0015 \times \left(\sqrt{D'} + \sqrt{T_{\max}}\right) R_{\max} \sqrt{kA'n}\\
\end{align*}

\subsection{Lower Bound for MDPs with options}

We first note that SMDP $M'$ depicted in Fig. \ref{fig:small.smdp} cannot be converted into an MDP with options. This is due to the fact that $\tau_i(s_{i-1},a_i,s_i)$ and $r_i(s_{i-1},a_i,s_i)$ were assumed to be independent and the fact that $\mathbb{P}(r^* = \frac{1}{2}R_{\max}T_{\max}) > \mathbb{P}(\tau = T_{\max})$. However, it is possible to prove a slightly smaller lower bound for a family of SMDPs that can be transformed into an equivalent MDP with options. We will first present the SMDPs and the lower bound and then we will describe how to transform them into an MDP with options.

The family of SMDPs is constructed in the same way as previously except that we use a slightly different SMDP $M'$, represented on Fig. \ref{fig:small.smdp2} with the random variables given in Table \ref{TimeVariables2}. We take: $p = \frac{1}{2(T_{\max} - T_{\min})}$ and $\delta = \frac{4 \overline{\tau}}{D}$. By assumption $p<\frac{1}{3}$ and $\delta <  \frac{1}{2}$. As before, we assume that for all $i$, $\tau_i(s_{i-1},a_i,s_i)$ and $r_{i}(s_{i-1},a_i,s_i)$ are independent of the next state $s_{i}$. But the main difference with the previous lower bound is that we assume that $r_i$ and $\tau_i$ are strongly correlated, namely: $r_i(s_{i-1},a_i,s_i) = R_{\max} \mathbb{1}_{\lbrace{s_{i-1} = s_1 \rbrace}}\tau_i(s_{i-1},a_i)$. We assume $\epsilon \leq \delta$ and $\eta \leq 1 - 2p$. The optimal gain of $M'$ is reached when $a_0^*$ and $a_1^*$ are chosen in $s_0$ and $s_1$ respectively and is equal to:
\begin{align*}
\rho^* = R_{\max}\times\frac{\left(\delta + \epsilon\right)\left(\overline{\tau} +\eta (T_{\max}-T_{\min}) \right)}{\left(2\delta + \epsilon\right)\overline{\tau} + \left(\delta + \epsilon\right) (T_{\max}-T_{\min})}
\end{align*}
By adapting the proof of the lower bound for general SMDPs, we can obtain the following result:
\begin{align*}
\mathbb{E}_{\ast}[\Delta(M'',\mathfrak{A},s_0,n)] \geq &\left[ \frac{c}{6} - \frac{c}{2kA'} - \frac{c}{32k^2A'^2\overline{\tau}} -\frac{c^2}{2\sqrt{\overline{\tau}}} -\frac{c^2}{8\sqrt{kA'\overline{\tau}}} -\frac{1}{8kA'} \right]R_{\max}\sqrt{D'kA'n}\\
&+ \left[ \frac{\kappa}{6} - \frac{2\kappa}{3kA'} - \frac{\kappa^2}{\sqrt{3/2}} \right]R_{\max}\sqrt{(T_{\max} - T_{\min})kA'n}\\
&\geq \left[ \frac{c}{6} - \frac{c}{2kA'} - \frac{c}{32k^2A'^2} -\frac{c^2}{2} -\frac{c^2}{8\sqrt{kA'}} -\frac{1}{8kA'} \right]R_{\max}\sqrt{D'kA'n}\\
&+ \left[ \frac{\kappa}{6} - \frac{2\kappa}{3kA'} - \frac{\kappa^2}{\sqrt{3/2}} \right]R_{\max}\sqrt{(T_{\max} - T_{\min})kA'n}
\end{align*}
by setting $\epsilon = c \sqrt{\frac{kA'}{nD'}}$ and $\eta = \kappa \sqrt{\frac{kA'}{n (T_{\max} - T_{\min})}}$. It is then possible to tune $c$ and $\kappa$ so that $\epsilon$ and $\eta$ satisfy the constraints and:
\begin{align*}
\mathbb{E}_{\ast}[\Delta(M'',\mathfrak{A},s_0,n)] = \Omega \left(\left(\sqrt{D'} + \sqrt{T_{\max} - T_{\min}}\right) R_{\max} \sqrt{kA'n} \right)
\end{align*}

\begin{figure}
\begin{tikzpicture}
	\useasboundingbox (-1.6,-2) rectangle (12,2);
	
	\tikzset{VertexStyle/.style = {draw, 
									shape          = circle,
	                                 text           = black,
	                                 inner sep      = 2pt,
	                                 outer sep      = 0pt,
	                                 minimum size   = 24 pt}}
	                                                               
	\tikzset{EdgeStyle/.style   = {->, >=latex}}
	\tikzset{LabelStyle/.style =   {above}}
	\tikzset{EdgeStyle/.append style = {bend left}}
	
		\node[VertexStyle](0) at (3,0) {$ s_{0} $};
	\node[VertexStyle](1) at (9,0){$ s_{1} $};
	
		\draw[->, >=latex](0) to [out=200,in=300,looseness=8]node[below]{$ 1-\delta, \tau, 0$} (0);
	\draw[EdgeStyle](1) to node[LabelStyle]{$ \delta, \tau, r $} (0);
	\draw[EdgeStyle](0) to node[below]{$ \delta , \tau, 0 $} (1);
	\draw[->, >=latex] (1) to [out=20,in=120,looseness=8] node[LabelStyle]{$ 1- \delta, \tau, r $} (1);
	
	\begin{scope}[dashed]
	\draw[EdgeStyle](0) to [out=50,in=135,looseness=1]node[LabelStyle]{$ \delta + \epsilon, \tau, 0 $} (1);
	\draw[EdgeStyle](1) [out=50,in=135,looseness=1]to node[below]{$ \delta, \tau^*, r^* $} (0);
	\draw[->, >=latex](0) to [out=160,in=60,looseness=8]node[LabelStyle]{$ 1-\delta-\epsilon, \tau, 0$} (0);
	\draw[->, >=latex] (1) to [out=340,in=250,looseness=8] node[below]{$ 1- \delta, \tau^*, r^* $} (1);
	\end{scope}
\end{tikzpicture}
	\caption{The two-state SMDP $M'$ for the lower-bound on MDPs with options. The two special actions $a_0^*$ and $a_1^*$ are shown as dashed lines.} \label{fig:small.smdp2}
\end{figure}
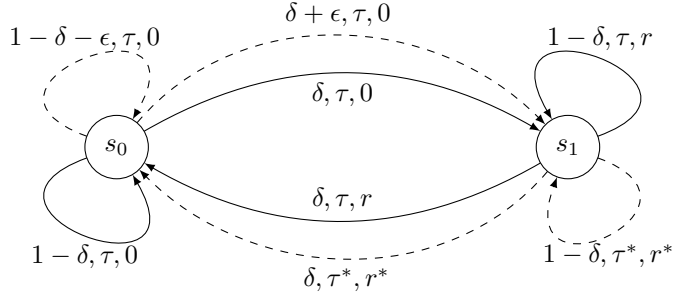

\begin{table}
\begin{align*}
\begin{array}{|c|c|c|c|c|c|}
R.v. \ \ X& X_{\min} & X_{\max} & \mathbb{P}(X = X_{\min}) & \mathbb{P}(X = X_{\max}) & \mathbb{E}[X]\\
\hline
\tau & T_{\min}& T_{\max} & 1-p & p & \overline{\tau} = T_{\min} + p(T_{\max} - T_{\min})\\
\tau^* & T_{\min}& T_{\max} & 1-p-\eta & p+\eta & \overline{\tau} + \eta(T_{\max} - T_{\min})\\
r & R_{\max}T_{\min} & R_{\max}T_{\max} & 1-p & p & \overline{\tau}R_{\max} \\
r^* & R_{\max}T_{\min} & R_{\max}T_{\max} & 1-p-\eta & p+\eta & \overline{\tau}R_{\max} + \eta R_{\max}\left(T_{\max} - T_{\min}\right)
\end{array}
\end{align*}
\caption{Definition of random variables $\tau$, $\tau^*$, $r$ and $r^*$.} \label{TimeVariables2}
\end{table}

SMDPs $M'$ and $M''$ can be transformed into equivalent MDPs with options. We illustrate this transformation on Fig. \ref{fig:smdp2mdp} for an action $a_1 \in \mathcal{A}_{s_1}$ different than $a_1^*$. The same method can be applied for the other actions. The idea consists in adding new states (blank states in Fig. \ref{fig:smdp2mdp}) and primitive actions between those states. Note that the states added are just "hidden" states from which no option can be started. Thus, they should not be counted in the number of states for the lower bound. In our example it is sufficient to consider primitive actions with constant (i.e., deterministic) reward. On Fig. \ref{fig:smdp2mdp} we give the probabilities of each primitive action when $a_1 \in \mathcal{A}_{s_1} \setminus \lbrace{ a_1^* \rbrace}$ is executed.

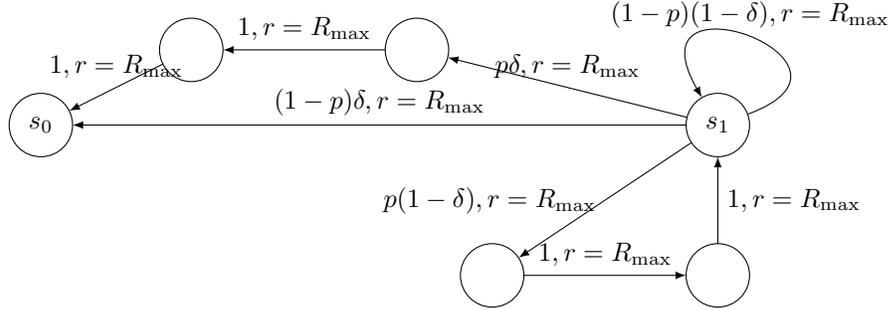
\begin{figure}
\begin{tikzpicture}
	\useasboundingbox (-0.5,-2.2) rectangle (12,2);
	
	\tikzset{VertexStyle/.style = {draw, 
									shape          = circle,
	                                 text           = black,
	                                 inner sep      = 2pt,
	                                 outer sep      = 0pt,
	                                 minimum size   = 24 pt}}
	                                                               
	\tikzset{EdgeStyle/.style   = {->, >=latex}}
	\tikzset{LabelStyle/.style =   {above}}
	
		\node[VertexStyle](0) at (3,0) {$ s_{0} $};
	\node[VertexStyle](1) at (12,0){$ s_{1} $};
	\node[VertexStyle](2) at (9,-2){};
	\node[VertexStyle](3) at (12,-2){};
	\node[VertexStyle](4) at (8,1){};
	\node[VertexStyle](5) at (5,1){};
	
		\draw[EdgeStyle](1) to node[LabelStyle]{$ (1-p) \delta, r = R_{\max} $} (0);
	\draw[EdgeStyle](1) to node[left]{$ p(1-\delta), r = R_{\max}$} (2);
	\draw[EdgeStyle](2) to node[LabelStyle]{$ 1, r = R_{\max}$} (3);
	\draw[EdgeStyle](3) to node[right]{$ 1, r = R_{\max}$} (1);
	\draw[EdgeStyle](1) to node[LabelStyle]{$ p \delta, r = R_{\max} $} (4);
	\draw[EdgeStyle](4) to node[LabelStyle]{$ 1, r = R_{\max}  $} (5);
	\draw[EdgeStyle](5) to node[LabelStyle]{$ 1, r = R_{\max}  $} (0);
	\draw[->, >=latex] (1) to [out=20,in=120,looseness=8] node[LabelStyle]{$ (1-p)(1-\delta), r=R_{\max}$} (1);
	
	\tikzset{EdgeStyle/.append style = {bend left}}
	
\end{tikzpicture}
	\caption{Decomposition of an action $a_1 \in \mathcal{A}_{s_1} \setminus \lbrace{ a_1^* \rbrace}$ into a set of primitive actions in an MDP. In this example, $T_{\max} = 3$ and $T_{\min} = 1$.} \label{fig:smdp2mdp}
\end{figure}
 
\section{Distribution of the holding time and reward of a finite Markov option (Lem.~\ref{lem:MDPwithOptions})}\label{DistributionOption}

We denote by $\mathbb{R}^{+}$ and $\mathbb{R}^{+*}$ the set of positive and non-negative reals respectively. For the definition of sub-exponential random variables, we refer to Def.~\ref{DefSubExp}, while sub-Gaussian random variables are defined as follows.

\begin{definition}[\cite{Wainwright15}]\label{DefSubGauss}
A random variable $X$ with mean $\mu < +\infty$ is said to be sub-Gaussian if and only if there exists $\sigma \in \mathbb{R}^{+}$ such that:
\begin{equation}\label{eq1}
\mathbb{E} [ e^{\lambda (X-\mu)} ] \leq e^{\frac{\sigma^2 \lambda^2}{2}} \ \textrm{  for all  } \ \lambda \in \mathbb{R}.
\end{equation}
\end{definition}

A finite\footnote{Note that if at least one option is not (almost surely) finite, the learning agent can potentially be stuck executing that option forever and the problem is ill-posed.} Markov option can be seen as an absorbing Markov Chain together with a reward process (i.e., a finite Markov option can be seen as an absorbing Markov Reward Process). To see this we add a new state $\widetilde{s}$ for every state $s$ for which $\beta_o(s)>0$. We then add a transitions from $s$ to $\widetilde{s}$ with probability $\beta_o(s)>0$ and reward 0, and we add a self-loop on $\widetilde{s}$ with probability 1 and reward 0 ($\widetilde{s}$ is an absorbing state). The Markov Reward Process obtained is indeed absorbing since we assumed the option to be a.s. finite, and it is equivalent to the original option (same reward and holding time). Let's denote by $P$ the transition matrix of the Markov Chain. In canonical form we have:
\begin{align*}
P = \begin{bmatrix}
Q & R\\
0 & I_r
\end{bmatrix}
\end{align*}
where $r$ is the number of absorbing states, $I_r$ is the identity matrix of dimension $r$, $Q$ is the transition matrix between non-absorbing states and $R$ the transition matrix from non-absorbing to absorbing states. If the option is a.s. finite then $Q$ is necessarily (strictly) sub-stochastic ($Qe \leq e$ where $e = (1,...,1)^\intercal$ and $\exists j$ s.t. $(Qe)_j < 1$) and irreducible (no recurrent class). It is well-known that such a matrix has a spectral radius strictly smaller than 1 ($\rho(Q) <1$) and thus $I-Q$ is invertible (where $I$ is the identity matrix). The holding time $\tau(s,o,s')$ of any option $o$ is defined as the first time absorbing state $s'$ is reached starting from state $s$: $\inf \lbrace n\geq 1: s_n = s' \ \mbox{ with } \ s_0 = s  \rbrace$ where $(s_n)_n$ is the sequence of states in the absorbing Markov Chain defined by $o$. It is well-known in the literature \citep{Buchholz2014} that this type of stopping times have Discrete Phase-Type distributions, with probability mass function given by:
\begin{align*}
\forall k \in \mathbb{N}^*, \ \ \mathbb{P}(\tau(s,a,s')=k) = e_s^\intercal Q^{k-1}R e_{s'}
\end{align*}
where $e_s = (0,0,... 0,1,0, ...,0)^\intercal$ is a vector of all zeros except in state $s$ where it equals 1. These distributions generalize the geometric distribution (defined in dimension 1) to higher dimensions. The Laplace transform can be computed as follows (we simplify notations and denote: $\tau \leftarrow \tau(s,o,s')$ and $\overline{\tau} \leftarrow \overline{\tau}(s,o,s') = \mathbb{E}[\tau(s,o,s')]$):
\begin{align*}
\mathbb{E}\left[ e^{\lambda (\tau - \overline{\tau})} \right] = \sum_{k=1}^{\infty}{e^{\lambda (k- \overline{\tau})}e_s^\intercal Q^{k-1}R e_{s'}} = e^{\lambda (1- \overline{\tau})}e_s^\intercal \left[\sum_{k=0}^{\infty}{\left( e^{\lambda} Q \right)^{k}} \right]R e_{s'}
\end{align*}
The term $\sum_{k=0}^{\infty}{\left( e^{\lambda} Q \right)^{k}}$ is finite if and only if $e^{\lambda} \rho(Q) <1$, in which case we have:
\begin{align*}
\mathbb{E}\left[ e^{\lambda (\tau - \overline{\tau})} \right] = e^{\lambda (1- \overline{\tau})}e_s^\intercal \left(I - e^\lambda Q \right)^{-1} R e_{s'}
\end{align*}
and otherwise: $\mathbb{E}\left[ e^{\lambda (\tau - \overline{\tau})} \right] = +\infty$. Note that $e^{\lambda} \rho(Q) <1$ if and only if either $\lambda < -\log \left(\rho(Q) \right)$ or $\rho(Q) = 0$. We will now analyse the two cases separately:

\begin{enumerate}
\item $\rho(Q)=0$ if and only if all the eigenvalues of $Q$ in $\mathbb{C}$ are $0$, if and only if $Q$ is nilpotent ($\exists n >0$ s.t. $Q^n =0$). This is because $Q$ can always be triangularized in $\mathbb{C}$: $Q=U T U^{-1}$ where $T$ is upper-triangular with the eigenvalues of $Q$ on the diagonal that is, only zeros if $\rho(Q)=0$. This implies that $\exists n >0$ s.t. $T^n = U^{-1} Q^n U =0 \implies Q^n = 0$ hence $Q$ is nilpotent. The reverse is obviously true: if $Q$ is nilpotent then $\rho(Q)=0$, (otherwise there would exist $\lambda \neq 0$, $v\neq 0$ and $n>0$ s.t. $Q^n =0$ and $Qv = \lambda v \implies Q^nv = \lambda^n v = 0$, which is absurd). By definition, matrix $Q$ is nilpotent of order $n$ if and only if the Markov Chain reaches an absorbing state in at most $n$ steps (a.s.). In conclusion, $\rho(Q)=0$ if and only if the option is almost surely bounded. This happens if and only if there is no cycle in the option (with probability 1, every non-absorbing state is visited at most once).
\item In the case where $\rho(Q)>0$: it is clear that $\mathbb{E}\left[ e^{\lambda (\tau - \overline{\tau})} \right]$ can not be bounded by a function of the form $\lambda \rightarrow e^{\frac{\sigma^2 \lambda^2}{2}}$ for $\lambda \geq -\log \left(\rho(Q) \right)$ so $\tau(s,o,s')$ is not sub-Gaussian (Definition \ref{DefSubGauss}). However, since $\rho(Q)<1$ we can choose $0< c_0 < -\log \left(\rho(Q)\right)$ and we have $\mathbb{E}\left[ e^{\lambda (\tau - \overline{\tau})} \right]< +\infty$ for all $|\lambda| < c_0$, which implies that $\tau(s,o,s')$ is sub-exponential (Definition \ref{DefSubExp}).
\end{enumerate}

In conclusion, either option $o$ contains inner-loops (some states are visited several times with non-zero probability) in which case the distribution of $\tau(s,o,s')$ is sub-Exponential but not sub-Gaussian, or $o$ has no inner-loop in which case $o$ is bounded (and thus sub-Gaussian). There is no other alternative.

The distribution of rewards $r(s,o,s')$ is not as simple: the reward of an option is the sum of all micro-rewards obtained at every time step before the option ends, and every micro-reward earned at each time step can have a different distribution. The only constraint is that all micro-rewards should be (a.s.) bounded between 0 and $R_{\max}$. As a result, if $\tau(s,o,s')$ is a.s. bounded (by let's say $T_{\max}$) then $r(s,o,s')$ is also a.s. bounded (by $R_{\max}T_{\max}$). But if $\tau(s,o,s')$ is unbounded then $r(s,o,s')$ may still be bounded if for example, all micro-rewards are 0. If however all micro-rewards are equal to $R_{\max}$ then $r(s,o,s')$ has a discrete phase-type distribution just like $\tau(s,o,s')$. $r(s,o,s')$ can thus be unbounded (and even not sub-Gaussian). However, we will show that $r(s,o,s')$ is always sub-Exponential. Using the law of total expectations and the fact that $\mathbb{P}\left( r \leq R_{\max}\tau \right) = 1$ we have:
\begin{align*}
\forall \lambda >0, \ \ \mathbb{E}\left[ e^{\lambda (r - \overline{r})} \right] = \sum_{k=1}^{\infty}\mathbb{E}\left[e^{\lambda (r- \overline{r})} \big| \tau =k \right] \mathbb{P}(\tau = k) &\leq \sum_{k=1}^{\infty}\mathbb{E}\left[e^{\lambda (R_{\max}\tau- \overline{r})} \big| \tau =k \right] \mathbb{P}(\tau = k)\\
&= \sum_{k=1}^{\infty}\mathbb{E}\left[e^{\lambda (R_{\max}k- \overline{r})} \big| \tau =k \right] \mathbb{P}(\tau = k)\\
&= \sum_{k=1}^{\infty}e^{\lambda (R_{\max}k- \overline{r})} \mathbb{P}(\tau = k)\\
&= e^{\lambda (R_{\max}- \overline{\tau})}e_s^\intercal \left[\sum_{k=0}^{\infty}{\left( e^{\lambda R_{\max}} Q \right)^{k}} \right]R e_{s'}
\end{align*}
We can now conclude as we did for $\tau(s,o,s')$: let $0< c_0 < -\frac{\log \left(\rho(Q)\right)}{R_{\max}}$, for all $0<\lambda < c_0$ the quantity $\mathbb{E}\left[ e^{\lambda (r - \overline{r})} \right]$ is finite. Note that for $\lambda \leq 0$: $\mathbb{E}\left[ e^{\lambda r} \right] \leq 1$ so $\mathbb{E}\left[ e^{\lambda (r - \overline{r})} \right] < +\infty$. By Definition \ref{DefSubExp}, $r(s,o,s')$ is sub-Exponential. 
\section{Equivalent policies in an MDP with options and the induced SMDP (Lem.~\ref{thm:smdp.mdp.equivalence})}

We consider the original MDP $M$ and the SMDP $M_{\O}$ induced by the set of options $\O$. By definition of $M_{\O}$, the reward of an option is equal to the sum of the rewards of all the primitive actions taken until the option terminates (when the option is executed in $M$). Therefore $\sum_{i=1}^{n}r_\mathcal{O}^i = \sum_{i=1}^{N(T_n)}r_\mathcal{O}^i = \sum_{t=1}^{T_n}r_t$ and:
\begin{align}\label{eq:regretDecomposition}
\begin{split}
\Delta(M,\mathfrak{A},s,T_n) &= T_n \rho^*(M) - \sum_{t=1}^{T_n}r_t\\
&= T_n \rho^*(M') + T_n \left(\rho^*(M)-\rho^*(M_{\O})\right) - \sum_{i=1}^{n}r_\mathcal{O}^i \\
&= \Delta(M_{\O},\mathfrak{A},s,n) + T_n \left(\rho^*(M)-\rho^*(M_{\O})\right)
\end{split}
\end{align}
The second part of Lem.~\ref{thm:smdp.mdp.equivalence} is thus proved. We now define the (finite-time) average reward in the two processes
\begin{align*}
\forall T \in \mathbb{N}^*, \ \ \rho^{\pi}(M,s,T) = \mathbb{E}_M^\pi\left[ \frac{\sum_{t=1}^{T}r_t}{T} \bigg| s_0 = s \right]\\
\forall T' \in \mathbb{R}^{+*}, \ \ \rho^{\pi_\mathcal{O}}(M_{\O},s,T') = \mathbb{E}_{M_{\O}}^{\pi_\mathcal{O}}\left[ \frac{\sum_{i=1}^{N(T')}r_\mathcal{O}^i}{T'} \bigg| s_0 = s \right].
\end{align*}
The limit $\lim_{n\rightarrow +\infty} T_n = + \infty$ since the sequence $(T_n)_{n \in \mathbb{N}^*}$ is strictly increasing and unbounded (at least one primitive action is executed before the option ends: ${\forall n \geq 1, T_{n+1} \geq T_n + 1}$). Moreover, $\lim_{T' \rightarrow + \infty} \rho^{\pi_\mathcal{O}}(M_{\O},s,T')$ exists since $\pi_\mathcal{O}$ is stationary and deterministic (see appendix \ref{GainSMDPs}) and by composition of the limit we have
\begin{align*}
\lim_{n \rightarrow + \infty} \rho^{\pi_\mathcal{O}}(M_{\O},s,T_n) = \lim_{T' \rightarrow + \infty} \rho^{\pi_\mathcal{O}}(M_{\O},s,T') = \rho^{\pi_\mathcal{O}}(M_\mathcal{O},s)
\end{align*}
The limit $\lim_{T \rightarrow + \infty} \rho^{\pi}(M,s,T)$ also exists. To see this, we can build an augmented MDP equivalent to $M$ where the state and actions encountered in two different options are duplicated (see section 3 of \citep{Levy2012}). The equivalence between the original and augmented MDPs is in the strong sense: for any optional policy, the corresponding policy in the augmented MDP yields exactly the same reward for any finite horizon. In the augmented MDP, policy $\pi$ is stationary deterministic and we know from MDP theory \citep{Puterman94} that the corresponding average reward exists.
We also have:
\begin{flalign*}
& \forall n \geq 1, \ \ \mathbb{E}_M^\pi\left[ \frac{\sum_{t=1}^{T_n}r_t}{T_n} \bigg| s_0 = s \right] = \mathbb{E}_{M_{\O}}^{\pi_\mathcal{O}}\left[ \frac{\sum_{i=1}^{n}r_\mathcal{O}^i}{T_n} \bigg| s_0 = s \right]& \\ 
&\implies \rho^{\pi_\mathcal{O}}(M_{\O},s) = \lim_{n \rightarrow +\infty}{\rho^{\pi}(M,s,T_n)} = \lim_{T \rightarrow +\infty}{\rho^{\pi}(M,s,T)} = \rho^{\pi}(M,s)&
\end{flalign*}
The first part of Lem.~\ref{thm:smdp.mdp.equivalence} is thus proved. Finally, we know from the literature \citep{Puterman94} that there exists a stationary deterministic optimal policy in the augmented MDP and thus there also exists a stationary deterministic optional policy (a policy using only options in $\O$) in the original MDP. As a result, $\rho^*(M_{\O})$ is the maximal average reward achievable in $M$ using only options in $\O$. In conclusion, the linear term $T_n\left(\rho^*(M)-\rho^*(M_{\O})\right)$ in equation \ref{eq:regretDecomposition} is the minimal asymptotic regret incurred if the learning agent decides to only use options. This linear loss is unavoidable. 
\section{Details of the illustrative experiments}\label{app:experiments}

In this section we will detail the experiments described in section \ref{sec:exp}.

\paragraph{Terminating Condition}
Let's denote the current state by $s_0$ and for all $k \in \lbrace 1 ... m \rbrace$, denote by $s_k$ the state which is $k$ steps on the left to $s_0$. Assume option \textit{LEFT} is taken in $s_0$. By definition, once $s_k$ is reached, the probability of ending the option is given by $\beta_o(s_k) = 1/(m-k+1)$. Since all transitions in the MDP have probability 1 (except at the target), the probability of ending in exactly $k$ steps can be computed as follows:
\begin{itemize}
\item If $k=1$: 
\begin{align*}
\mathbb{P}(\tau = 1) = \beta_o(s_1) = \frac{1}{m}
\end{align*}
\item If $k \geq 1$: 
\begin{align*}
\mathbb{P}(\tau = k) &= \left(\prod_{i = 1}^{k-1}{\left(1-\beta_o(s_i)\right)}\right) \times \beta_o(s_k) \\
&= \left(\prod_{i = 1}^{k-1}{\left(1-\frac{1}{m-i+1}\right)}\right) \times \frac{1}{m-k+1} \\
&= \left(\prod_{i = 1}^{k-1}{\left(\frac{m-i}{m-i+1}\right)}\right) \times \frac{1}{m-k+1} = \frac{1}{m}
\end{align*}
\end{itemize}
By symmetry, the other options (\textit{RIGHT}, \textit{UP} and \textit{DOWN}) have the same holding time.

\paragraph{Expected Holding Time}
Based on the previous result, we can easily compute the expected holding time:
\begin{align*}
\mathbb{E}[\tau] = \sum_{k = 1}^{m}{k \cdot \mathbb{P}(\tau = k)} = \frac{1}{m}\sum_{k = 1}^{m}{k} = \frac{m+1}{2}
\end{align*}

\begin{figure}
\centerline{
\scalebox{0.7} {
\begin{pspicture}(0,-6.29)(13.409312,6.29)
\definecolor{color251c}{rgb}{0.027450980392156862,0.00784313725490196,0.00784313725490196}
\definecolor{color253}{rgb}{0.00784313725490196,0.5098039215686274,0.054901960784313725}
\definecolor{color254b}{rgb}{0.00784313725490196,0.5098039215686274,0.058823529411764705}
\definecolor{color1}{rgb}{0.9921568627450981,0.9372549019607843,0.9372549019607843}
\definecolor{color261b}{rgb}{0.5098039215686274,0.01568627450980392,0.00784313725490196}
\definecolor{color69}{rgb}{0.5803921568627451,0.00392156862745098,0.00392156862745098}
\rput(0.033999998,-6.25){\psgrid[gridwidth=0.014199999,subgridwidth=0.014199999,gridlabels=0.0pt,unit=2.5cm,subgridcolor=color251c](0,0)(0,0)(5,5)
\psset{unit=1.0cm}}
\psframe[linewidth=0.1,linecolor=color253,dimen=outer](10.573,0.269)(7.021,-3.283)
\psframe[linewidth=0.004,dimen=outer,fillstyle=solid,fillcolor=color254b](9.068,-1.216)(8.54,-1.744)
\psline[linewidth=0.096cm,fillcolor=black,arrowsize=0.05291667cm 2.0,arrowlength=1.4,arrowinset=0.4,dotsize=0.07055555cm 2.0]{<-**}(10.834,-1.45)(10.814,0.15)
\psline[linewidth=0.096cm,fillcolor=black,arrowsize=0.05291667cm 2.0,arrowlength=1.4,arrowinset=0.4,dotsize=0.07055555cm 2.0]{<-**}(8.834,-3.53)(7.134,-3.55)
\psline[linewidth=0.096cm,fillcolor=black,arrowsize=0.05291667cm 2.0,arrowlength=1.4,arrowinset=0.4,dotsize=0.07055555cm 2.0]{<-**}(10.834,-1.53)(10.814,-3.17)
\psline[linewidth=0.096cm,fillcolor=black,arrowsize=0.05291667cm 2.0,arrowlength=1.4,arrowinset=0.4,dotsize=0.07055555cm 2.0]{<-**}(8.834,-3.53)(10.454,-3.51)
\usefont{T1}{ptm}{m}{it}
\rput(8.85275,-1.505){\fontsize{18pt}{18pt}\selectfont \color{color1}\bm{$s'$}}
\psframe[linewidth=0.004,dimen=outer,fillstyle=solid,fillcolor=color261b](0.528,6.284)(0.0,5.756)
\usefont{T1}{ptm}{m}{it}
\rput(0.26103124,6.015){\fontsize{18pt}{18pt}\selectfont \color{color1}\bm{$s$}}
\usefont{T1}{ptm}{m}{}
\rput(12.3085315,-0.505){\fontsize{14pt}{14pt}\selectfont $m$-close (up)}
\usefont{T1}{ptm}{m}{}
\rput(12.608532,-2.485){\fontsize{14pt}{14pt}\selectfont $m$-close (down)}
\usefont{T1}{ptm}{m}{}
\rput(10.65,-4){\fontsize{14pt}{14pt}\selectfont $m$-close (right)}
\usefont{T1}{ptm}{m}{}
\rput(7.45,-4){\fontsize{14pt}{14pt}\selectfont $m$-close (left)}
\usefont{T1}{ptm}{m}{n}
\rput(6.897594,4.415){\fontsize{20pt}{20pt}\selectfont \bm{$ \leq D$}}
\psline[linewidth=0.09,linecolor=color69,arrowsize=0.1291667cm 2.0,arrowlength=1.7,arrowinset=0]{->}(0.574,5.99)(3.374,5.97)(3.374,3.97)(5.374,3.97)(5.374,3.45)(6.734,3.47)(6.734,3.05)(7.794,3.05)(7.794,-0.49)(7.794,-0.53)
\end{pspicture} 
}
}
\caption{Upper bound on the diameter of the SMDP used in the experiments.} \label{fig:diameter}
\end{figure}

\paragraph{Diameter}
Let $s$ and $s'$ be two distinct states in the grid. With the options defined above, the expected shortest path from $s$ to $s'$ is obtained if in each visited state on the way to $s'$, we choose an option that goes in the direction of $s'$. For example, if $s$ is the state located in the top left corner of the grid and $s'$ is the target, the expected shortest path is obtained when either \textit{RIGHT} or \textit{DOWN} is taken in every state. With this policy, the expected time to get $m$-close to $s'$ both horizontally and vertically is trivially bounded by $D$ (red path on Fig.~\ref{fig:diameter}). Once we are $m$-close to $s'$ (green square on Fig.~\ref{fig:diameter}, $m = 3$ on this example), we will potentially start cycling until we reach $s'$. On Fig.~ \ref{fig:close.states}, we give an example (in one dimension) of a possible path before reaching $s'$ once in an $m$-close state (the green arrows represent the successive transitions, and $m=3$ on this example). Since all options end after at most $m$ time steps, once we are $m$-close to $s'$, we stay $m$-close with the chosen policy. The expected time it takes to reach $s'$ once we are $m$-close to it is $m(m+1)/2$ both horizontally and vertically. To prove this, we need to solve a linear system. For all $i \in \lbrace 1 ... m-1 \rbrace$, denote by $\tau_{i}$ the time it takes to go from $s$ to the i-th state to the left (respectively right, up or down) when the option chosen is left (respectively right, up or down). The value is the same in all directions by symmetry. We can express the $\tau_{i}$ as follows:
\begin{align}\label{eq:diameterExperiments}
\begin{split}
\tau_1 &= \frac{1}{m} + \frac{1}{m}(2+\tau_{1}) + \dots + \frac{1}{m}(m+\tau_{m-1}) \\
\tau_2 &= \frac{1}{m} \times 2 + \frac{1}{m}(1+\tau_{1}) + \frac{1}{m}(3+\tau_{1}) + \dots + \frac{1}{m}(m+\tau_{m-2}) \\
\tau_3 &= \frac{1}{m} \times 3 + \frac{1}{m}(1+\tau_{2}) + \frac{1}{m}(2+\tau_{1}) + \frac{1}{m}(4+\tau_{1}) + \dots + \frac{1}{m}(m+\tau_{m-3}) \\
&\dots\\
\tau_{i} &= \frac{m+1}{2} + \frac{1}{m}\sum_{j=1}^{i-1}{\tau_j} + \frac{1}{m}\sum_{j=1}^{m-i}{\tau_j}
\end{split}
\end{align}
With probability $1/m$, the next state after executing the option is $1$ step to the left of $s$ and the value of $\tau_{1}$ is then $1$ . With probability $1/m$ the next state is $2$ steps to the left of $s$ and so $s'$ is now located $1$ step to the right of the new state: the value of $\tau_{1}$ is thus $2+\tau_{1}$. With probability $1/m$ the next state is $3$ steps to the left of $s$ and and so $s'$ is now located $2$ steps to the right of the new state: the value of $\tau_{1}$ is thus $3+\tau_{2}$. And so on and so forth. What we used here is basically the law of total expectations where the partition of events is the set of all possible states reached after executing the option only once. The same thing can be done for $\tau_2 \dots \tau_{m-1}$. It is trivial to verify that the only solution of the linear system in equation \ref{eq:diameterExperiments} is: $\tau_{i} = m(m+1)/2, \forall i \in \lbrace 1 ... m-1 \rbrace$. This results is rather intuitive: $m$ corresponds to the expected number of times the option needs to be executed to end up in the desired state $s'$ whereas $(m+1)/2$ is the expected duration at every decision step. The simplicity of this result comes from the symmetry of the problem: every time an option is executed, we stay $m$-close to $s'$ and the probability to exactly reach $s'$ is always $1/m$. So in this sense, we have i.i.d. Bernoulli trials where the probability of success is $1/m$. The expected time to reach $s'$ when we start in an $m$-close state both horizontally and vertically is thus $2 \times m (m+1)/2 = m(m+1)$. Therefore, the expected time to go from $s$ to $s'$ is always bounded by $D + m(m+1)$.

\begin{figure}
\centerline{
\scalebox{0.7} {
\begin{pspicture}(0,-2.2784376)(11.2,2.2734375)
\definecolor{color229c}{rgb}{0.5019607843137255,0.5019607843137255,0.5019607843137255}
\definecolor{color45}{rgb}{0.00784313725490196,0.4980392156862745,0.00784313725490196}
\rput(0.0,-1.2415625){\psgrid[gridwidth=0.028222222,subgridwidth=0.014111111,gridlabels=0.0pt,subgriddiv=1,unit=1.6cm,subgridcolor=color229c](0,0)(0,0)(7,1)
\psset{unit=1.0cm}}
\usefont{T1}{ptm}{m}{n}
\rput(0.82234377,-0.4565625){$RIGHT$}
\usefont{T1}{ptm}{m}{n}
\rput(7.201875,-0.4365625){$LEFT$}
\usefont{T1}{ptm}{m}{n}
\rput(2.4423437,-0.4365625){$RIGHT$}
\usefont{T1}{ptm}{m}{n}
\rput(4.0223436,-0.4365625){$RIGHT$}
\usefont{T1}{ptm}{m}{n}
\rput(8.841875,-0.4365625){$LEFT$}
\usefont{T1}{ptm}{m}{n}
\rput(10.381875,-0.4365625){$LEFT$}
\usefont{T1}{ptm}{m}{n}
\rput(5.586094,-0.35){\fontsize{20pt}{20pt}\selectfont \bm{$s'$}}
\psline[linewidth=0.11399999cm,fillcolor=black,dotsize=0.07055555cm 2.0,arrowsize=0.05291667cm 2.0,arrowlength=1.4,arrowinset=0.4]{**->}(0.8,-1.5615625)(5.48,-1.5415626)
\psline[linewidth=0.11399999cm,fillcolor=black,arrowsize=0.05291667cm 2.0,arrowlength=1.4,arrowinset=0.4,dotsize=0.07055555cm 2.0]{<-**}(5.82,-1.5415626)(10.46,-1.5415626)
\usefont{T1}{ptm}{m}{n}
\rput(3.1221876,-2.0565624){\fontsize{14pt}{14pt}\selectfont $m$ steps}
\usefont{T1}{ptm}{m}{n}
\rput(8.162188,-2.0565624){\fontsize{14pt}{14pt}\selectfont $m$ steps}
\psbezier[linewidth=0.07,linecolor=color45,arrowsize=0.05291667cm 4.0,arrowlength=1.2,arrowinset=0.4]{<-}(4.04,0.3784375)(4.02,1.6455839)(0.82,1.7384375)(0.82,0.42625433)
\psbezier[linewidth=0.07,linecolor=color45,arrowsize=0.05291667cm 4.0,arrowlength=1.2,arrowinset=0.4]{<-}(8.92,0.3784375)(8.92,2.1779072)(4.2,2.2384374)(4.24,0.47483766)
\psbezier[linewidth=0.07,linecolor=color45,arrowsize=0.05291667cm 4.0,arrowlength=1.2,arrowinset=0.4]{->}(8.72,0.4468096)(8.72,1.3584375)(7.24,1.2900654)(7.24,0.3784375)
\psbezier[linewidth=0.07,linecolor=color45,arrowsize=0.05291667cm 4.0,arrowlength=1.2,arrowinset=0.4]{->}(7.0,0.48680958)(7.0,1.3984375)(5.6,1.3300654)(5.6,0.4184375)
\end{pspicture} 
}
}
\caption{Behaviour of the agent in $m$-close states.} \label{fig:close.states}
\end{figure}
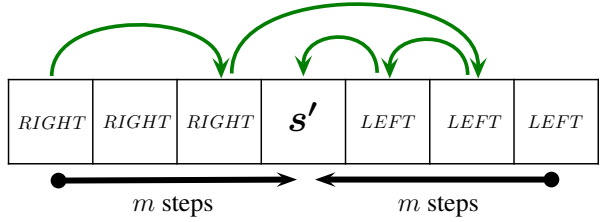

On Fig. \ref{fig:deterministic.options} we illustrate what happens when the options are deterministic i.e., when they terminate after exactly $m$ time steps. On this example we chose $m = 3$. If we start from state $s_0$, only the green states can be reached without resorting to the restart triggered by the target state, whereas if we start from $s_1$ only the blue states can be reached \footnote{Here we slightly simplified the problem. In reality, due to the truncation operated by the walls and if we assume $s_0$ to be the leftmost state, it is possible to go from $s_1$ to $s_0$ in one time step by executing $LEFT$. But for $m \geq 3 $ there will always be pairs of states that cannot be reached from each other without a restart. If $s_0$ is the leftmost state, this is the case for $s_1$ and the white state adjacent to it on Fig. \ref{fig:deterministic.options}. So the proof remains valid.} (and similarly for the white states). Let's assume that we want to go from a green state to a blue state. The only way to do so is to go to the target and "hope" to end up in one of the blue states after the restart (we recall that the restart state is chosen randomly with equi-probability). The shortest path to go from any state to the target is bounded by $D$ and the probability to restart in a state with the desired colour is $1/m$ ($1/m^2$ in dimension 2). We can thus upper bound the diameter of the SMDP $M_\mathcal{O}$ by the expected time needed to go from $s_0$ to $s_1$ in the SMDP of Fig. \ref{fig:deterministic.options.smdp}, that is: $D_\mathcal{O} \leq D(1+m^2)$. This bound is tight (up to a constant factor) since the average time to go from any state chosen at random with equi-probability to the target is exactly $D/2$ in the 2-dimensional grid.

\paragraph{Optimality}
Since the target state is located in a corner of the grid, the shortest path to go from any state to the target is equally long in the original MDP and the MDP with options. As a result, the optimal average rewards are also equal (i.e., there exists an optimal policy using only options \textit{LEFT} \textit{RIGHT}, \textit{UP} and \textit{DOWN} which consists in applying only \textit{RIGHT} or \textit{DOWN}).

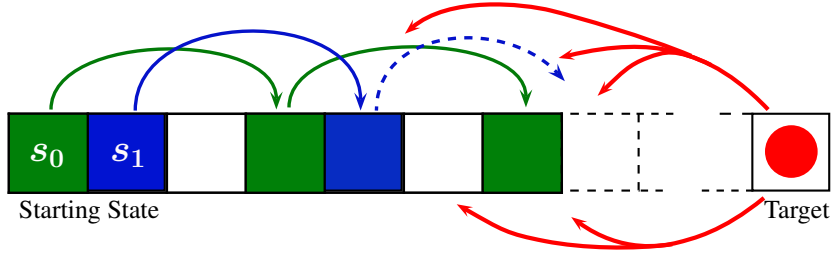
\begin{figure}[t]
\centerline{
\scalebox{0.7} {
\begin{pspicture}(0,-2.49)(15.8375,2.49)
\definecolor{color94c}{rgb}{0.5019607843137255,0.5019607843137255,0.5019607843137255}
\definecolor{color96b}{rgb}{0.0,0.5019607843137255,0.03137254901960784}
\definecolor{color108}{rgb}{0.00784313725490196,0.47843137254901963,0.03137254901960784}
\definecolor{color164b}{rgb}{0.00392156862745098,0.07450980392156863,0.8196078431372549}
\definecolor{color165b}{rgb}{0.027450980392156862,0.17254901960784313,0.7607843137254902}
\definecolor{color166}{rgb}{0.0196078431372549,0.07058823529411765,0.792156862745098}
\definecolor{color12}{rgb}{1.0,0.9803921568627451,0.9803921568627451}
\definecolor{color8}{rgb}{0.996078431372549,0.996078431372549,0.996078431372549}
\rput(0.1975,-1.29){\psgrid[gridwidth=0.0482,subgridwidth=0.014111111,gridlabels=0.0pt,subgriddiv=1,unit=1.5cm,subgridcolor=color94c](0,0)(0,0)(7,1)
\psset{unit=1.0cm}}
\psframe[linewidth=0.04,dimen=outer,fillstyle=solid,fillcolor=color96b](1.7175,0.23)(0.1775,-1.31)
\psframe[linewidth=0.04,dimen=outer,fillstyle=solid,fillcolor=color96b](6.2175,0.23)(4.6775,-1.31)
\psframe[linewidth=0.04,dimen=outer,fillstyle=solid,fillcolor=color96b](10.7175,0.23)(9.1775,-1.31)
\psframe[linewidth=0.04,linestyle=dashed,dash=0.16cm 0.16cm,dimen=outer](12.1775,0.23)(10.6775,-1.27)
\psframe[linewidth=0.04,dimen=outer](15.7975,0.23)(14.2975,-1.27)
\psdots[dotsize=1.0,linecolor=red](15.0575,-0.51)
\usefont{T1}{ptm}{m}{n}
\rput(15.162344,-1.665){\fontsize{14pt}{14pt}\selectfont Target}
\psline[linewidth=0.04cm,linestyle=dashed,dash=0.16cm 0.16cm](13.6975,0.21)(14.3375,0.21)
\psline[linewidth=0.04cm,linestyle=dashed,dash=0.16cm 0.16cm](13.3575,-1.25)(14.3375,-1.25)
\psline[linewidth=0.04cm,linestyle=dashed,dash=0.16cm 0.16cm](12.1175,0.21)(12.8175,0.21)
\psline[linewidth=0.04cm,linestyle=dashed,dash=0.16cm 0.16cm](12.1975,-1.25)(12.6175,-1.25)
\psbezier[linewidth=0.064,linecolor=color108,arrowsize=0.05291667cm 3.0,arrowlength=1.4,arrowinset=0.4]{<-}(5.2775,0.29)(5.2775,1.79)(0.9975,1.79)(0.9975,0.29)
\psbezier[linewidth=0.064,linecolor=color108,arrowsize=0.05291667cm 3.0,arrowlength=1.4,arrowinset=0.4]{<-}(10.0175,0.33)(10.0175,1.83)(5.5175,1.83)(5.5175,0.33)
\psbezier[linewidth=0.08,linecolor=red,arrowsize=0.05291667cm 2.0,arrowlength=1.4,arrowinset=0.4]{->}(14.6175,0.29)(13.8175,1.13)(13.182116,1.3114285)(12.7575,1.41)(12.332885,1.5085714)(11.6951065,1.53)(10.6375,1.27)
\psbezier[linewidth=0.08,linecolor=red,arrowsize=0.05291667cm 2.0,arrowlength=1.4,arrowinset=0.4]{->}(14.5375,-1.39)(13.746388,-2.041077)(13.168269,-2.1684616)(12.8944235,-2.2392309)(12.620577,-2.31)(11.8375,-2.45)(10.8775,-1.7537001)
\psbezier[linewidth=0.08,linecolor=red,arrowsize=0.05291667cm 2.0,arrowlength=1.4,arrowinset=0.4]{->}(13.1575,1.31)(12.2975,1.67)(11.033213,2.0375035)(10.4975,2.17)(9.961787,2.3024967)(8.5575,2.45)(7.7175,1.73)
\psbezier[linewidth=0.08,linecolor=red,arrowsize=0.05291667cm 2.0,arrowlength=1.4,arrowinset=0.4]{->}(13.5375,1.1351395)(12.840605,1.4652156)(12.656909,1.3961469)(12.3175,1.27)(11.978091,1.1438531)(11.8375,0.99)(11.3775,0.55)
\psbezier[linewidth=0.08,linecolor=red,arrowsize=0.05291667cm 2.0,arrowlength=1.4,arrowinset=0.4]{->}(12.7975,-2.27)(12.0575,-2.37)(11.0975,-2.35)(10.3775,-2.21)(9.6575,-2.07)(9.5775,-2.01)(8.702432,-1.51)
\usefont{T1}{ptm}{m}{n}
\rput(1.72109376,-1.665){\fontsize{14pt}{14pt}\selectfont Starting State}
\psframe[linewidth=0.04,dimen=outer,fillstyle=solid,fillcolor=color164b](3.1775,0.23)(1.6775,-1.27)
\psframe[linewidth=0.04,dimen=outer,fillstyle=solid,fillcolor=color165b](7.6775,0.23)(6.1775,-1.27)
\psbezier[linewidth=0.064,linecolor=color166,arrowsize=0.05291667cm 3.0,arrowlength=1.4,arrowinset=0.4]{<-}(6.8775,0.27)(6.8775,2.25)(2.5975,2.25)(2.5975,0.27)
\psbezier[linewidth=0.064,linecolor=color166,linestyle=dashed,dash=0.16cm 0.16cm,arrowsize=0.05291667cm 3.0,arrowlength=1.4,arrowinset=0.4]{<-}(10.7575,0.8281818)(9.635121,2.05)(7.1775,1.8730303)(7.1775,0.27)
\usefont{T1}{ptm}{m}{n}
\rput(0.92234373,-0.585){\color{color8} \fontsize{20pt}{20pt}\selectfont \bm{$s_0$}}
\usefont{T1}{ptm}{m}{n}
\rput(2.4667187,-0.585){\color{color12} \fontsize{20pt}{20pt}\selectfont \bm{$s_1$}}
\end{pspicture} 
}
}
\caption{Behaviour of the agent with deterministic options.}\label{fig:deterministic.options}
\end{figure}

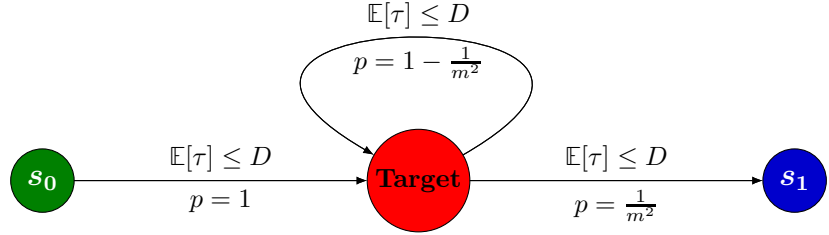
\begin{figure}[t]
\begin{tikzpicture}
	\useasboundingbox (-1,-0.5) rectangle (15,2.3);
	
	\tikzset{VertexStyle/.style = {draw, 
									shape          = circle,
	                                 text           = black,
	                                 inner sep      = 2pt,
	                                 outer sep      = 0pt,
	                                 minimum size   = 24 pt}}
	                                                               
	\tikzset{EdgeStyle/.style   = {->, >=latex}}
	\tikzset{LabelStyle/.style =   {above}}
	
		\node[VertexStyle,fill=black!50!green](0) at (3,0) {\color{white} \fontsize{12pt}{12pt}\selectfont \bm{$ s_{0} $}};
	\node[VertexStyle,fill=red](1) at (8,0){\textbf{Target}};
	\node[VertexStyle,fill=black!20!blue](2) at (13,0){\color{white} \fontsize{12pt}{12pt}\selectfont \bm{$ s_{1} $}};
	
		\draw[EdgeStyle](0) to node[above]{$\mathbb{E}[\tau] \leq D $} (1);
	\draw[EdgeStyle](0) to node[below]{$ p = 1$} (1);
	\draw[EdgeStyle](1) to node[above]{$ \mathbb{E}[\tau] \leq D $} (2);
	\draw[EdgeStyle](1) to node[below]{$ p = \frac{1}{m^2}$} (2);
	\draw[->, >=latex] (1) to [out=30,in=150,looseness=9] node[above]{$\mathbb{E}[\tau] \leq D $} (1);
	\draw[->, >=latex] (1) to [out=30,in=150,looseness=9] node[below]{$ p = 1- \frac{1}{m^2}$} (1);
\end{tikzpicture}
	\caption{Upper bound on the diameter of $M_\mathcal{O}$ for deterministic options.} \label{fig:deterministic.options.smdp}
\end{figure}

\paragraph{Asymptotic behaviour}
We will now analyse the behaviour of the ratio $\frac{n}{T_n}$ using results on martingales.

\begin{theorem}[Martingale Strong Law of Large Numbers, \citep{Vovk:2005:ALR:1062391}]\label{LLNMartingales}
Let $X_1, ..., X_n$ be a martingale difference sequence w.r.t. a filtration $\mathcal{F}_0, \mathcal{F}_1,..., \mathcal{F}_n$ and let $A_1, ..., A_n$ be an increasing predictable sequence w.r.t. the same filtration with $A_1 >0$ and $\lim_{n \rightarrow + \infty} A_n = +\infty$ almost surely. If:
\begin{align*}
\sum_{i=1}^{+\infty}{\frac{\mathbb{E}[X_i^2 | \mathcal{F}_{i-1}]}{A_i^2}}< + \infty \ \ \mbox{    a.s.}
\end{align*}
then:
\begin{align*}
\frac{1}{A_n} \sum_{i=1}^{n}{X_i} \xrightarrow[n \rightarrow +\infty]{} 0  \ \ \mbox{    a.s.}
\end{align*}
\end{theorem}

Let's take $X_i = \tau_i - \overline{\tau}_i$ (where $\overline{\tau}_i = \overline{\tau}_i(s_{i-1},a_i)$) and $\mathcal{F}_i = \sigma\left(s_0, a_1, \tau_1, r_1, ..., s_{i}, a_{i+1} \right)$. The sequence $(X_i)_{i\leq 1}$ is a martingale difference because $\mathbb{E}[X_i]<+\infty$ and $\mathbb{E}[X_{i}|\mathcal{F}_{i-1}] = 0$. Since $(\tau(s,a,s'))_{s,a,s'}$ are sub-Exponential, all moments are finite and it is well known from the literature that the variance is bounded by the sub-Exponential constant $\sigma_\tau^2$ hence: $\mathbb{E}[X_i^2 | \mathcal{F}_{i-1}]<\sigma_\tau^2$. If in addition we take $A_i = i$ then the conditions of Theorem \ref{LLNMartingales} are satisfied and thus:
\begin{align*}
\frac{T_n}{n} - \frac{\overline{T_n}}{n} \xrightarrow[n \rightarrow +\infty]{} 0  \ \ \mbox{    a.s.}
\end{align*}
where $\overline{T_n} = \mathbb{E}[T_n]$. By definition: $\tau_{\max}n \geq \overline{T_n} \geq \tau_{\min}n$ hence: $\forall \epsilon>0, \ \exists N_\epsilon >0$  s.t. $\forall n \geq N_\epsilon :$
\begin{align*}
\bigg| \frac{T_n}{n} - \frac{\overline{T_n}}{n} \bigg| \leq \epsilon  \ \ \mbox{a.s.} \implies \tau_{\min} - \epsilon \leq \frac{\overline{T_n}}{n} - \epsilon \leq \frac{T_n}{n} \leq \epsilon + \frac{\overline{T_n}}{n} \leq \epsilon + \tau_{\max}
\end{align*}
and so: $\liminf_{n \rightarrow +\infty} \frac{T_n}{n} \geq \tau_{\min}$ and $\limsup_{n \rightarrow +\infty} \frac{T_n}{n} \leq \tau_{\max}$ a.s. Finally:
\begin{align*}
\frac{\log \lbrace\frac{n}{\delta}\rbrace}{\log \lbrace\frac{T_n}{\delta}\rbrace}  = \frac{\log \lbrace\frac{n}{\delta}\rbrace}{\log \lbrace\frac{T_n - \overline{T_n}}{n} + \frac{\overline{T_n}}{n}\rbrace + \log \lbrace\frac{n}{\delta}\rbrace} \leq \frac{\log \lbrace\frac{n}{\delta}\rbrace}{\log \lbrace\frac{T_n - \overline{T_n}}{n} + \tau_{\min}\rbrace + \log \lbrace\frac{n}{\delta}\rbrace} \xrightarrow[n \rightarrow +\infty]{} 1
\end{align*}
In the general case of sub-Exponential rewards and holding times our results provide no theoretical evidence of the advantage of introducing options due to the fact that $\mathcal{C}(M',n,\delta)$ scales as $\sqrt{\log(n)}$:
\begin{align*}
\lim_{n \rightarrow +\infty} \mathcal{R}(M,n,\delta) =  +\infty \ \ \mbox{a.s.}
\end{align*}
but if the rewards and holding times are bounded we have:
\begin{align*}
\limsup_{n \rightarrow +\infty} \mathcal{R}(M,n,\delta) \leq  \frac{1}{\sqrt{\tau_{\min}}}\left(1+ \frac{T_{\max}}{D\sqrt{S}}\right) \ \ \mbox{a.s.}
\end{align*}
Note that $\tau_{\min}$ is a very loose upper-bound on $\liminf_{n \rightarrow +\infty} \frac{T_n}{n}$ and in practice the ratio $\frac{T_n}{n}$ can take much higher values if $\tau_{\max}$ is big and many options have a high expected holding time.

\paragraph{Tightness of the upper bounds}

On Fig. \ref{fig:theoretical.ratios}, we plot the expected theoretical values taken by the ratio of the regrets according to our upper bounds (formula given in Sect.~\ref{sec:exp}). On Fig. \ref{fig:empirical.ratios} however, we plot the empirical values of the ratios in our experiments (same graph as on Fig.~\ref{fig:results.ratio}). We can see that the curves have similar shapes. In particular, they reach their respective minima for the same value of $T_{\max}$ and the value of these minimum is below 1 (meaning that learning with options is more efficient than with primitive actions in this case). Moreover, the theoretical ratios are upper-bounding the empirical ones for all values of $T_{\max}$. We can conclude that the ratio of the upper bounds is a good proxy for the true ratio in this example.

\begin{figure}
\begin{subfigure}{1\textwidth}
\center
\includegraphics[width=0.8\textwidth]{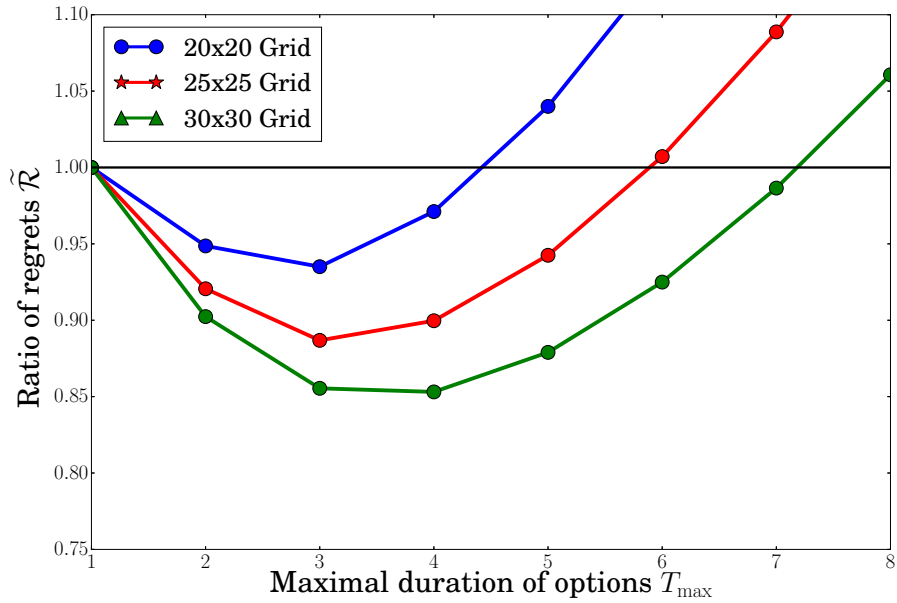}
\caption{}
\label{fig:theoretical.ratios}
\end{subfigure}\\
\begin{subfigure}{1\textwidth}
\center
\includegraphics[width=0.8\textwidth]{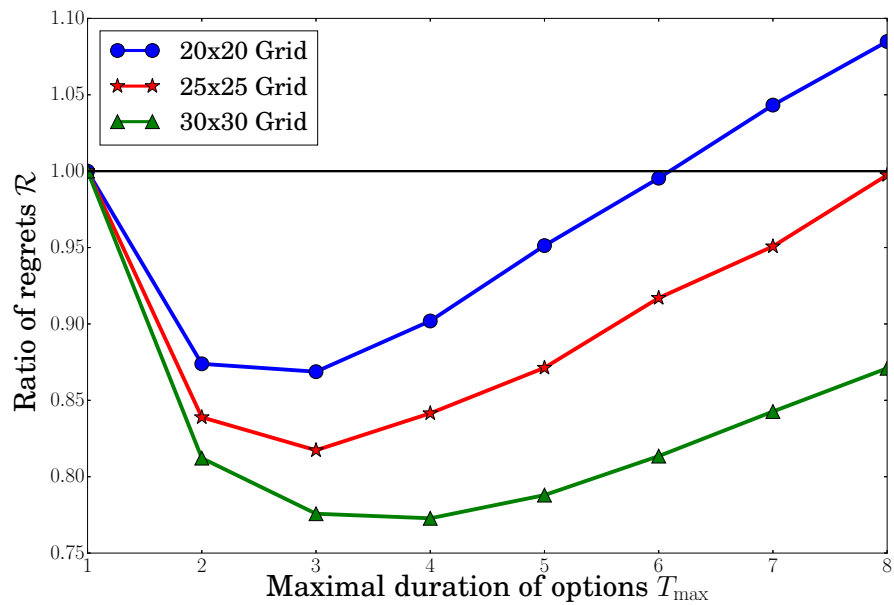}
\caption{}
\label{fig:empirical.ratios}
\end{subfigure}
\caption{\textit{(a)} Theoretical ratios of the regrets with and without options for different values of $T_{\max}$; \textit{(b)} Empirical ratios of the regrets with and without options for different values of $T_{\max}$.}
\label{fig:comparison.ratios}
\end{figure} 
\end{document}